\documentclass[twoside,leqno,twocolumn]{article}

\usepackage[letterpaper]{geometry}

\usepackage{ltexpprt}
\usepackage{hyperref}

\input{styles/include.sty}
\input{styles/problems.sty}
\input{styles/variables.sty}
\input{styles/utilities.sty}

\begin{document}

\title{\Large A Space-Efficient Algebraic Approach to Robotic Motion Planning}
\author{%
Matthias Bentert\thanks{University of Bergen, Norway}\and
Daniel Coimbra Salomao\thanks{University of Utah, USA}\and
Alex Crane\footnotemark[2]\and
Yosuke Mizutani\footnotemark[2]\and
Felix Reidl\thanks{Birkbeck, University of London, UK} \and
Blair D. Sullivan\footnotemark[2]
}

\date{}

\maketitle







\begin{abstract}
  We consider efficient route planning for robots in applications such as infrastructure inspection and automated surgical imaging.
  These tasks can be modeled via the combinatorial problem \gi{}.
  The best known algorithms for this problem are limited in practice by exponential space complexity.
  In this paper, we develop a memory-efficient approach using algebraic tools related to monomial testing on the polynomials associated with certain arithmetic circuits.
  Our contributions are two-fold. 
  We first repair a minor flaw in existing work on monomial detection using a new approach we call tree certificates. We further show that, in addition to detection, these tools allow us to efficiently \emph{recover} monomials of interest from circuits, opening the door for significantly broadened application of related algebraic tools.
  For \gi{}, we design and evaluate a complete algebraic pipeline. 
  Our engineered implementation demonstrates that circuit-based algorithms are indeed memory-efficient in practice, thus encouraging further engineering efforts.
\end{abstract}

\section{Introduction}\label{sec:intro}

We are motivated by robotic inspection planning~\cite{englot2017planning,fu2021computationally}, where a robot is tasked with inspecting ``points of interest'' by traveling a route in its configuration space.
This problem is modeled combinatorially as \gi{}:
given an edge-weighted and
vertex-multi-colored graph, find a minimum-weight closed walk from a given
starting vertex $s$ that collects at least~$t$ colors.
The colors allow us to model a ``collection'' problem,
generalizing\footnote{Related is the \textsc{Generalized Traveling Salesman Problem} (\textsc{GTSP})~\cite{pop2023comprehensive}, which asks for a simple path rather than a walk. Our results extend to \textsc{GTSP} restricted to complete metric graphs.\looseness=-1}
the \textsc{Traveling Salesman Problem} by, in our setting, allowing points of interest to be observable from multiple robot configurations. 
Traditionally, roboticists have relied upon heuristic solvers, as exact solutions
were viewed as prohibitively expensive to compute.
Recently, however, Mizutani \emph{et al.}~\cite{mizutani2024leveraging} improved upon the state-of-the-art planning heuristic~\cite{fu2021computationally} by using exact integer linear programming (ILP) and dynamic programming (DP) solvers for \gi{} as subroutines.
They demonstrated improved solution quality with comparable running times on simulated robotic tasks.
Mizutani \emph{et al.} noted two limitations of their solvers which constrain their scalability.
The ILP solver fails on large networks\footnote{The experiments in~\cite{mizutani2024leveraging} were limited to instances with roughly $2,000$ vertices and $40,000$ edges.} in part because the ILP itself scales with the number of edges in the network times the number $|\colorset|$ of colors.
In contrast, the DP solver's running time scales linearly in the network size but exponentially in~$|\colorset|$.
Crucially, this exponential dependence on~$|\colorset|$ is also true of the DP's memory consumption, creating a sharp limit for its use case, especially given increasing interest in reducing the hardware (and economic) requirements of motion planning~\cite{ichnowski2020cloud,ichnowski2014cache,ichnowski2020economic,noreen2016optimal}.

At a high level, the goal of this paper is to determine whether an algebraic circuit-based framework~\cite{koutis2016limitsapplications} might
overcome these shortcomings. Toward this end,
we first contribute new machinery which allows us to (a) repair a minor flaw in the framework, and (b) use the framework \emph{constructively}, rather than only to answer decision problems.
These contributions apply to the framework as a whole, not just our specific application. 
Second, we use the newly-augmented framework to design, analyze, implement, and evaluate an algorithm for \gi{} which
has \emph{polynomial} space complexity and similar theoretical running time to the DP solver~\cite{mizutani2024leveraging}. Our implementation
fits seamlessly into the pipeline of~Mizutani \emph{et al.} \cite{mizutani2024leveraging} and uses significantly less memory.
Though the empirical running time of our algorithm remains slow,
our experiments reveal that it is highly sensitive to a variety of algorithmic and implementation optimizations.
This motivates further research on the engineering of algebraic techniques.

\bigskip
\noindent\textbf{Algebraic Simulation of Dynamic Programming}.
One theoretical remedy to dynamic programming algorithms with exponential space complexity are algebraic approaches which ``simulate'' dynamic programming by constructing compact arithmetic circuits and testing the polynomials represented by these circuits for the presence of multilinear monomials~\cite{guillemot2013findingcounting,koutis2008fasteralgebraic,koutis2016limitsapplications,williams2009findingpaths}. We can conceptualize this as a polynomial-time (and thus polynomial-space) reduction from the source problem to an instance of \textsc{Multilinear Detection}, in which we are given an arithmetic circuit and are tasked with deciding whether a multilinear monomial exists (see~\Cref{sec:prelims} for definitions). 
Our basic algorithmic strategy, outlined in~\Cref{sec:gi-algorithm}, is to design such a reduction from \gi{} and then apply algorithms for \MLD{}, which crucially require only polynomial space despite running times similar to the direct dynamic programming approach.
However, we must overcome numerous challenges.\looseness=-1

\bigskip
\noindent\textbf{Tree Certificates}.
Our main theoretical contribution is the introduction of \emph{tree certificates}, which serve two purposes.
First, they allow us to repair a newly-discovered flaw in the best-known algorithm for \MLD{} (see~\Cref{sec:MLD}).
Second, they enable efficient \emph{solution recovery}.
From the viewpoint of a theorist, it is appropriate to design a reduction from one's source problem to \MLD{} (a decision problem),
and then simply note that a solution may be constructed via self-reduction.
In practice, however, it is essential to avoid this latter step if there is to be any hope of acceptable running times.
Toward this end, in~\Cref{sec:solution-recovery-mld} we present two
randomized algorithms (one Monte Carlo and one Las Vegas) which recover tree certificates for certain linear monomials of degree $k$ in
time~$\tilde{\Oh}(2^k \alpha m)$ in circuits with~$m$ edges,
where $\alpha$ is the maximum number of addition gates in a tree certificate.
In~\Cref{sec:two-phase}, we refine our approach to
provide a faster solution recovery method specific to~\gi{}.\looseness=-1

\bigskip
\noindent\textbf{Refining our Reduction}.
Our initial reduction from \gi{} to \MLD{} is non-trivial, but since the running time
of a \MLD{} algorithm depends on various structural aspects of the constructed circuit,
we optimize the reduction further.
In~\Cref{sec:circuit}, we give four circuit constructions (i.e., reductions) which all guarantee correctness, but have varying asymptotic guarantees on the structure of the constructed circuits.
A similar challenge is that when constructing a circuit to represent an instance of~\gi{}, we must repeatedly guess
the weight of an optimal solution walk. In~\Cref{sec:search-strategies}, we present three strategies, each with randomized correctness guarantees.
Putting it all together, we prove in~\Cref{sec:proof-main-thm} that our techniques yield an algorithm which solves \gi{} in~$\tilde{\Oh}(2^t (\ell t^3 n^2 + t^3 \colorsetsize n))$ time and uses~$\tilde{\Oh}(\ell t n^2 + t|C| n)$ space, where~$\ell$ is an upper bound on the solution weight, $\colorsetsize$ is the number of colors in the instance, and~$t$ is the minimum number of colors we must collect.


\vspace*{1em}
\noindent\textbf{Empirical Evaluation}.
In~\Cref{sec:experiments}, we present the results of a comprehensive evaluation of our techniques
on two real-world datasets from Fu \emph{et al.}~\cite{fu2021computationally}, one for a bridge inspection scenario and another for a surgical imaging scenario.
We test different combinations of
strategies for circuit construction, finding the optimal solution weight, and recovering solution walks.
By comparing directly against the DP-based implementation of Mizutani \emph{et al.}~\cite{mizutani2024leveraging},
we demonstrate that our circuit-based approach uses significantly less memory once at least~$19$ colors must be collected, and that the memory consumption trends for both software solutions match the asymptotics, suggesting that this gap will rapidly grow.
Moreover, while the circuit-based software remains slower than dynamic programming, the advances made
in~\Cref{sec:solution-recovery-mld,sec:circuit,sec:search-strategies,sec:two-phase} produce
large improvements. We also detail several findings related to parameter tuning and multithreading. 

\section{Preliminaries}\label{sec:prelims}

\begin{figure*}[h]
    \pgfdeclarelayer{bg}
    \pgfsetlayers{bg, main}

    \tikzstyle{bigblacknode} = [circle, fill=gray, text=white, draw, thick, scale=1, minimum size=0.6cm, inner sep=1.5pt]
    \tikzstyle{bigwhitenode} = [circle, fill=white, text=black, draw, thick, scale=1, minimum size=0.6cm, inner sep=1.5pt]

    \tikzstyle{blacknode} = [circle, fill=gray, draw, thick, scale=1, minimum size=0.2cm, inner sep=1.5pt]
    \tikzstyle{whitenode} = [circle, fill=white, draw, thick, scale=1, minimum size=0.2cm, inner sep=1.5pt]

    \tikzstyle{hugewhitenode} = [circle, fill=white, text=black, draw, thick, scale=1, minimum size=1.5cm, inner sep=1.5pt, font=\large]
    \tikzstyle{directed} = [color=black, arrows=- triangle 45]

    \tikzset{
        old inner xsep/.estore in=\oldinnerxsep,
        old inner ysep/.estore in=\oldinnerysep,
        double circle/.style 2 args={
            circle,
            old inner xsep=\pgfkeysvalueof{/pgf/inner xsep},
            old inner ysep=\pgfkeysvalueof{/pgf/inner ysep},
            /pgf/inner xsep=\oldinnerxsep+#1,
            /pgf/inner ysep=\oldinnerysep+#1,
            alias=sourcenode,
            append after command={
            let     \p1 = (sourcenode.center),
                    \p2 = (sourcenode.east),
                    \n1 = {\x2-\x1-#1-0.5*\pgflinewidth}
            in
                node [inner sep=0pt, draw, circle, minimum width=2*\n1,at=(\p1),#2] {}
            }
        },
        double circle/.default={-3pt}{black}
    }

    \tikzmath{\yunit = 0.6;}

    \centering
    \begin{minipage}[m]{0.90\linewidth}
        \vspace{0pt}
        \centering
        \begin{tikzpicture}
            \node[bigblacknode] (x) at (0, 0) {$y$};
            \node[bigblacknode] (y) at (0, 2 * \yunit) {$x$};
            \node[bigwhitenode] (a1) at (1, \yunit) {$+$};
            \node[bigwhitenode] (m1) at (2.4, \yunit) {$\times$};
            \node[bigwhitenode] (a2) at (3.4, 0) {$+$};
            \node[bigwhitenode] (a3) at (3.4, 2 * \yunit) {$+$};
            \node[bigwhitenode, double circle] (m2) at (4.4, \yunit) {$\times$};

            \draw[directed] (x) -- (a1);
            \draw[directed] (y) -- (a1);
            \draw[directed] (a1) -- (m1);
            \draw[directed] (m1) -- (a2);
            \draw[directed] (m1) -- (a3);
            \draw[directed] (a2) -- (m2);
            \draw[directed] (a3) -- (m2);

            \node () at (0.8, \yunit + 0.5) {$a_1$};
            \node () at (0.8, \yunit - 0.5) {$a_2$};
            \node () at (2.8, \yunit + 0.6) {$a_3$};
            \node () at (2.8, \yunit - 0.6) {$a_4$};

            \draw[rounded corners, dashed, gray] (-0.4, -0.4) rectangle ++ (0.8, 2 * \yunit + 0.8);
            \node () at (0, -0.7) {\emph{variables}};
            \node () at (5.1, \yunit + 0.6) {\emph{output node}};
            \node () at (3.2, -0.7) {Arithmetic circuit $\C$};
            \node[white] () at (-1.0, 0) {.};

            \node[align=left,anchor=west] () at (5.8, 0.6) {
                \begin{minipage}{0.65\textwidth}
                \begin{align*}
                    P_\C(X) &= (x+y)^2=x^2 + 2xy + y^2\\
                    \\
                    P_\C(X,A) &= ((a_1 x+ a_2 y)\cdot a_3)((a_1 x+ a_2 y)\cdot a_4)\\
                &= a_3 a_4 a_1^2 x^2 + 2 a_1 a_2 a_3 a_4 xy + a_3 a_4 a_1^2 y^2
                \end{align*}
                \end{minipage}
            };
            \draw[gray] (7.1, -0.7) rectangle ++ (7.8, 2.2);
        \end{tikzpicture}
    \end{minipage}
    \caption{%
        An arithmetic circuit (left)
        with 2 variables, 3 addition gates, and 2 multiplication gates, and its associated polynomial $P_\C(X)$ and fingerprint polynomial $P_\C(X,A)$
        (right).
        This is a counterexample of a claim in \cite{koutis2016limitsapplications}.\looseness=-1
    }
    \label{fig:example-circuit}
\end{figure*}

We refer to the textbook by Diestel~\cite{diestel2005graph} for standard graph-theoretic definitions and notation.
For a set $S$, the notation $2^S$ indicates the power set of $S$.
Unless otherwise specified, all graphs $G = (V, E)$ in this work are undirected, with edges weighted
by a function $w \colon E \rightarrow \R_{\geq 0}$ and vertices
multi-colored by a function ${\col \colon V \rightarrow 2^\colorset}$, where $\colorset$ is the \emph{color set}.
Given a vertex subset $S \subseteq V$, we use the notations $\col(S)$ for $\bigcup_{v \in S} \col(v)$,
$G[S]$ for the subgraph induced by $S$, and $G-S$ for ${G[V \setminus S]}$. If $S = \{v\}$, we write $G-v$
instead of $G-\{v\}$.\looseness=-1

A (simple) \emph{path} $P = (v_1, v_2, \ldots, v_p)$ is a sequence of
distinct vertices with $v_iv_{i+1} \in E$ for all $i < p$. A \emph{walk} is defined similarly, but in this case a vertex
may appear more than once. A walk is \emph{closed} if it starts and ends at the same vertex.
The \emph{weight} of a walk is the sum of the weights of its edges, i.e., $\sum_{i = 1}^{p-1}w(v_iv_{i+1})$.
The \emph{distance} between two vertices $u, v \in V$, denoted by $d(u, v)$, is the minimum weight
across all walks between $u$ and $v$. The distance between a vertex $v$ and a vertex set $S \subseteq V$
is the minimum distance between $v$ and any vertex in $S$, i.e., $d(v, S) = d(S, v) = \min_{u \in S} d(v, u)$.

The $\Oh^*(\cdot)$ and $\tilde{\Oh}(\cdot)$ notations are variants on the standard $\Oh(\cdot)$, hiding polynomial and polylogarithmic factors, respectively.
Now, we define our problem:

\begin{problembox}{Graph Inspection}
    \Input & A graph~$G=(V,E)$, a color set $\colorset$, edge-weights $w \colon E \rightarrow \R_{\geq 0}$,
    vertex-colors $\col \colon V \rightarrow 2^\colorset$, a vertex $s$, and an integer $t$. \\
    \Prob  & Find a minimum-weight closed walk $P= (v_0,v_1,\ldots,v_{p})$ in $G$
    with $v_0=v_p=s$ and $|\bigcup_{i=1}^p \col(v_i)|\geq t$.
\end{problembox}

For the sake of simplicity, we may assume that~$G$ is connected, $t \leq |\colorset |$, and~$\col(s) = \emptyset$.


\smallskip
\noindent\textbf{Polynomials and arithmetic circuits.} 
A \emph{monomial} over a set of variables~$X$ is a (commutative) product of variables from~$X$. We call a monomial \emph{multilinear} if no variable appears more than once. A polynomial is a linear combination of monomials with coefficients from~$\mathbb Z_+$.\looseness=-1

An \emph{arithmetic circuit}~$\C$ over~$X$ is a directed acyclic graph (DAG) for which every source is labelled either by a constant from~$\mathbb Z_+$ (a \emph{scalar}) or by a variable from~$X$, and furthermore every internal node is labelled as either an \emph{addition} or a \emph{multiplication} node. The internal nodes and sinks of the DAG are called \emph{gates} and \emph{outputs}, respectively, of the circuit $\C$. 

For a node~$v \in \C$, we define~$\C[v]$ to be the arithmetic circuit induced by all
nodes that can reach~$v$ in~$\C$, including~$v$.
We further define~$P_{\C[v]}(X)$ to be the polynomial that results from expanding the arithmetic expression of~$\C[v]$
into a sum of products.
For a circuit~$\C$ with a single output~$r \in V(\C)$, we write~$P_\C(X)$ for the polynomial~$P_{\C[r]}(X)$.
\Cref{fig:example-circuit} illustrates an example of an arithmetic circuit.

We now formalize the problem of checking whether 
the polynomial representation of an output node includes
a multilinear monomial.

\begin{problembox}{\MLD}
  \Input & An arithmetic circuit $\C$ over a set $X$ of variables
  and an integer $k$.\\
  \Prob  & For each output node $r \in V(\C)$,
  determine if $P_{\C[r]}(X)$ contains a multilinear monomial of degree at most $k$.\\
\end{problembox}

\section{Solving \MLD}\label{sec:MLD}


Our strategy of reducing \gi{} to \MLD{} will only be effective if we can efficiently solve the latter problem.
The following lemma was claimed previously by Koutis and Williams~\cite{koutis2016limitsapplications}.

\begin{restatable}[$\bigstar$]{lemma}{mldlemma}\label{lem:williams}
  \MLD{} can be solved in randomized ${\Oh^*}(2^k)$-time and polynomial space.
  This is a one-sided error Monte Carlo algorithm with a constant success probability.
\end{restatable}

Koutis and Williams begin their argument by showing the result for a restricted class of circuits.
Let $P_\C(X)$ be a polynomial represented by an (arithmetic) circuit $\C$;
Koutis and Williams define $P_\C(X,A)$ to be a \emph{fingerprint} polynomial derived from~$\C$ 
as follows:
for every addition gate~$v \in V(\C)$ and input edge~$uv \in E(\C)$,
they annotate the edge~$uv$ with a dedicated variable~$a_{uv} \in A$.
The semantic of this annotation is that the output of node~$u$ is multiplied by~$a_{uv}$
before it is fed to~$v$.
Koutis and Williams prove the following:

\begin{lemma}[\cite{koutis2016limitsapplications, williams2009findingpaths}]\label{lem:williams-poly}
  Let $\C$ be a connected arithmetic circuit over a set $X$ of variables,
  and let $k$ be an integer.
  Suppose the coefficient of each monomial in $P_\C(X,A)$
  is $1$.
  %
  Then, there exists a randomized $\Oh^*(2^k)$-time polynomial-space
  algorithm for \prob{Multilinear Detection} on $(\C, X, k)$.
  %
  This is a one-sided error Monte Carlo algorithm with a constant success probability.
\end{lemma}

In an attempt to generalize this result to arbitrary circuits and thereby prove~\Cref{lem:williams},
Koutis and Williams defined $\mathcal A$-circuits
as circuits where addition and multiplication gates alternate,
addition gates have an out-degree of one, and all scalar inputs are either~0 or~1.
They claim that the associated fingerprint polynomials of $\mathcal A$-circuits only contain monomials with coefficient one.
However, our initial implementation showed that this claim does not hold,
and in fact the $\mathcal A$-circuit in \Cref{fig:example-circuit}
does not have the claimed property:
in the associated fingerprint polynomial $P_\C(X,A)$,
the only multilinear monomial ($2 a_1 a_2 a_3 a_4 xy$) has coefficient $2$.
This leaves \Cref{lem:williams} open\footnote{We observe that~\MLD{} can also be solved via determinantal sieving~\cite{eiben2024determinantal}, but this technique does not achieve polynomial space complexity on general circuits. Further, the running time is worse than the claim in~\Cref{lem:williams}, unless the matrix multiplication exponent $\omega = 2$.}.

\begin{figure}[t]
    \pgfdeclarelayer{bg}
    \pgfsetlayers{bg, main}

    \tikzstyle{bigblacknode} = [circle, fill=gray, text=white, draw, thick, scale=1, minimum size=0.6cm, inner sep=1.5pt]
    \tikzstyle{bigwhitenode} = [circle, fill=white, text=black, draw, thick, scale=1, minimum size=0.6cm, inner sep=1.5pt]

    \tikzstyle{blacknode} = [circle, fill=gray, draw, thick, scale=1, minimum size=0.2cm, inner sep=1.5pt]
    \tikzstyle{whitenode} = [circle, fill=white, draw, thick, scale=1, minimum size=0.2cm, inner sep=1.5pt]

    \tikzstyle{hugewhitenode} = [circle, fill=white, text=black, draw, thick, scale=1, minimum size=1.5cm, inner sep=1.5pt, font=\large]
    \tikzstyle{directed} = [color=black, arrows=- triangle 45]

    \tikzmath{\yunit = 0.8;}
    \definecolor{myblue}{RGB}{5,113,176}
    \definecolor{mypurple}{RGB}{123,50,148}
    \definecolor{myred}{RGB}{202,0,32}
    
    \tikzset{
        old inner xsep/.estore in=\oldinnerxsep,
        old inner ysep/.estore in=\oldinnerysep,
        double circle/.style 2 args={
            circle,
            old inner xsep=\pgfkeysvalueof{/pgf/inner xsep},
            old inner ysep=\pgfkeysvalueof{/pgf/inner ysep},
            /pgf/inner xsep=\oldinnerxsep+#1,
            /pgf/inner ysep=\oldinnerysep+#1,
            alias=sourcenode,
            append after command={
            let     \p1 = (sourcenode.center),
                    \p2 = (sourcenode.east),
                    \n1 = {\x2-\x1-#1-0.5*\pgflinewidth}
            in
                node [inner sep=0pt, draw, circle, minimum width=2*\n1,at=(\p1),#2] {}
            }
        },
        double circle/.default={-3pt}{black}
    }

    \centering
    \begin{minipage}[m]{\linewidth}
        \vspace{0pt}
        \centering
        \begin{tikzpicture}
            \node[bigblacknode] (x) at (0, 0) {$y$};
            \node[bigblacknode] (y) at (0, 2 * \yunit) {$x$};
            \node[bigwhitenode] (a1) at (1.6, 0) {$+$};
            \node[bigwhitenode] (a2) at (1.6, 2 * \yunit) {$+$};
            \node[bigwhitenode] (a3) at (3.2, \yunit) {$+$};
            \node[bigwhitenode] (m1) at (3.2, 2 * \yunit) {$\times$};
            \node[bigwhitenode] (m2) at (4.8, 0) {$\times$};
            \node[bigwhitenode, double circle] (out) at (6.4, \yunit) {$+$};

            \draw[line width=0.4mm,directed,draw=myblue] (y) -- (a1);
            \draw[line width=0.4mm,directed,draw=myred] (y) -- (a2);
            \draw[line width=0.4mm,directed] (x) -- (a2);
            \draw[line width=0.4mm,directed,draw=myblue] (x) -- (a1);
            \draw[line width=0.4mm,directed,draw=myred] (x) -- (m1);
            \draw[line width=0.4mm,directed,draw=myblue] (a1) -- (a3);
            \draw[line width=0.4mm,directed,draw=myblue] (a1) -- (m2);
            \draw[line width=0.4mm,directed,draw=myred] (a2) -- (m1);
            \draw[line width=0.4mm,directed,draw=myblue] (a3) -- (m2);
            \draw[line width=0.4mm,directed,draw=myred] (m1) -- (out);
            \draw[line width=0.4mm,directed,draw=myblue] (m2) -- (out);

            \node () at (1.6, 2 * \yunit + 0.5) {$x+y$};
            \node () at (3.2, 2 * \yunit + 0.5) {$(x+y)y$};
            \node () at (1.6, -0.6) {$x+y$};
            \node () at (4.0, \yunit) {$x+y$};
            \node () at (4.8, -0.6) {$(x+y)^2$};
            \node () at (6.0, \yunit + 0.7) {$(x+y)y+(x+y)^2$};
        \end{tikzpicture}
    \end{minipage}
    \caption{%
      Tree (red) and non-tree (blue) certificates for the same circuit,
      drawn with polynomials.
    }
    \label{fig:example-certificates}
\end{figure}

We prove~\Cref{lem:williams} by showing that the presence of multilinear monomials with coefficient $1$
in the fingerprint polynomial is directly related to the presence of
tree-shaped substructures in the circuit.
We state our result in \Cref{lem:tree-cert} and defer\footnote{
  Proofs of results marked with $\bigstar$ can be found in the appendix.
}
the proof to \Cref{sec:proof-tree-cert}. First, we need some machinery.

Given a circuit $\C$ with single output $r$ such that
$P_\C(X,A)$ contains a multilinear monomial,
we define a \emph{certificate} $\ccert$ to be a minimal sub-circuit of $\C$ with the same output node
such that $P_{\ccert}(X,A)$ contains a multilinear monomial, and
every multiplication gate in~$\ccert$ takes the same inputs in~$\C$.
In other words, we require $N_{\ccert}^-(v)=N_{\C}^-(v)$ for every multiplication gate~$v \in V(\ccert)$.
We say a certificate $\ccert$ is a \emph{tree certificate} if the underlying graph of $\ccert$ is a tree.
\Cref{fig:example-certificates} illustrates two certificates for the same circuit,
where the red one is a tree certificate.
In several lemmas, for simplicity, we assume that $\ccert$ does not contain scalar inputs\footnote{%
In linear time, one can identify the output nodes whose polynomial contains a constant term and transform $\C$ into an equivalent scalar-free sub-circuit of $\C$ (\Cref{lem:preprocess-scalar}).},
and call it \emph{scalar-free}.
We now state our characterization.\looseness-1

\begin{restatable}[$\bigstar$]{lemma}{treecertlemma}\label{lem:tree-cert}
  Let $\C$ be a scalar-free circuit.
  For a node $v \in V(\C)$, 
  $P_{\C[v]}(X,A)$ contains a multilinear monomial
  of coefficient $1$
  if and only if
  there exists a tree certificate for $\C$ with output node $v$.
\end{restatable}

It is simple to verify that the circuit in \Cref{fig:example-circuit}
does \emph{not} contain a tree certificate since any certificate must include the cycle.
The following lemma helps to show the existence of tree certificates by construction.

\begin{restatable}[$\bigstar$]{lemma}{lemsimplecircuit}\label{lem:simple-circuit}
  If every multiplication gate in a scalar-free circuit $\C$ has at most one non-variable in-neighbor,
  then every certificate in $\C$ is a tree certificate.
\end{restatable}

\Cref{lem:williams-poly} requires 
the coefficients of \emph{all} multilinear monomials in $P_\C(X,A)$ to be $1$,
but it is sufficient to have \emph{at least one} multilinear monomial in $P_\C(X,A)$ with coefficient $1$.
Hence, we relax \Cref{lem:williams-poly} as follows.\looseness=-1

\begin{lemma}\label{lem:williams-poly-weak}
  Let $\C$ be a connected arithmetic circuit with $m$ edges over a set $X$ of variables,
  and $k$ be an integer.
  Suppose either \emph{there exists} a multilinear monomial
  in 
  $P_\C(X,A)$ whose coefficient is~$1$,
  or $P_\C(X)$ does not contain a multilinear monomial of degree at most~$k$.
  Then, there exists a randomized $\tilde{\Oh}(2^k k m)$-time $\tilde{\Oh}(m + k|X|)$-space
  algorithm for \prob{Multilinear Detection} with $(\C, X, k)$.
  This is a one-sided error Monte Carlo algorithm with a constant success probability.
\end{lemma}

\begin{proof}
  We use an algorithm by Koutis for \prob{Odd Multilinear $k$-Term} \cite{koutis2008fasteralgebraic},
  where we want to decide if the polynomial represented by an arithmetic circuit
  contains a multilinear monomial of degree at most $k$ with odd coefficient.
  If $P_\C(X)$ does not contain a multilinear monomial of degree at most $k$,
  then the algorithm always decides correctly.
  Suppose $(\C,X,k)$ is a yes-instance.
  Then, by definition, if $P_\C(X,A)$ contains a multilinear monomial of coefficient $1$,
  then $P_\C(X,A)$ contains a multilinear monomial with odd coefficient;
  notice that not \emph{all} multilinear monomials must have coefficient $1$.
  The algorithm works in $\Oh(2^k(k|X| + T))$ time and $\Oh(k|X| + S)$ space,
  where $T$ and $S$ are the time and the space
  taken for evaluating the circuit over the integers modulo~$2^{k+1}$, respectively.

  In our case, each node stores an $\Oh(\log k)$-size vector of integers up to $2^{k+1}$,
  representing the coefficients of a polynomial with degree $\Oh(\log k)$.
  This requires $\tilde{\Oh}(k)$ bits of information.
  The most time-consuming operation is multiplication of these coefficient vectors,
  which can be done in $\tilde{\Oh}(k)$ time with a Fast-Fourier-Transform style algorithm~\cite{williams2009findingpaths}.
  Hence, $T \in \tilde{\Oh}(k m)$ and the running time is $\Oh(2^k(k|X|+T)) \subseteq \tilde{\Oh}(2^k(k|X|+k m))=\tilde{\Oh}(2^k k m)$.

  Storing the circuit and fingerprint polynomial requires $\tilde{\Oh}(m)$ space and
  each computation requires $\tilde{\Oh}(k)$ space at a time.
  Hence, $S \in \tilde{\Oh}(m+k)$ and $\Oh(k|X|+S) \subseteq \tilde{\Oh}(m + k|X|)$.
\end{proof}

We complete our proof of \Cref{lem:williams} in \Cref{sec:proof-tree-cert}
via a transformation of a general circuit into a circuit
where every certificate is a tree certificate.\looseness=-1

\section{\MLD: Solution Recovery}\label{sec:solution-recovery-mld}


We say a fingerprint circuit $\C$ over $(X,A)$ is \emph{recoverable}
with respect to an output node $r$ and an integer $k$
if $P_{\C[r]}(X,A)$ contains a multilinear monomial
of degree at most $k$ with coefficient $1$.
Once we have determined that $\C$ is recoverable,
we want to construct a solution by finding a tree certificate of $\C$
with the single output $r$.\looseness-1

We designed two algorithms for finding a tree certificate: \rmc and \rlv.
The basic idea common in both algorithms is backtracking from the output node of a circuit.
When seeing a multiplication gate, we keep all in-edges.
For an addition gate, we use binary search to find exactly one in-edge
that is included in a tree certificate.
Specifically, for every addition gate $v$ encountered during the solution recovery
process, let $E'$ be the set of in-edges of $v$, i.e. $E'=\{uv \colon u \in N_{\C}^-(v)\}$.
Then, we say a partition $(A,B)$ of $E'$ is a \emph{balanced partition of the in-edges of~$v$}
if $0 \leq |A|-|B| \leq 1$.
We then run an algorithm for \MLDshort with either $\C - A$ or $\C - B$
to decide which edges to keep.
We present our pseudocode and proofs of the following in \Cref{sec:tree-cert-alg}.

\begin{restatable}[$\bigstar$ \rmc]{lemma}{recoverymclemma}\label{lem:recovery-mc}
  Let $\C$ be a scalar-free connected recoverable circuit of $m$ edges with respect to
  degree $k$ and output $r$.
  Also, let $\alpha$ be the maximum number of addition gates in the tree certificates of $\C$.
  %
  %
  There exists a one-sided error Monte Carlo algorithm
  that finds a tree certificate of $\C$
  in $\tilde{\Oh}(2^k \alpha m)$ time with a constant success probability.
\end{restatable}

\begin{restatable}[$\bigstar$ \rlv]{lemma}{recoverylvlemma}\label{lem:recovery-lv}
  Let $\C$ be a scalar-free connected recoverable circuit of $m$ edges with respect to
  degree $k$ and output $r$.
  Also, let $\alpha$ be the maximum number of addition gates in the tree certificates of $\C$.
  %
  %
  There exists a Las Vegas algorithm
  that finds a tree certificate of $\C$
  with an expected running time of $\tilde{\Oh}(2^k \alpha m)$.\looseness=-1
\end{restatable}

The running times of the algorithms above depend on $\alpha$,
the number of addition gates in tree certificates.
We note that Algorithm \textnormal{\textsf{SimplifyCircuit}} in \Cref{sec:tree-cert-alg} transforms a scalar-free recoverable circuit
to another scalar-free recoverable instance with at most $\Oh(k)$ addition gates,
without increasing the number of nodes (the number of edges may increase).


\begin{restatable}[$\bigstar$]{lemma}{recoverypreprocess}\label{lem:recovery-preprocess}
  \simplifycircuit transforms a scalar-free recoverable circuit~$\C$
  of $n$ nodes with respect to degree~$k$ and output~$r$
  into a scalar-free recoverable circuit $\C'$ of $n' \leq n$ nodes
  such that every tree certificate of $\C'$ contains at most $(2k-1)$ addition gates.
  The algorithm runs in $\Oh(n^2)$-time.
\end{restatable}

The other direction, converting a tree certificate of~$\C'$
into a tree certificate of $\C$, is not trivial.

\section{Algorithm for \texorpdfstring{\gi{}}{}}\label{sec:gi-algorithm}

The following describes a high-level algorithm \algipa (ALGebraic Inspection Planning Algorithm)
for \gi{}.
There are three key subroutines: (1) circuit construction, (2) search, and (3) solution recovery.
We designed and engineered several approaches to each, described in 
Sections \ref{sec:circuit}, \ref{sec:search-strategies}, and \ref{sec:solution-recovery-mld} and \ref{sec:two-phase}, respectively.
%

\br
\noindent \textbf{Algorithm \algipa{}:}

\noindent \textit{Input:} 
A complete metric graph $G=(V,E)$,
a color set $\colorset$,
an edge-weight function $w\colon E \to \R_{\geq 0}$,
a vertex-coloring function $\col\colon V \to 2^\colorset$,
a vertex $s \in V$ such that $\col(s)=\emptyset$,
a number $t$ of colors to collect,
and a failure count threshold $\theta \in \N$.

\noindent \textit{Output:}
A minimum-weight closed walk in $G$,
starting at $s$ and collecting at least $t$ colors.
%
\begin{enumerate}[label=(\arabic*)]
  \item \label{algipa:step:1} \emph{Finding bounds.}
  Using the algorithms from Mizutani \emph{et al.}~\cite{mizutani2024leveraging},
  find lower ($\lo$) and upper bounds ($\hi$) for the solution weight.
  If the lower bound is fractional, then round it up to the nearest integer.

  \item \label{algipa:step:2} \emph{Search for the optimal weight.}
  Find the minimum weight~$\tilde{\ell}$
  such that there exists a closed walk from~$s$ with weight~$\tilde{\ell}$,
  collecting at least $t$ colors.
  We repeat the following steps until finding the optimal weight.

  \begin{enumerate}
    \item \emph{Construction of an arithmetic circuit.}
    Construct an arithmetic circuit
    for a target weight $\ell$ ($\lo \leq \ell \leq \hi$).
    %
    %

    \item \emph{Evaluation of the arithmetic circuit.}
    Solve \MLDshort for the constructed arithmetic circuit.
    If the output for $\ell$ contains a multilinear monomial,
    then we can immediately conclude that $\ell$ is feasible
    and update $\hi$.
    Otherwise, we repeatedly solve \MLDshort $\theta$ times
    until we conclude that $\ell$ is infeasible and update $\lo$.\looseness-1
  \end{enumerate}

  \item \label{algipa:step:3} \emph{Solution recovery.}
  Recover and output a solution walk with weight $\tilde{\ell}$.
  This can be done by reconstructing an arithmetic circuit for $\tilde{\ell}$,
  obtaining a tree certificate $\ccert$ for it,
  and constructing a walk from vertex $s$ in $G$ based on $\ccert$.
\end{enumerate}

Let $\lambda \in \R$ be a \emph{scaling factor}.
We perform the following preprocessing steps in our implementation.
Given a graph $G=(V,E)$, first create the transitive closure of $G$
by computing all-pairs shortest paths.
Remove all vertices $v \in V \setminus \{s\}$
that are unreachable from $s$ or have no colors (i.e. $\col(v)=\emptyset$).
For every edge $e$, update its edge weight to $\lambda \cdot w(e)$
and round to the nearest integer%
\footnote{For accuracy, round weights on the transitive closure.}.
Again, compute all-pairs shortest paths to make $G$ a metric graph%
\footnote{Rounding may turn $G$ into a non-metric graph.}.
The solution quality may incur a penalty based on rounding errors.

Given a solution walk $W$ from \algipa,
we simply replace every edge~$uv$ in~$W$ with any shortest $u$-$v$ path in $G$.
The result is a solution for \gi{}.

We prove in Theorem \ref{thm:algebra} that \algipa is a randomized $\Oh^*(2^t)$-time and polynomial space algorithm.

\section{Circuit Construction}\label{sec:circuit}

Here, we present four constructions for solving \gi via \MLD:
\cnaive, \cstandard, \ccompact, and \csemi.
Each circuit consists of the following nodes:

\begin{itemize}
  \item \emph{Variables}: Variable node $x_c$ for each color $c \in \colorset$.

  \item \emph{Internal nodes}: 
  We conceptually create $t$ computational layers corresponding to the degree of a polynomial.
  Each layer contains two types of nodes: \emph{transmitters} and \emph{receivers}.
  A transmitter, denoted by $T_{t',v,d}$ or $T_{t',v,d,i}$, is a gate that transfers information to the next layer
  and is identified by
  a layer $1 \leq t' \leq t$,
  vertex $v \in V \setminus \{s\}$,
  the weight $d$ of a walk from $s$ to $v$,
  and (optionally) an index~$i$.

  A receiver, denoted by $R_*$ (indices vary with construction type),
  is a gate that receives information from the previous layer
  and sends it to the transmitter in the same layer.

  \item \emph{Output nodes}:
  The construction of output nodes (sinks that matter)
  depends on the search algorithm,
  but in every case they aggregate information
  from the transmitters in the last layer, i.e. layer~$t$.
    
  If using \sunified (see~\Cref{sec:unified-search}),
  then there are addition nodes $O_{\ell}$ for every target value $\ell$
  ($\ell_{\text{lo}} \leq \ell \leq \ell_\text{hi}$)
  as output nodes, and the edge from $T_{t,v,d}$ in layer~$t$ to $O_\ell$
  exists if $\ell = d + w(v,s)$.
  
  If the search algorithm is \sstandard or \sprob (\Cref{sec:standard-binary-search,sec:prob-binary-search}),
  then a target weight $\ell$ is given when creating a circuit.
  In this case, an addition node $O_\ell$ is the only output node,
  and the edge from $T_{t,d,i}$ in layer $t$ to $O_\ell$
  exists if $\ell_\text{lo} \leq d+w(v,s) \leq \ell$.

  \item \emph{Auxiliary nodes}:
  \ccompact and \csemi (\Cref{sec:compact-circuit,sec:semi-compact-circuit}) have another set of nodes
  located in-between variable nodes and computational layers.
\end{itemize}

\begin{observation}
  There are $\colorsetsize$ variable nodes
  and at most ($\ell_\textnormal{hi} - \ell_\textnormal{lo}+1$) output nodes
  for all constructions.
\end{observation}

Now we present how we construct computational layers.
We then analyze the size of each circuit and its correctness.
We argue that a construction is \emph{correct} when
the following conditions are met:
(1) the circuit contains a tree certificate for \MLDshort
with degree at most~$t$ and the output node corresponding to objective $\ell$
if and only if there exists a solution walk
with weight at most~$\ell$
in an input graph for \algipa, and
(2) every tree certificate contains at most $\Oh(t)$ addition nodes.
In the following arguments, we write $k$ for $\colorsetsize$,
and set $w(v,v)=0$.

\subsection{\cnaive{}.}
In this construction, transmitters are indexed
by a pair of a vertex and one of its colors,
i.e., we split each color of a vertex into a distinct entity.

For the first layer, create a multiplication gate
$T_{1,v,w(s,v),c}$ as a transmitter for every vertex $v$ for every color $c \in \col(v)$
if $w(s,v) \leq \ell$.
Every transmitter in the first layer has only one input $x_c$.

In other layers $1 < t' \leq t$, proceed as follows.
For every vertex $v$ for every color $c \in \col(v)$,
and for every transmitter~$T_{t'-1,v',d',c'}$ in the previous layer,
let $d=d'+w(v',v)$.
We continue only if $d \leq \ell$ and $c \neq c'$.

First, create a new multiplication gate $r$ as a receiver taking
$T_{t'-1,v',d',c'}$ and $x_c$ as input.
Next, create an addition gate $T_{t',v,d,c}$ as a transmitter if this does not exist.
Finally, add an edge from $r$ to the transmitter $T_{t',v,d,c}$.
Notice that each receiver has $2$ in-neighbors, and
each transmitter has at most $k(n-2)$ in-neighbors.\looseness=-1

\begin{restatable}[$\bigstar$]{lemma}{lemconstructnaive}\label{lem:construct-naive}
  \cnaive is correct and creates a circuit of $\Oh(\ell_\textnormal{hi} t k^2 n^2)$ nodes
  and $\Oh(\ell_\textnormal{hi} t k^2 n^2)$ edges.
\end{restatable}


\begin{figure*}[h]
    \pgfdeclarelayer{bg}
    \pgfsetlayers{bg, main}

    \tikzstyle{bigblacknode} = [circle, fill=gray, text=white, draw, thick, scale=1, minimum size=0.6cm, inner sep=1.5pt]
    \tikzstyle{bigwhitenode} = [circle, fill=white, text=black, draw, thick, scale=1, minimum size=0.6cm, inner sep=1.5pt]

    \tikzstyle{blacknode} = [circle, fill=gray, text=white, draw, thick, scale=1, minimum size=0.2cm, inner sep=1.5pt]
    \tikzstyle{whitenode} = [circle, fill=white, draw, thick, scale=1, minimum size=0.2cm, inner sep=0pt]

    \tikzstyle{hugewhitenode} = [circle, fill=white, text=black, draw, thick, scale=1, minimum size=1.5cm, inner sep=1.5pt, font=\large]
    \tikzstyle{directed} = [color=gray, arrows=- triangle 45]
    \tikzstyle{thickedge} = [color=black, line width=0.8mm, arrows=-{Latex[length=2mm,width=3mm]}]

    \definecolor{mygreen}{RGB}{0,136,55}
    \definecolor{myblue}{RGB}{5,113,176}
    \definecolor{mypurple}{RGB}{123,50,148}
    \definecolor{myred}{RGB}{202,0,32}

    \tikzmath{\offset = 0.2;}
    \tikzmath{\layershrink = 0.3;}
    \tikzmath{\vshrink = 0.1;}

    \tikzset{
        old inner xsep/.estore in=\oldinnerxsep,
        old inner ysep/.estore in=\oldinnerysep,
        double circle/.style 2 args={
            circle,
            old inner xsep=\pgfkeysvalueof{/pgf/inner xsep},
            old inner ysep=\pgfkeysvalueof{/pgf/inner ysep},
            /pgf/inner xsep=\oldinnerxsep+#1,
            /pgf/inner ysep=\oldinnerysep+#1,
            alias=sourcenode,
            append after command={
            let     \p1 = (sourcenode.center),
                    \p2 = (sourcenode.east),
                    \n1 = {\x2-\x1-#1-0.5*\pgflinewidth}
            in
                node [inner sep=0pt, draw, circle, minimum width=2*\n1,at=(\p1),#2] {}
            }
        },
        double circle/.default={-3pt}{myred}
    }

    \begin{minipage}[m]{.98\linewidth}
        \centering
        \begin{tikzpicture}
          \begin{pgfonlayer}{main}
            \node[bigblacknode] (s) at (-8.5, 8-2) {$s$};
            \node[bigwhitenode] (u) at (-7, 8-2) {$u$};
            \node[bigwhitenode] (v) at (-7, 6.5-2) {$v$};
            \node[bigwhitenode] (w) at (-8.5, 6.5-2) {$w$};
            \draw (s) -- (u) node[midway, above] {$2$};
            \draw (s) -- (v) node[midway, xshift=-10, yshift=15] {$4$};
            \draw (s) -- (w) node[midway, left, xshift=2] {$3$};
            \draw (u) -- (v) node[midway, right, xshift=-2] {$2$};
            \draw (u) -- (w) node[midway, xshift=-10, yshift=-4]{$2$};
            \draw (v) -- (w) node[midway, below]{$1$};
            \draw (u) node[above, yshift=6] {$\{c_1,c_2\}$};
            \draw (w) node[below, yshift=-5] {$\{c_3\}$};
            \draw (v) node[below, yshift=-5] {$\{c_2,c_3\}$};
            \node[] () at (-7.6, 3.3) {Graph instance $G$};

            \node[] () at (-5.8 + 3 * \layershrink, 3) {$2$};
            \node[] () at (-5.8 + 3 * \layershrink, 4) {$3$};
            \node[] () at (-5.8 + 3 * \layershrink, 5) {$4$};
            \node[] () at (-5.8 + 3 * \layershrink, 6) {$5$};
            \node[] () at (-5.8 + 3 * \layershrink, 7) {$6$};
            \node[] () at (-5.8 + 3 * \layershrink, 8) {$7$};

            \draw[dashed, gray] (-5.8 + 3 * \layershrink, 2.5) -- (7.4, 2.5);
            \draw[dashed, gray] (-5.8 + 3 * \layershrink, 3.5) -- (7.4, 3.5);
            \draw[dashed, gray] (-5.8 + 3 * \layershrink, 4.5) -- (7.4, 4.5);
            \draw[dashed, gray] (-5.8 + 3 * \layershrink, 5.5) -- (7.4, 5.5);
            \draw[dashed, gray] (-5.8 + 3 * \layershrink, 6.5) -- (7.4, 6.5);
            \draw[dashed, gray] (-5.8 + 3 * \layershrink, 7.5) -- (7.4, 7.5);
            \draw[dashed, gray] (-5.8 + 3 * \layershrink, 8.5) -- (7.4, 8.5);

            \draw[dashed, gray] (-4.5 + 2 * \layershrink, 1 + 4 * \vshrink) -- (-4.5 + 2 * \layershrink, 9);
            \draw[dashed, gray] (-3.5 + \layershrink    , 1 + 4 * \vshrink) -- (-3.5 + \layershrink, 9);
            \draw[dashed, gray] (-0.5                   , 1 + 4 * \vshrink) -- (-0.5, 9);
            \draw[dashed, gray] (+0.5                   , 1 + 4 * \vshrink) -- (+0.5, 9);
            \draw[dashed, gray] (+3.5                   , 1 + 4 * \vshrink) -- (+3.5, 9);
            \draw[dashed, gray] (+4.5                   , 1 + 4 * \vshrink) -- (+4.5, 9);

            \node[blacknode] (c1) at (-1, 0 + 7 * \vshrink) {$x_{c_1}$};
            \node[blacknode] (c2) at (0,  0 + 7 * \vshrink) {$x_{c_2}$};
            \node[blacknode] (c3) at (1,  0 + 7 * \vshrink) {$x_{c_3}$};

            \node[whitenode, draw=mygreen] (a11) at (-5 + 2.5 * \layershrink, 1.5 + 3 * \vshrink) {$+$};
            \node[whitenode, draw=mygreen] (a12) at (-4 + 1.5 * \layershrink, 1.5 + 3 * \vshrink) {$+$};
            \node[whitenode, draw=mygreen] (a13) at (-3 + 0.5 * \layershrink, 1.5 + 3 * \vshrink) {$+$};
            \node[whitenode, draw=mygreen] (a21) at (-1                     , 1.5 + 3 * \vshrink) {$+$};
            \node[whitenode, draw=mygreen] (a22) at (+0                     , 1.5 + 3 * \vshrink) {$+$};
            \node[whitenode, draw=mygreen] (a23) at (+1                     , 1.5 + 3 * \vshrink) {$+$};
            \node[whitenode, draw=mygreen] (a31) at (3                      , 1.5 + 3 * \vshrink) {$+$};
            \node[whitenode, draw=mygreen] (a32) at (4                      , 1.5 + 3 * \vshrink) {$+$};
            \node[whitenode, draw=mygreen] (a33) at (5                      , 1.5 + 3 * \vshrink) {$+$};

            \node[whitenode, draw=myblue] (y112) at (-5 + 2.5 * \layershrink, 3) {$\times$};
            \node[whitenode, draw=myblue] (y124) at (-4 + 1.5 * \layershrink, 5) {$\times$};
            \node[whitenode, draw=myblue] (y133) at (-3 + 0.5 * \layershrink, 4) {$\times$};

            \node[whitenode, draw=mypurple] (z212) at (-1 - \offset, 3 - \offset) {$+$};
            \node[whitenode, draw=myblue] (y212) at (-1 + \offset, 3 + \offset) {$\times$};
            \node[whitenode, draw=mypurple] (z215) at (-1 - \offset, 6 - \offset) {$+$};
            \node[whitenode, draw=myblue] (y215) at (-1 + \offset, 6 + \offset) {$\times$};
            \node[whitenode, draw=mypurple] (z216) at (-1 - \offset, 7 - \offset) {$+$};
            \node[whitenode, draw=myblue] (y216) at (-1 + \offset, 7 + \offset) {$\times$};

            \node[whitenode, draw=mypurple] (z224) at ( 0 - \offset, 5 - \offset) {$+$};
            \node[whitenode, draw=myblue] (y224) at ( 0 + \offset, 5 + \offset) {$\times$};

            \node[whitenode, draw=mypurple] (z233) at ( 1 - \offset, 4 - \offset) {$+$};
            \node[whitenode, draw=myblue] (y233) at (1  + \offset, 4 + \offset) {$\times$};
            \node[whitenode, draw=mypurple] (z234) at ( 1 - \offset, 5 - \offset) {$+$};
            \node[whitenode, draw=myblue] (y234) at (1  + \offset, 5 + \offset) {$\times$};
            \node[whitenode, draw=mypurple] (z235) at ( 1 - \offset, 6 - \offset) {$+$};
            \node[whitenode, draw=myblue] (y235) at (1  + \offset, 6 + \offset) {$\times$};

            \node[whitenode, draw=mypurple] (z312) at (3 - \offset, 3 - \offset) {$+$};
            \node[whitenode, draw=myblue] (y312) at (3 + \offset, 3 + \offset) {$\times$};
            \node[whitenode, draw=mypurple] (z315) at (3 - \offset, 6 - \offset) {$+$};
            \node[whitenode, draw=myblue] (y315) at (3 + \offset, 6 + \offset) {$\times$};
            \node[whitenode, draw=mypurple] (z316) at (3 - \offset, 7 - \offset) {$+$};
            \node[whitenode, draw=myblue] (y316) at (3 + \offset, 7 + \offset) {$\times$};
            \node[whitenode, draw=mypurple] (z317) at (3 - \offset, 8 - \offset) {$+$};
            \node[whitenode, draw=myblue] (y317) at (3 + \offset, 8 + \offset) {$\times$};

            \node[whitenode, draw=mypurple] (z324) at (4 - \offset, 5 - \offset) {$+$};
            \node[whitenode, draw=myblue] (y324) at (4 + \offset, 5 + \offset) {$\times$};
            \node[whitenode, draw=mypurple] (z325) at (4 - \offset, 6 - \offset) {$+$};
            \node[whitenode, draw=myblue] (y325) at (4 + \offset, 6 + \offset) {$\times$};
            \node[whitenode, draw=mypurple] (z326) at (4 - \offset, 7 - \offset) {$+$};
            \node[whitenode, draw=myblue] (y326) at (4 + \offset, 7 + \offset) {$\times$};
            \node[whitenode, draw=mypurple] (z327) at (4 - \offset, 8 - \offset) {$+$};
            \node[whitenode, draw=myblue] (y327) at (4 + \offset, 8 + \offset) {$\times$};

            \node[whitenode, draw=mypurple] (z333) at (5 - \offset, 4 - \offset) {$+$};
            \node[whitenode, draw=myblue] (y333) at (5 + \offset, 4 + \offset) {$\times$};
            \node[whitenode, draw=mypurple] (z334) at (5 - \offset, 5 - \offset) {$+$};
            \node[whitenode, draw=myblue] (y334) at (5 + \offset, 5 + \offset) {$\times$};
            \node[whitenode, draw=mypurple] (z335) at (5 - \offset, 6 - \offset) {$+$};
            \node[whitenode, draw=myblue] (y335) at (5 + \offset, 6 + \offset) {$\times$};
            \node[whitenode, draw=mypurple] (z337) at (5 - \offset, 8 - \offset) {$+$};
            \node[whitenode, draw=myblue] (y337) at (5 + \offset, 8 + \offset) {$\times$};

            \node[bigwhitenode,double circle, draw=myred] (r7) at (6.5, 8) {$+$};
            \node[bigwhitenode,double circle, draw=myred] (r6) at (6.5, 7) {$+$};
            \node[bigwhitenode,double circle, draw=myred] (r5) at (6.5, 6) {$+$};
            \node[bigwhitenode,double circle, draw=myred] (r4) at (6.5, 5) {$+$};
            \draw (r7) node[right,xshift=10] {\color{myred}$O_7$};
            \draw (r6) node[right,xshift=10] {\color{myred}$O_6$};
            \draw (r5) node[right,xshift=10] {\color{myred}$O_5$};
            \draw (r4) node[right,xshift=10] {\color{myred}$O_4$};

            \draw[directed,thickedge] (c1) -- (a11);
            \draw[directed,thickedge] (c2) -- (a21);
            \draw[directed,thickedge] (c3) -- (a33);
            \draw[directed,thickedge] (a11) -- (y112);
            \draw[directed,thickedge] (a21) -- (y212);
            \draw[directed,thickedge] (y112) -- (z212);
            \draw[directed,thickedge] (y212) -- (z334);
            \draw[directed,thickedge] (z212) -- (y212);
            \draw[directed,thickedge] (z334) -- (y334);
            \draw[directed,thickedge] (a33) -- (y334);
            \draw[directed,thickedge] (y334) -- (r7);
          \end{pgfonlayer}

          \begin{pgfonlayer}{bg}
            \draw[directed] (c1) -- (a21);
            \draw[directed] (c1) -- (a31);
            
            \draw[directed] (c2) -- (a11);
            \draw[directed] (c2) -- (a31);

            \draw[directed] (c2) -- (a12);
            \draw[directed] (c2) -- (a22);
            \draw[directed] (c2) -- (a32);

            \draw[directed] (c3) -- (a12);
            \draw[directed] (c3) -- (a22);
            \draw[directed] (c3) -- (a32);

            \draw[directed] (c3) -- (a13);
            \draw[directed] (c3) -- (a23);

            \draw[directed] (a12) -- (y124);
            \draw[directed] (a13) -- (y133);

            \draw[directed] (a21) -- (y215);
            \draw[directed] (a21) -- (y216);
            \draw[directed] (a22) -- (y224);
            \draw[directed] (a23) -- (y233);
            \draw[directed] (a23) -- (y234);
            \draw[directed] (a23) -- (y235);

            \draw[directed] (a31) -- (y312);
            \draw[directed] (a31) -- (y315);
            \draw[directed] (a31) -- (y316);
            \draw[directed] (a31) -- (y317);
            \draw[directed] (a32) -- (y324);
            \draw[directed] (a32) -- (y325);
            \draw[directed] (a32) -- (y326);
            \draw[directed] (a32) -- (y327);
            \draw[directed] (a33) -- (y333);
            \draw[directed] (a33) -- (y335);
            \draw[directed] (a33) -- (y337);

            \draw[directed] (y112) -- (z224);
            \draw[directed] (y112) -- (z234);

            \draw[directed] (y124) -- (z216);
            \draw[directed] (y124) -- (z224);
            \draw[directed] (y124) -- (z235);

            \draw[directed] (y133) -- (z215);
            \draw[directed] (y133) -- (z224);
            \draw[directed] (y133) -- (z233);

            \draw[directed] (z215) -- (y215);
            \draw[directed] (z216) -- (y216);
            \draw[directed] (z224) -- (y224);
            \draw[directed] (z233) -- (y233);
            \draw[directed] (z234) -- (y234);
            \draw[directed] (z235) -- (y235);

            \draw[directed] (y212) -- (z312);
            \draw[directed] (y212) -- (z324);

            \draw[directed] (y215) -- (z315);
            \draw[directed] (y215) -- (z327);
            \draw[directed] (y215) -- (z337);
            \draw[directed] (y216) -- (z316);

            \draw[directed] (y224) -- (z316);
            \draw[directed] (y224) -- (z324);
            \draw[directed] (y224) -- (z335);

            \draw[directed] (y233) -- (z315);
            \draw[directed] (y233) -- (z324);
            \draw[directed] (y233) -- (z333);
            \draw[directed] (y234) -- (z316);
            \draw[directed] (y234) -- (z325);
            \draw[directed] (y234) -- (z334);
            \draw[directed] (y235) -- (z317);
            \draw[directed] (y235) -- (z326);
            \draw[directed] (y235) -- (z335);

            \draw[directed] (z312) -- (y312);
            \draw[directed] (z315) -- (y315);
            \draw[directed] (z316) -- (y316);
            \draw[directed] (z317) -- (y317);
            \draw[directed] (z324) -- (y324);
            \draw[directed] (z325) -- (y325);
            \draw[directed] (z326) -- (y326);
            \draw[directed] (z327) -- (y327);
            \draw[directed] (z333) -- (y333);
            \draw[directed] (z335) -- (y335);
            \draw[directed] (z337) -- (y337);

            \draw[directed] (y312) -- (r4);
            \draw[directed] (y315) -- (r7);
            \draw[directed] (y333) -- (r6);

          \draw[rounded corners, gray] (-5.6 + 3 * \layershrink, 2.4) rectangle ++ (3.2 - 3 * \layershrink, 6.2);
          \draw[rounded corners, gray] (-1.6, 2.4) rectangle ++ (3.2, 6.2);
          \draw[rounded corners, gray] (2.4, 2.4) rectangle ++ (3.2, 6.2);
          \node[] at (-4 + 1.5 * \layershrink, 9.3) {Layer $1$};
          \node[] at ( 0, 9.3) {Layer $2$};
          \node[] at ( 4, 9.3) {Layer $3$};
          \node[] at ( 6.5, 9.2) {Output};
          \node[] at ( 6.5, 8.8) {nodes};
          \node[] at (-5 + 2.5 * \layershrink, 8.8) {$u$};
          \node[] at (-4 + 1.5 * \layershrink, 8.8) {$v$};
          \node[] at (-3 + 0.5 * \layershrink, 8.8) {$w$};
          \node[] at (-1, 8.8) {$u$};
          \node[] at ( 0, 8.8) {$v$};
          \node[] at ( 1, 8.8) {$w$};
          \node[] at ( 3, 8.8) {$u$};
          \node[] at ( 4, 8.8) {$v$};
          \node[] at ( 5, 8.8) {$w$};

          \node[rotate=90] at (-6.1 + 3 * \layershrink - 0.1, 5.5) {walk weight from $s$};

          \draw[rounded corners, mygreen] (-5.6 + 3 * \layershrink, 1 + 4 * \vshrink) rectangle ++ (3.2 - 3 * \layershrink, 1 - 2 * \vshrink);
          \draw[rounded corners, mygreen] (-1.6                   , 1 + 4 * \vshrink) rectangle ++ (3.2, 1 - 2 * \vshrink);
          \draw[rounded corners, mygreen] (2.4                    , 1 + 4 * \vshrink) rectangle ++ (3.2, 1 - 2 * \vshrink);
          \node[] at ( 5.1, 0.7 + 4 * \vshrink) {\color{mygreen} $a_{3,w}$};
          \node[] at ( -1, 7.5 + \offset) {\color{myblue}$T_{2,u,6}$};
          \node[] at (-1.5, 7.0 + \offset) {\color{mypurple}$R_{2,u,6}$};

          \node[] () at (-5.0, 0.7 + 4 * \vshrink) {Auxiliary nodes};
          \end{pgfonlayer}

        \end{tikzpicture}
    \end{minipage}
    \caption{
      An example of an arithmetic circuit (\ccompact with \sunified),
      encoding the graph instance $G$ with colors $\colorset=\{c_1,c_2,c_3\}$,
      illustrated on the left,
      with $t=3$.
      The circuit consists of variable nodes (gray) for each color in $\colorset$,
      auxiliary nodes (green),
      receivers (purple), transmitters (blue), and output nodes (red).
      Notice that each receiver/transmitter pair is identifiable by
      a layer, index ($V \setminus \{s\}$ for \ccompact), and walk weight from the starting vertex $s$.
      A tree certificate, corresponding to walk $(s,u,w,s)$ with weight $7$, is highlighted in bold.\looseness-1}

  \label{fig:construction}
\end{figure*}

\subsection{\cstandard{}.}
Intuitively, this type of circuit is built by switching addition and multiplication nodes
in the computational layers of \cnaive, as described in more detail below.
The indexing scheme is unchanged,
and the first layer is identical to that of \cnaive.

For layers $1 < t' \leq t$, proceed as follows.
For each vertex $v$, for each color $c \in \col(v)$,
and for each transmitter $T_{t'-1,v',d',c'}$ in the previous layer,
let $d=d'+w(v',v)$.
We continue only if $d \leq \ell$ and $c \neq c'$.

First, create an addition gate $R_{t',v,d,c}$ as a receiver if it does not exist.
Next, create a multiplication gate $T_{t',v,d,c}$ as a transmitter if it does not exist.
Lastly, add the following edges:
$(T_{t'-1,v',d',c'}, R_{t',v,d,c})$,
$(R_{t',v,d,c}, T_{t',v,d,c})$,
and $(x_c, T_{t',v,d,c})$.
Note that in this construction, a transmitter has at most $2$ in-neighbors,
and a receiver has at most $k(n-2)$ in-neighbors.

\begin{restatable}[$\bigstar$]{lemma}{lemconstructstandard}\label{lem:construct-standard}
  \cstandard is correct and creates a circuit of $\Oh(\ell_\textnormal{hi} t k n)$ nodes
  and $\Oh(\ell_\textnormal{hi} t k^2 n^2)$ edges.
\end{restatable}


\subsection{\ccompact{}.}\label{sec:compact-circuit}
This is designed to have an asymptotically smaller number of nodes than the others.
In this construction, we do not split vertex colors.
We have transmitters $T_{t',v,d}$, receivers $R_{t',v,d}$, and auxiliary nodes $a_{t',v}$
for layer $t'$, vertex $v \in V \setminus \{s\}$, and weight $d$.\looseness=-1

First, construct auxiliary nodes.
For each $1 \leq t' \leq t$
and for each vertex $v$,
add an addition gate $a_{t',v}$ and take as input $\{x_c \colon c \in \col(v)\}$.

For $t'=1$, add an edge from $a_v$ to $T_{1,v,w(s,v)}$ if $w(s,v) \leq \ell$.
For $1 < t' \leq t$, for each vertex $v$,
and for each transmitter $T_{t'-1,v',d'}$ in the previous layer,
let $d=d'+w(v',v)$.
Notice that it is possible that~$v=v'$.
Intuitively, this means collecting a new color without moving.
We add the following edges if $d \leq \ell$:
$(T_{t'-1,v',d'}, R_{t',v,d})$,
$(R_{t',v,d}, T_{t',v,d})$,
and $(a_{t',v}, T_{t',v,d})$.
\Cref{fig:construction} illustrates an example.

\begin{restatable}[$\bigstar$]{lemma}{lemconstructcompact}\label{lem:construct-compact}
  \ccompact is correct and creates a circuit of $\Oh(\hi tn + k)$ nodes
  and $\Oh(\hi tn^2 + tkn)$ edges.
\end{restatable}


\subsection{\csemi{}.}\label{sec:semi-compact-circuit}
This construction is similar to \cstandard,
but instead of splitting vertex colors into vertex-color tuples,
we keep track of vertex-multiplicity tuples.
Here, the multiplicity means how many colors are collected at the same vertex.

First, add addition gates $a_v^{(i)}$ as auxiliary nodes
that takes as input $\{x_c \colon c \in \col(v)\}$
for every vertex $v \in V \setminus \{s\}$
and every multiplicity $1 \leq i \leq \min\{t, |\col(v)|\}$.

The first layer contains multiplication gates $T_{1,v,d,1}$ as transmitters
that takes $a_v^{(1)}$ as the only input.
The other layers ($1 < t' \leq t$) contain addition gates $R_{t',v,d,i}$ as receivers
 and multiplication gates $T_{t',v,d,i}$ as transmitters.
$T_{t',v,d,i}$ takes two inputs, $R_{t',v,d,i}$ and $a_v^{(i)}$.

For every vertex $v \in V \setminus \{s\}$ add the following edges:
(1) $(T_{t'-1,v,d,i}, R_{t',v,d,i+1})$ for $i < \min\{t, |\col(v)|\}$, and
(2) $(T_{t'-1,v',d',i}, R_{t',v,d'+w(v',v),1})$
for $v' \neq v$, ${d'+w(v',v) \leq \ell}$, and $1 \leq i \leq \min\{t, |\col(v')|\}$.
The former represents collecting another color at the same vertex,
thus preserving the walk length and incrementing the multiplicity by one.
The latter represents moving to another vertex, thus
resetting the multiplicity to one.

\begin{restatable}[$\bigstar$]{lemma}{lemconstructsemicompact}\label{lem:construct-semicompact}
  \csemi is correct and creates a circuit of $\Oh(\ell_\textnormal{hi} t^2 n + k)$ nodes
  and $\Oh(\ell_\textnormal{hi} t^2 n^2 + tkn)$ edges.
\end{restatable}

We evaluate these constructions in \Cref{sec:circuit-type-experiment}.

\section{Search Strategies}\label{sec:search-strategies}

In this section, we describe three search algorithms.
We prove that each algorithm correctly finds the optimal weight
with probability $1-(1-p)^\theta$,
where $p$ is the constant success probability from \Cref{lem:williams-poly},
and $\theta$ is the failure count threshold greater than $p^{-1}$.
Also, we analyze the expected running time under simplistic assumptions:
the optimal weight $\tilde{\ell}$ is uniformly distributed between $\lo$ and $\hi$;
a single run of construction and evaluation of a circuit takes (non-decreasing) $f(\ell)$ time;
and the evaluation results in \true if $\tilde{\ell} \leq \ell$ with probability~$p$
and \false otherwise.
We write $T(\lo,\hi)$ for the expected running time with lower and upper bounds
$\lo$ and $\hi$, respectively.
We also define $\diff := \hi - \lo + 1$.

\subsection{\sstandard{}.}\label{sec:standard-binary-search}

We first implemented the standard binary search.
Given $\lo$ and $\hi$,
we examine the middle value $\ell = \lfloor (\ell_\text{lo} + \ell_\text{hi})/2 \rfloor$
to see if $\ell$ is feasible.
If $\ell$ is feasible,
then we update $\ell_\text{hi}$ to $\ell$.
On the other hand, if the circuit evaluates only to \false for
$\theta'$ trials, then we set $\ell_\text{lo}$ to $\ell + 1$.
Here we set $\theta'$ as a number in $\Oh(\theta \log\log(\diff))$ such that 
${(1-(1-p)^{\theta'})^{\lfloor \log_2(\diff) \rfloor} \geq 1 - (1-p)^\theta}$.
The algorithm terminates when $\ell_\text{lo}=\ell_\text{hi}$,
and outputs this number.

\begin{restatable}[$\bigstar$]{lemma}{lemSearchBS}\label{lem:search-bs}
  Algorithm \sstandard{} correctly finds the optimal weight
  in expected running time~$\tilde{\Oh}(\theta \cdot f(\hi))$
  with probability $1-(1-p)^\theta$.
\end{restatable}

\subsection{\sprob{}.}\label{sec:prob-binary-search}

The performance of the previous algorithm degrades when the optimal value is close to $\ell_\text{lo}$
because concluding that the value $\ell$ is infeasible requires $\theta$ evaluations of a circuit.
To mitigate this penalty, \sprob{} evaluates each circuit once at a time
and randomly goes higher \emph{before} concluding that the value is infeasible.
We set this probability $p'$ to $1-p/2$.
We maintain midpoints and failure counts as a stack.
See \Cref{alg:search-prob} in \Cref{sec:proof-search}.\looseness=-1

\begin{restatable}[$\bigstar$]{lemma}{lemSearchProbBS}\label{lem:search-prob-bs}
  Algorithm \sprob{} correctly finds the optimal weight
  in expected running time $\tilde{\Oh}((\theta + \diff) \cdot f(\hi))$
  with probability $1-(1-p)^\theta$.
\end{restatable}

\subsection{\sunified{}.}\label{sec:unified-search}

The previous approach aims to reduce the number of evaluations for infeasible circuits,
but in most cases, evaluating circuits with large $\ell$
(more time consuming than the evaluation of a circuit with small $\ell$)
multiple times is unavoidable.

We resolve this by tweaking a circuit to have multiple output nodes.
Now, we assume that a circuit returns $\textsf{out}(\ell) \in \{\true, \false\}$ for each $\ell_\text{lo} \leq  \ell \leq \ell_\text{hi}$.
As described in \Cref{sec:circuit} and illustrated in \Cref{fig:construction},
we create $(\hi-\lo+1)$ output nodes connected from the last layer of internal nodes.
Other nodes are the same as the other search strategies.

\sunified is shown in \Cref{alg:search-unified}.
Once the output for $\ell$ is evaluated to $\true$,
it is safe to update $\hi$ with $\ell$.
We then construct a smaller circuit for the new $\hi$,
which may speed up circuit evaluation.

\begin{algorithm2e}[h]
  \SetAlgoLined
  \KwInput{An instance $\mathcal{I}$ of \gi{},
  bounds of the optimal weight $\lo \leq \hi$,
  and a failure count threshold $\theta$.\\
  }
  \KwOutput{Optimal weight.}

  \BlankLine
  \For{$i \gets 1$ \KwTo $\theta$}{
    \If{$\lo = \hi$}{
      \KwRet{$\hi$}\;
    }
    Create a multi-output circuit $\C$ of $\mathcal{I}$ for all~$\ell$
    between $\lo$ and $\hi$, inclusive.\;
    Solve \MLDshort with~$(\C, k)$ and obtain results $\textsf{out}(\ell): \ell \mapsto \{\false, \true\}$.\;
    Let $\hi \gets \min\{\ell \colon \textsf{out}(\ell)=\true \}$.
  }
  \KwRet{$\hi$}\;
  \caption{\sunified}
  \label{alg:search-unified}
\end{algorithm2e}


\begin{restatable}[$\bigstar$]{lemma}{lemSearchUnified}\label{lem:search-unified}
  Algorithm \sunified{} correctly finds the optimal weight
  in expected running time $\Oh(\theta \cdot f(\hi))$
  with probability $1-(1-p)^\theta$.
\end{restatable}

We evaluate these search strategies in \Cref{sec:exp:search}.

\section{Two Phase Solution Recovery}\label{sec:two-phase}

We introduced solution recovery algorithms for general circuits in \Cref{sec:solution-recovery-mld}.
Once we find a tree certificate for the output node corresponding to weight $\ell$,
it is trivial to construct a solution walk collecting at least $t$ colors with
weight $\ell$ (see the proof of \Cref{lem:construct-compact}).
However, finding a tree certificate necessitates many circuit evaluations,
often requiring more time than that for finding the solution weight.
To cope with this issue, we propose \rtp, a solution recovery strategy specific to \gi.
In this section, we extend the definition of a walk to allow consecutively repeated vertices.
For example, $(v,v,u)$ can be considered as a walk,
and we treat the weight from $v$ to $v$ as $0$.

For an instance $(G=(V,E),\colorset,w,\chi,s,t)$ of \gi
such that $G$ is a complete metric graph,
we say an ordered set $C=\{c_1,\ldots,c_t\} \subseteq \colorset$ is
an \emph{optimal color order} if
there exists an optimal closed walk $W=(s, v_1, \ldots, v_t, s)$
with possibly (consecutively) repeated vertices in $G$
satisfying $c_i \in \chi(v_i)$ for each~$1 \leq i \leq t$.
We also say that the walk~$W$ is \emph{consistent with} the color order~$C$.
The following lemma guarantees that any instance has an optimal color order of size $t$
and an optimal walk consistent with that color order.

\begin{restatable}[$\bigstar$]{lemma}{lemconsistentwalk}\label{lem:consistent-walk}
  Let $(G=(V,E),\colorset,w,\chi,s,t)$ be an instance of \gi
  such that $G$ is a complete metric graph and $\colorsetsize \geq t$.
  Suppose $\ell$ is the weight of an optimal walk.
  Then, there exist an ordered color set~$C=\{c_1,\ldots,c_t\}$
  and a walk $W=(s,v_1,\ldots,v_t,s)$ with weight $\ell$ that is
  consistent with $C$.
\end{restatable}

Our strategy is based on the following observation:
if we fix an \emph{ordered} set of $t$ colors to be collected,
then we can recover a minimum-weight walk in polynomial time.\looseness=-1

The \rtp algorithm takes as input
a recoverable circuit for \gi with respect to an optimal weight $\ell$
and proceeds with the following two phases.
First, it finds an optimal color order from the given circuit.
Second, the algorithm constructs a minimum-weight walk $W$ in $G$ consistent with the obtained color order.
From \Cref{lem:consistent-walk}, $W$ must be an optimal walk with weight~$\ell$ and collect at least $t$ colors.

\vspace*{0.8em}
\subparagraph{Phase 1: Finding an optimal color order.}
We determine an optimal color order in reverse.
Intuitively,
instead of using binary search to identify the previous vertex in a walk,
we use it to identify the previous color.
This enables an asymptotic speed-up of $\log_2 n / \log_2 \colorsetsize$.

\begin{restatable}[$\bigstar$]{lemma}{lemtpcolorrecovery}\label{lem:tp-color-recovery}
  Let $\F$ be a recoverable circuit of $m$ edges for an instance $(G=(V,E),\colorset,w,\chi,s,t)$ of \gi
  where $G$ is a complete metric graph.
  There exists a Las Vegas algorithm that finds an optimal color order
  with expected running time $\tilde{\Oh}(2^t tm)$.
\end{restatable}

Specifics depend on the circuit constructions.
Take \ccompact, for example.
Let $Y$ be candidates for the final color collected and initialize it to $\colorset$.
We know that there must exist such a color $y \in \colorset$.
Now, we modify the circuit so that only the variables for $Y$ are directly fed into the last layer.
This can be done by removing edges from $X \setminus Y$ to the auxiliary nodes in the last layer,
where $X$ denotes the set of variable nodes.
We run binary search on $Y$ until $|Y|=1$, i.e., we find $y$.\looseness-1

Lastly, we remove all edges from $y$ to all the auxiliary nodes except those in the last layer.
This ensures that we collect ($t-1$) other colors before collecting color $y$.
We then proceed to the ($t-1$)-th layer,
and by continuing this process, we can determine an ordered set of $t$ colors
collected in an optimal walk.

\vspace*{0.8em}
\subparagraph{Phase 2: Reconstructing a walk from ordered colors.}

Given ordered colors $\{c_1,\ldots,c_t\}$, we can find a walk with the minimum weight in polynomial time
by a reduction to the classic \prob{Shortest Path} problem.

\begin{restatable}[$\bigstar$]{lemma}{lemtpwalkrecovery}\label{lem:tp-walk-recovery}
  Let $(G=(V,E),\colorset,w,\chi,s,t)$ be an instance of \gi
  where $G$ is an $n$-vertex complete metric graph.
  There exists a $\tilde{\Oh}(tn^2)$-time algorithm that,
  for a given color order $(c_1,\ldots,c_t)$,
  finds a walk $W=(s,v_1, \ldots, v_t, s)$ such that $c_i \in \chi(v_i)$ for each $1 \leq i \leq t$,
  and its weight is minimized.
\end{restatable}

Using \Cref{lem:tp-color-recovery} and \Cref{lem:tp-walk-recovery} together,
we obtain the following. We defer the proofs to \Cref{appendix:two-phase}.

\begin{restatable}{lemma}{lemtprecovery}\label{lem:tp-recovery}
  Let $\F$ be a recoverable circuit of $m$ edges for an instance $(G=(V,E),\colorset,w,\chi,s,t)$ of \gi
  where $G$ is a complete metric graph, and the target weight is $\ell$.
  There exists a Las Vegas algorithm that finds a closed walk from $s$ in $G$ with weight at most $\ell$,
  collecting at least $t$ colors,
  with an expected running time of $\tilde{\Oh}(2^t tm + tn^2)$.
\end{restatable}

\begin{figure*}[t]
  \centering
  \begin{subfigure}{0.48\textwidth}
    \includegraphics[width=\linewidth]{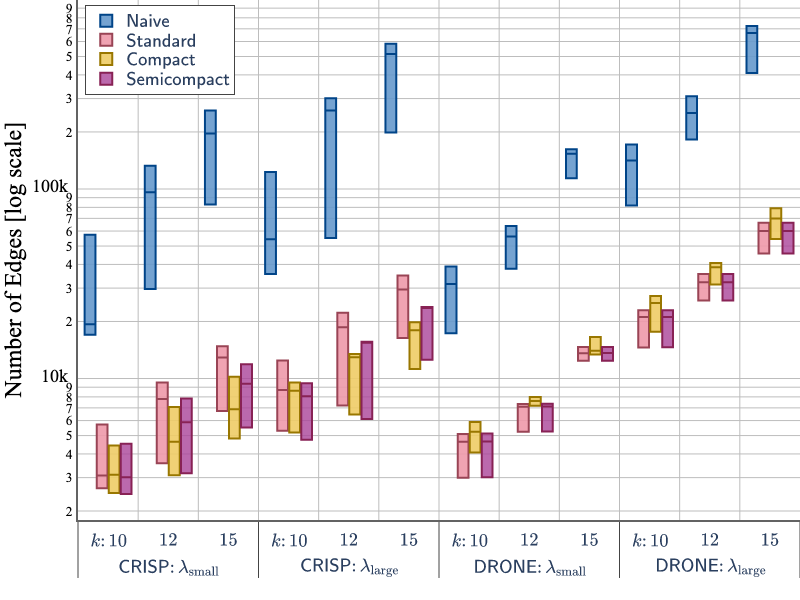}
  \end{subfigure}
  \hfill
  \begin{subfigure}{0.48\textwidth}
    \includegraphics[width=\linewidth]{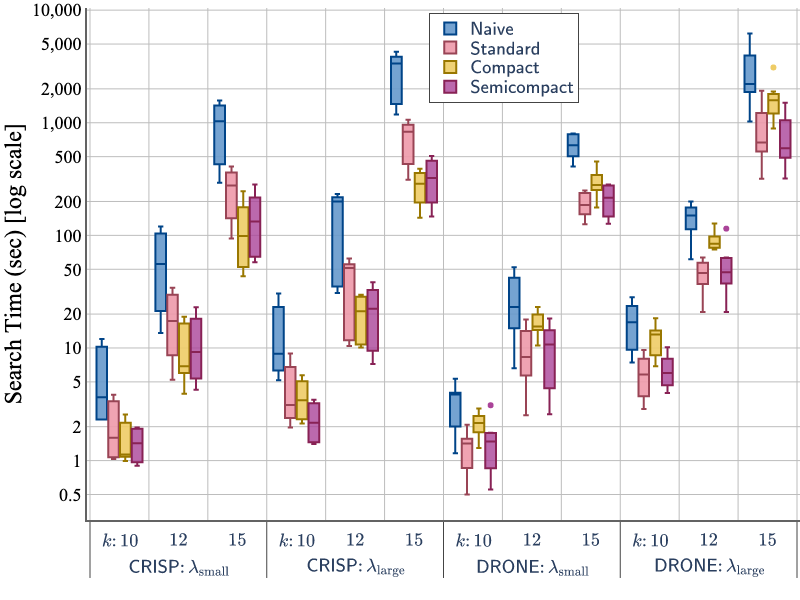}
  \end{subfigure}
  \caption{
  Edge count (left) and 
  search time (right) for each circuit type,
  tested on \CRISP and \DRONE instances
  with different $k$ values and scaling factors $\lambda$.
  Figures follow the standard box plot convention.
  %
  \csemi performs well on both instances
  (discussions in \Cref{sec:circuit-type-experiment}).}
  \label{fig:circuit}
\end{figure*}

\section{Proof of Main Theorem}\label{sec:proof-main-thm}

Now we are ready to formally state our main result.

\begin{theorem}
  \label{thm:algebra}
  If the edge weights are restricted to integers
  and there exists a solution with weight at most~$\ell \in \N$,
  then \gi{} can be solved in randomized
  $\tilde{\Oh}(2^t (\ell t^3 n^2 + t^3 \colorsetsize n))$~time
  and with~$\tilde{\Oh}(\ell t n^2 + t{\colorsetsize}n)$~space,
  with a constant success probability.
\end{theorem}

\begin{proof}
  Consider running \algipa with \ccompact, \sunified, and \rmc
  with appropriate preprocessing.
  We set scaling factor $\lambda=1$ (input weights are integral),
  and let $\theta$ be a constant that controls the overall success probability.
  The preprocessing steps described in \Cref{sec:gi-algorithm}
  take $\tilde{\Oh}(n^2)$ time to create a complete, metric graph.
  Step \ref{algipa:step:1} is optional if we know an upper bound $\ell$ for the solution weight.

  In Step \ref{algipa:step:2}, as shown by~\Cref{lem:construct-semicompact},
  \ccompact creates a circuit~$\C$ of
  $\Oh(\ell t^2n + \colorsetsize)$ nodes and $\Oh(\ell t n^2 + t \colorsetsize n)$ edges
  such that if there exists a solution walk with weight at most~$\ell$
  in $G$, then $\C$ contains a tree certificate for \MLDshort with degree
  at most $t$.
  This implies that for an output node $O_{\ell'}$ for every $\ell' \leq \ell$
  the fingerprint polynomial $P_{\C[O_{\ell'}]}(X,A)$ contains
  a multilinear monomial with coefficient $1$ if $\ell'$ is feasible.
  Note that $\C$ contains $\Oh(k)$ addition gates.

  By \Cref{lem:williams-poly-weak} and \Cref{lem:search-unified},
  we can find the optimal weight $\tilde{\ell}$ in
  $\tilde{\Oh}(2^t t (\ell tn^2 + t\colorsetsize n))$ time and
  $\tilde{\Oh}(\ell tn^2 + t\colorsetsize n + t\colorsetsize)=\tilde{\Oh}(\ell tn^2 + t\colorsetsize n)$ space
  with a constant probability.
  Lastly, in Step \ref{algipa:step:3} we use \rmc to find a solution walk with weight $\tilde{\ell}$.
  From \Cref{lem:recovery-mc}, we can find a tree certificate of $\C$ in
  $\tilde{\Oh}(2^t t\cdot t (\ell t n^2 + t \colorsetsize n))
  =2^t (\ell t^3 n^2 + t^3 \colorsetsize n)$ time with a constant probability.
  As stated in the proof of \Cref{lem:construct-compact},
  once we find a tree certificate of $\C$, we can reconstruct a walk $W$ collecting
  at least $t$ colors with weight $\tilde{\ell}$ on $G$.\looseness=-1

  Lastly, we replace each edge $uv$ in $W$
  with any of the shortest $u$-$v$ paths in the original graph for \gi{}
  to construct a solution walk $\tilde{W}$.
  Since $G$ is complete and metric, the weight of $\tilde{W}$ is also $\tilde{\ell}$,
  and since edge weights are restricted to (non-negative) integers,
  this weight is optimal in the original instance.
  Also, replacing edges in $W$ with shortest paths in the original graph is safe
  because it is enough to collect colors at the vertices that appeared in $W$.

  The algorithm fails only when either \sunified or \rmc fails.
  Since both subroutines have a constant success probability,
  the overall success probability is also constant.
\end{proof}

\begin{figure*}[t]
  \centering
  \begin{subfigure}{0.48\textwidth}
    \includegraphics[width=\linewidth]{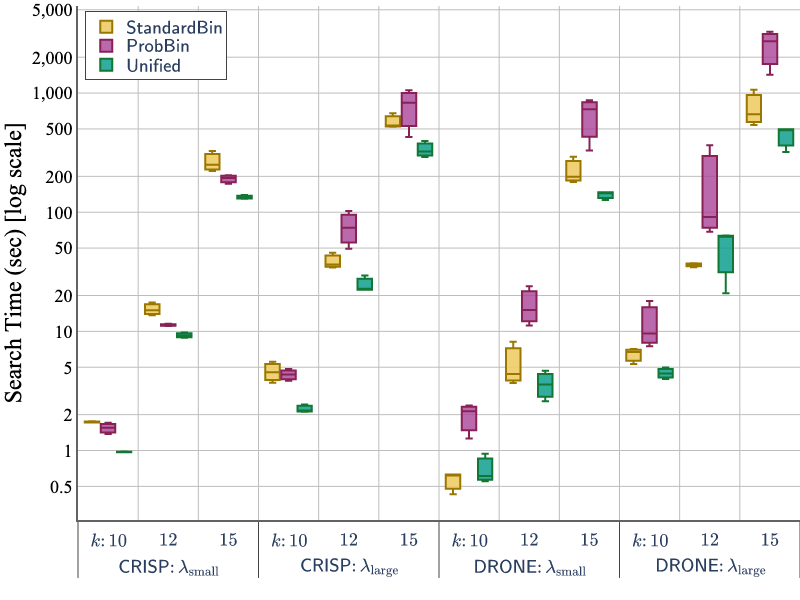}
  \end{subfigure}
  \hfill
  \begin{subfigure}{0.48\textwidth}
    \includegraphics[width=\linewidth]{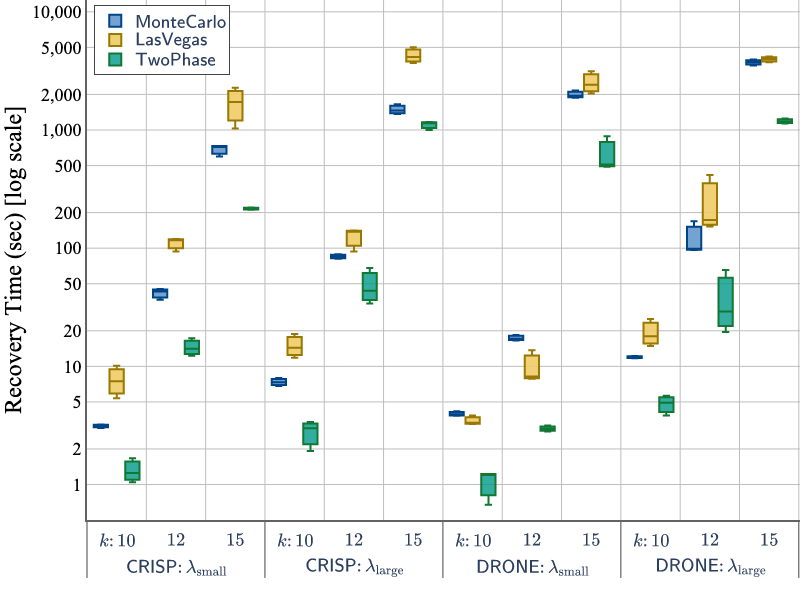}
  \end{subfigure}
  \caption{
  Evaluation of search strategies (left) and solution recovery strategies (right).
  \sunified and \rtp are clear winners (discussions in \Cref{sec:exp:search,sec:exp:recovery}).}
  \label{fig:search}
\end{figure*}

\section{Experimental Results}\label{sec:experiments}

We now describe an empirical evaluation of the trade-offs in our circuit-based approach to robotic motion planning.
We begin by summarizing the highlights.

In validation of the main thrust of our paper, we demonstrate (\Cref{sec:exp:dp}) that circuit-based implementations use less than half the memory of dynamic programming-based software
even for a moderate number of colors, and the observed trends suggest that this gap will reach several orders of magnitude for even slightly more colors.
In~\Cref{sec:circuit-type-experiment,sec:exp:search,sec:exp:recovery,sec:exp:scale,sec:exp:success}, we demonstrate that our software's performance is highly sensitive to various choices related to circuit construction, optimal weight identification, solution recovery, edge-weight scaling, and a parameter governing failure probability; these results motivate further research on the engineering of algebraic methods.
As expected, circuit-based algorithms are slower than dynamic programming,
but our software can consistently handle realistic instances in less than a day of computation time,
which may already be acceptable in some applications such as planning automated surgery, where there is a significant time delay between imaging and operation.
This is the application from which our first dataset, \CRISP{}, arises. Our second, \DRONE{}, is from a bridge-inspection scenario.\looseness=-1

Both datasets are from Fu~\emph{et al.}~\cite{fu2023asymptoticallyoptimal,fu2021computationally}.
We built RRGs (Rapidly-exploring Random Graphs \cite{karaman2011sampling}) using their software.
For each dataset,
we created roughly~$100$-vertex instances
and sampled $k \in \{10, 12, 15\}$ dispersed POIs (Points Of Interest) using an algorithm from Mizutani \emph{et al.}~\cite{mizutani2024leveraging}.
See~\Cref{appendix:preprocessing} for details.
Throughout the experiments, we set $t=k$, i.e. algorithms try to collect all colors (POIs)
in the given graph,
and parallelized using $80$ threads by default.\looseness=-1

We first verified the effect of algorithmic choices in subroutines---circuit types, search strategies,
and solution recovery strategies. 
%
We tested two scaling factors for each dataset:
$\lambda_\text{small}=50$ and $\lambda_\text{large}=100$ for \CRISP and
$\lambda_\text{small}=0.1$ and $\lambda_\text{large}=0.5$ for \DRONE.
We measured running times by averaging over three different seeds for a pseudorandom number generator;
for the circuit-type experiment we also used three different RRGs for each configuration.
For these experiments, we set the expected success probability to $90\%$.\looseness-1

Second, we tested quality-controlling parameters---scaling factors and
expected success probabilities---to observe trade-offs between running time/space usage and accuracy
for practical instances with real-valued edge weights.
After determining baseline configurations,
we compared the performance of \algipa to that of \dpipa, with various $k$ values
to verify their asymptotic running time and memory usage.
In \Cref{appendix:scalability}, we present further experimental results
on scalability with respect to graph size, and a multithreading analysis.\looseness=-1

We used C++17
and ran all experiments on Rocky Linux release 8.8
on identical hardware,
equipped with 40 physical cores (Intel(R) Xeon(R) Gold 6230 CPU @ 2.10 GHz) and 191000 MB of RAM.
Our code and data are available at \url{https://osf.io/w4gs3/?view_only=7465595847874192a17d960aa693bc72}.

\subsection{Choice of Circuit Types.}\label{sec:circuit-type-experiment}

To assess the four circuit types proposed in \Cref{sec:circuit},
we ran our algorithm with each construction, using the \sunified search strategy.
We first measured the number of edges in the constructed arithmetic circuit
as this number is a good estimator for running time and space usage.

\Cref{fig:circuit} (left) plots the edge-count of the circuit
for each instance.
By construction, \cstandard always gives a smaller circuit than \cnaive.
\ccompact has asymptotically the smallest circuit, but in some configurations,
especially in \DRONE, \ccompact results in a larger circuit than \cstandard.
We observed that \DRONE has a lower bound (relative to the optimal weight) higher than that of \CRISP
due to the fact that each POI of \DRONE belongs to fewer vertices.
\ccompact has to maintain low-weight walks that are ``below'' lower bounds,
thus creating more edges.
\csemi has about the same size as \cstandard with \DRONE
and is much smaller with \CRISP.\looseness=-1

Next, we measured the search time to obtain the optimal weight
using the search strategy \sunified.
\Cref{fig:circuit} (right) plots
the distribution of the search time\footnote{%
Circuit construction time was negligible (less than $1$ second).
} for each instance.
For both datasets, search time is closely related to the number of edges in the circuit.
With the \CRISP dataset with larger {$k$ ($\geq 12$)},
\ccompact recorded the best search time on average.
This was not true with \DRONE; for that dataset, \ccompact is slower than \cstandard and \csemi.
It is also surprising to us that in practice, \cstandard performed best on most of the \DRONE instances.
This suggests that the circuit size is not the only factor for determining the running time.\looseness=-1

\subsection{Choice of Search Strategies.}\label{sec:exp:search}

To evaluate search strategies,
we ran the algorithm with each strategy from \Cref{sec:search-strategies} with the circuit type \csemi
and measured the search time (as done in~\Cref{sec:circuit-type-experiment}).
\Cref{fig:search} (left) plots the distribution of the search time for each instance.
\sunified is the fastest with a few exceptions on \DRONE,
which matches our expectation.
\sprob is more unstable; it is faster than \sstandard with \CRISP
except using $\lambda_\text{large}$ with larger {$k$ ($\geq 12$)},
but it is the slowest among the three with \DRONE.

\begin{figure*}[t]
  \centering
  \begin{subfigure}{0.48\textwidth}
    \includegraphics[width=\linewidth]{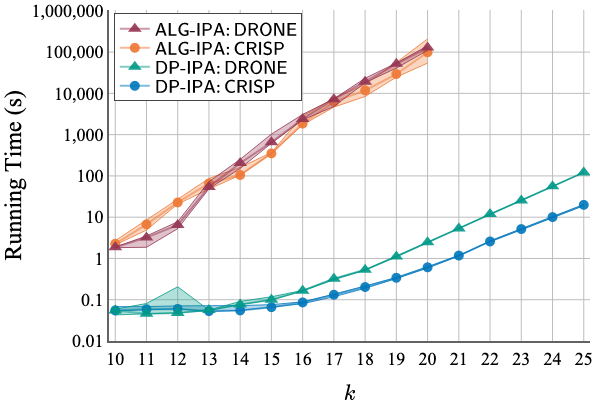}
  \end{subfigure}
  \hfill
  \begin{subfigure}{0.48\textwidth}
    \includegraphics[width=\linewidth]{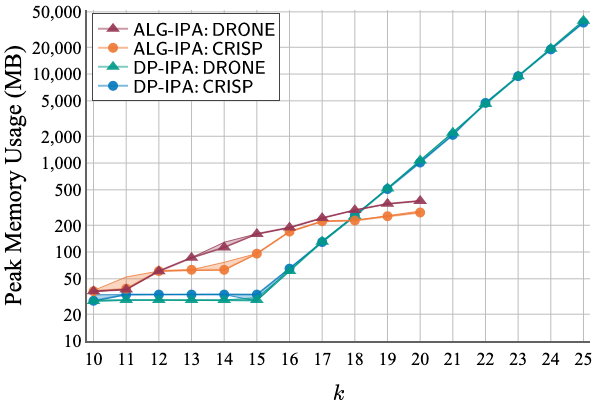}
  \end{subfigure}
  \caption{
  Overall running time (left) and peak memory usage (right) in log scale, comparing \algipa with \dpipa.
  Data points for \algipa with $k\geq 21$ are missing due to the time limit of $2$ days ($172,800$ seconds).
  \algipa is slower than \dpipa but has a clear advantage in memory usage
  (discussions in \Cref{sec:exp:dp}).}
  \label{fig:algdp}
\end{figure*}

\subsection{Choice of Solution Recovery Strategies.}\label{sec:exp:recovery}

For each recovery strategy,
we measured the time for solution recovery after finding the optimal weight,
using circuit type \csemi.
As shown in \Cref{fig:search} (right),
\rtp performed better than the others with all test instances.
\rmc was faster than \rlv with all but \DRONE with~$\lambda_\text{small}$
and smaller $k$ ($\leq 12$).
Also, notice that it is guaranteed that \rlv and \rtp always succeed.
The running time of \rmc would increase if we required a better success rate.
We conclude that \rtp has clear advantages over \rmc and \rlv for \gi,
but note \rmc and \rlv are still valuable since they apply to general circuits
and can find a tree certificate.

\subsection{Choice of Scaling Factors.}\label{sec:exp:scale}

Next, we evaluated the effect of varied scaling factors ($\lambda$)
by running our algorithm using subroutines \csemi, \sunified and \rtp
with fixed $k=15$.

We first measured how close the weight of the walk obtained by \algipa is to the optimal weight.
\Cref{fig:scale} (left) in the Appendix shows the ratio of the weight achieved by \algipa to the optimal
(the lower, the better).
This ratio was below $1.2$ for the \CRISP instance,
and with larger scaling factors (up to $150$), the ratio converges to around $1.1$.
The trend is different with \DRONE, where we obtain optimal solutions with scaling factor $1.0$.
The solution quality is as bad as $1.3$ with $\lambda=0.1$,
but is vastly improved (within $1.03$) with $\lambda = 0.25$.\looseness=-1

\Cref{fig:scale} (middle) plots the peak memory usage for \algipa.
We observe the memory usage grows almost linearly with respect to the scaling factor.
This is due to the fact that the size of an arithmetic circuit is proportional to the solution weight
in integers.
\Cref{fig:scale} (right) shows the overall running time
(including search time and solution recovery time) for each instance.
The precise running times depend on randomness,
but our experiment demonstrates that running time also increases almost linearly with the scaling factor.
Taken together, the plots in~\Cref{fig:scale} illustrate a trade-off between
fidelity (influenced by rounding errors) and computational resources (time and space). 

\subsection{Choice of Success Probability.}\label{sec:exp:success}

We then tested trade-offs between the expected success probability and the solution quality.
\Cref{fig:success} (left) in the Appendix suggests that even with the expected success probability around~$0.7$,
\algipa successfully finds optimal integer values.
Hence, we concluded that the expected success probability of $0.9$ is a valid choice.
\Cref{fig:success} (right) shows the relationship between the expected success probability and
the search time.
As expected, running times grow with success probabilities $p$,
nearly proportional to the value $-\log(1-p)$.

\subsection{Comparison to \dpipa.}\label{sec:exp:dp}

Finally, we compared overall running time and peak memory usage of \algipa with \dpipa for varied $k$ values
($10 \leq k \leq 20$ for \algipa and $10 \leq k \leq 25$ for \dpipa).
We did not run \algipa for $k \geq 21$ because each run would take more than $2$ days.
We used subroutines \csemi, \sunified, and \rtp
with scaling factor $\lambda_{\text{small}}$ and expected success probability $0.9$.

Results were straightforward;
\Cref{fig:algdp} (left) shows that \algipa is significantly slower than \dpipa,
but the growth of the running time is exponential in $k$ for both algorithms.
Also, as Theorem \ref{thm:algebra} suggests, the running time of \algipa is $\Oh(2^k k^3)$
if we fix the other parameters, which is worse than \dpipa's $\Oh(2^k)$.
However, in \Cref{fig:algdp} (right) we see a clear advantage of \algipa.
The plot verifies the exponential growth of the peak memory usage by \dpipa, with respect to $k$,
while validating that the memory usage by \algipa is polynomial in~$k$.
At $k=19$, the memory usage of \algipa is less than that of \dpipa for all instances,
and we expect that this trend will grow for any larger $k$.

\section{Conclusion}

In this paper we address the exponential space complexity of dynamic-programming algorithms for robotic motion planning by
designing a memory-efficient algorithm based on a reduction to \MLD{}. We introduce the concept of tree certificates, which
have broad implications for the recovery of multilinear monomials from arithmetic circuits, as well as several carefully designed application-specific algorithmic subroutines.
Our experiments verify that our algebraic approach is memory-efficient in practice, and provide evidence that
algebraic alternatives to dynamic programming are not as far from practicality as one might suspect.
We hope to inspire further research on both the theory and practice of these techniques; in particular, we wish to highlight
the potential of GPGPU acceleration, which has proven effective (albeit very complicated) for a related class of algebraic algorithms~\cite{kaski2018engineering}.\looseness=-1


\clearpage
\bibliographystyle{siam}
\bibliography{refs}

\begin{thebibliography}{10}

\bibitem{diestel2005graph}
{\sc R.~Diestel}, {\em Graph theory}, Springer-Verlag, Berlin, 2005.

\bibitem{eiben2024determinantal}
{\sc E.~Eiben, T.~Koana, and M.~Wahlstr{\"o}m}, {\em Determinantal sieving}, in
  Proceedings of the 2024 Annual ACM-SIAM Symposium on Discrete Algorithms
  (SODA), SIAM, 2024, pp.~377--423.

\bibitem{englot2017planning}
{\sc B.~Englot and F.~Hover}, {\em Planning complex inspection tasks using
  redundant roadmaps}, in Proceedings of the 15th International Symposium on
  Robotics Research (ISRR), Springer, 2017, pp.~327--343.

\bibitem{fu2023asymptoticallyoptimal}
{\sc M.~Fu, A.~Kuntz, O.~Salzman, and R.~Alterovitz}, {\em Asymptotically
  optimal inspection planning via efficient near-optimal search on sampled
  roadmaps}, The International Journal of Robotics Research, 42 (2023),
  pp.~150--175.

\bibitem{fu2021computationally}
{\sc M.~Fu, O.~Salzman, and R.~Alterovitz}, {\em Computationally-efficient
  roadmap-based inspection planning via incremental lazy search}, in
  Proceedings of the 2021 International Conference on Robotics and Automation
  (ICRA), IEEE, 2021, pp.~7449--7456.

\bibitem{guillemot2013findingcounting}
{\sc S.~Guillemot and F.~Sikora}, {\em Finding and counting vertex-colored
  subtrees}, Algorithmica, 65 (2013), pp.~828--844.

\bibitem{ichnowski2020cloud}
{\sc J.~Ichnowski, J.~Prins, and R.~Alterovitz}, {\em Cloud-based motion plan
  computation for power-constrained robots}, in Algorithmic Foundations of
  Robotics XII: Proceedings of the Twelfth Workshop on the Algorithmic
  Foundations of Robotics, Springer, 2020, pp.~96--111.

\bibitem{ichnowski2014cache}
{\sc J.~Ichnowski, J.~F. Prins, and R.~Alterovitz}, {\em Cache-aware
  asymptotically-optimal sampling-based motion planning}, in 2014 IEEE
  International Conference on Robotics and Automation (ICRA), IEEE, 2014,
  pp.~5804--5810.

\bibitem{ichnowski2020economic}
{\sc J.~Ichnowski, J.~F. Prins, and R.~Alterovitz}, {\em The economic case for
  cloud-based computation for robot motion planning}, in Robotics Research, The
  18th International Symposium, {ISRR} 2017, vol.~10 of Springer Proceedings in
  Advanced Robotics, Springer, 2017, pp.~59--65.

\bibitem{karaman2011sampling}
{\sc S.~Karaman and E.~Frazzoli}, {\em Sampling-based algorithms for optimal
  motion planning}, The International Journal of Robotics Research, 30 (2011),
  pp.~846--894.

\bibitem{kaski2018engineering}
{\sc P.~Kaski, J.~Lauri, and S.~Thejaswi}, {\em Engineering motif search for
  large motifs}, in Proceedings of the 17th International Symposium on
  Experimental Algorithms {(SEA)}, Schloss Dagstuhl - Leibniz-Zentrum f{\"{u}}r
  Informatik, 2018, pp.~28:1--28:19.

\bibitem{koutis2008fasteralgebraic}
{\sc I.~Koutis}, {\em Faster {Algebraic} {Algorithms} for {Path} and {Packing}
  {Problems}}, in Proceedings of the 35th {International} {Colloquium} of
  {Automata}, {Languages} and {Programming} ({ICALP}), Springer, 2008,
  pp.~575--586.

\bibitem{koutis2016limitsapplications}
{\sc I.~Koutis and R.~Williams}, {\em {LIMITS} and {Applications} of {Group}
  {Algebras} for {Parameterized} {Problems}}, ACM Transactions on Algorithms,
  12 (2016), pp.~31:1--31:18.

\bibitem{mizutani2024leveraging}
{\sc Y.~Mizutani, D.~C. Salomao, A.~Crane, M.~Bentert, P.~G. Drange, F.~Reidl,
  A.~Kuntz, and B.~D. Sullivan}, {\em Leveraging fixed-parameter tractability
  for robot inspection planning}, arXiv preprint arXiv:2407.00251, 2024.

\bibitem{noreen2016optimal}
{\sc I.~Noreen, A.~Khan, and Z.~Habib}, {\em Optimal path planning for mobile
  robots using memory efficient a}, in 2016 International Conference on
  Frontiers of Information Technology (FIT), IEEE, 2016, pp.~142--146.

\bibitem{pop2023comprehensive}
{\sc P.~C. Pop, O.~Cosma, C.~Sabo, and C.~P. Sitar}, {\em A comprehensive
  survey on the generalized traveling salesman problem}, European Journal of
  Operational Research, 314 (2024), pp.~819--835.

\bibitem{williams2009findingpaths}
{\sc R.~Williams}, {\em Finding paths of length k in {$O^*(2^k)$} time},
  Information Processing Letters, 109 (2009), pp.~315--318.

\end{thebibliography}

\clearpage
\appendix
\section{Theoretical Results}\label{sec:proofs}

\subsection{Proofs on Tree Certificates.}\label{sec:proof-tree-cert}

We start by stating the following property of a tree certificate.

\begin{proposition}\label{prp:tree-cert-1}
  Let $\ccert$ be a tree certificate without scalar inputs.
  Then, for any node $v \in V(\ccert)$,
  we have that $P_{\ccert[v]}(X, A)$ is a multilinear monomial
  (without any other terms) of coefficient~$1$.
  Moreover, every addition gate in $\ccert$ has exactly one in-neighbor.
\end{proposition}

\begin{proof}
  We prove by induction on the number~$n$ of nodes in $\ccert$ that
  $P_{\ccert[v]}(X, A)$ contains exactly one monomial of coefficient~$1$ for all $v \in V(\ccert)$.
  For the base case with $n=1$, the only node is a variable node,
  which is by definition a multilinear monomial of coefficient~$1$.
  For the inductive step,
  consider $P_{\ccert[v]}(X,A)$ for an output node $v$.
  Notice that $v$ must be either an addition gate or a multiplication gate.

  If $v$ is an addition gate, we know that there exists
  an in-neighbor of $v$ whose fingerprint polynomial contains a multilinear monomial of coefficient~$1$.
  Due to the minimality of a certificate, $v$ cannot have more than one in-neighbor.
  Let $u \in N_{\ccert}^-(v)$ be $v$'s in-neighbor.
  Then, $\ccert[u]$ is a tree certificate for $u$.
  From the inductive hypothesis, $P_{\ccert[v]}(X,A) = a_{uv} \cdot P_{\ccert[u]}(X,A)$,
  which is a multilinear monomial of coefficient~$1$.

  If $v$ is a multiplication gate,
  then for every in-neighbor~$u$ of~$v$,
  $\ccert[u]$ must be a tree certificate
  because otherwise $P_{\ccert[v]}(X,A)$ cannot have a multilinear monomial.
  From the inductive hypothesis, $P_{\ccert[u]}(X,A)$ is a multilinear monomial of coefficient~$1$
  for every~$u$.
  Notice that since the underlying graph of $\ccert$ is a tree,
  for each distinct $u, u' \in N_{\ccert}^-(v)$,
  $P_{\ccert[u]}(X,A)$ and $P_{\ccert[u']}(X,A)$ do not share any variables or fingerprints.
  Hence, $P_{\ccert[v]}(X,A) = \prod_{u \in N_{\ccert}^-(v)}P_{\ccert[u]}(X,A)$
  is also a multilinear monomial of coefficient~$1$.
\end{proof}

\begin{corollary}\label{cor:tree-cert-sub}
  Let $\ccert$ be a tree certificate without scalar inputs.
  Then, for any node $v \in V(\ccert)$, $\ccert[v]$ is also a tree certificate.
\end{corollary}

\begin{proof}
  From \Cref{prp:tree-cert-1}, for any node $v \in V(\ccert)$
  we have that $P_{\ccert[v]}(X,A)$ is a multilinear monomial
  of coefficient $1$.
  Necessarily, $\ccert[v]$ is a certificate as it should be minimal,
  and clearly the underlying graph of $\ccert[v]$ is a tree.
  Thus $\ccert[v]$ is a tree certificate.
\end{proof}




Here we prove \Cref{lem:tree-cert}.

\treecertlemma*

To show this, we introduce new notation and then prove a stronger lemma.
For a nonempty set of variables~$\emptyset \neq X' \subseteq X$ and
a (possibly empty) set of fingerprints~$A' \subseteq A$,
we write~$\mu(X',A')$ for~$\left(\prod_{x \in X'} x \right)\left(\prod_{a \in A'} a \right)$.
By definition, $\mu(X',A')$ is a multilinear monomial of coefficient~$1$.
Similarly, we write~$\mu(X')$ for~$\prod_{x \in X'}x$ in case fingerprints are irrelevant.
Also, we say a tree certificate $\ccert$ for $v \in V(\ccert)$ \emph{encodes} $(X',A')$
if $P_{\ccert[v]}(X,A) = \mu(X',A')$.


\begin{lemma}\label{lem:tree-cert-strong}
  Let $\C$ be a circuit without scalar inputs.
  For a node $v \in V(\C)$, the fingerprint polynomial $P_{\C[v]}(X,A)$ contains
  a monomial~$\mu(X',A')$ for some $\emptyset \neq X' \subseteq X$ and $A' \subseteq A$
  if and only if there exists a tree certificate $\ccert$ for $v$ encoding $(X',A')$.
\end{lemma}

\begin{proof}
  We will show the both directions of implications.

  ($\Rightarrow$)
  For a node $v \in V(\C)$,
  assume that $P_{\C[v]}(X,A)$ contains
  a monomial~$\mu(X',A')$ for some~$\emptyset \neq X' \subseteq X$ and $A' \subseteq A$.
  We will show that $\C$ has a tree certificate~$\ccert$ for~$v$ encoding~$(X',A')$
  by induction on the number~$n$ of nodes in~$\C[v]$.
  It is trivial for the base case $n=1$:
  since there are no scalar inputs,
  we have $n=1$ if and only if $v$ is a variable node representing $x \in X$.
  We have $P_{\C[v]}(X,A)=x=\mu(\{x\},\emptyset)$,
  and $\C[v]$ encodes $(\{x\},\emptyset)$.
  For the inductive step with $n>1$, consider two cases.

  Suppose node~$v$ is an addition gate.
  Then, there exists $u \in N_{\C}^-(v)$ such that
  $P_{\C[u]}(X,A)$ contains a monomial~$\mu(X', A' \setminus \{a_{uv}\})$.
  From the inductive hypothesis, there exists a tree certificate~$\ccert$ for~$u$ encoding~$(X', A' \setminus \{a_{uv}\})$.
  We see that the circuit~$(V(\ccert) \cup \{v\}, E(\ccert) \cup \{uv\})$
  forms a tree certificate for $v$ encoding~$(X',A')$.

  Suppose node~$v$ is a multiplication gate, and let $u_1,\ldots,u_\ell$ be the in-neighbors of~$v$ in~$\C$.
  Since there are no scalar inputs, there must be a partition~$(X_1,\ldots,X_\ell)$ of~$X'$
  and a partition~$(A_1,\ldots,A_\ell)$ of~$A'$
  such that for each $1 \leq i \leq \ell$,
  $X_i \neq \emptyset$ and
  $P_{\C[u_i]}(X,A)$ contains a monomial~$\mu(X_i,A_i)$.
  From the inductive hypothesis, there exists a tree certificate~$\ccert_i$ for~$u_i$ encoding~$(X_i,A_i)$
  for each~$1 \leq i \leq \ell$.
  If tree certificates $\ccert_i$ are vertex-disjoint, then
  the circuit~$(\{v\} \cup \bigcup_{i} V(\ccert_i), \bigcup_{i} (E(\ccert_i) \cup \{u_i v\}))$
  forms a tree certificate for $v$ encoding $(X',A')$.

  Now, assume towards a contradiction that tree certificates $\ccert_1$ and $\ccert_2$
  (without loss of generality) share a node.
  Let $w \in V(\ccert_1) \cap V(\ccert_2)$ be a shared node
  first appearing in an arbitrary topological ordering of $(V(\ccert_1) \cup V(\ccert_2), E(\ccert_1) \cup E(\ccert_2))$.
  Observe that $w$ cannot be a variable node because ${X_1 \cap X_2 = \emptyset}$. \looseness-1

  From \Cref{cor:tree-cert-sub}, 
  $\ccert_1[w]$ and $\ccert_2[w]$ are tree certificates for $w$.
  Then, there exist some sets $X_1',A_1',X_2',A_2'$ such that
  $\emptyset \neq X_1' \subseteq X_1$,
  $\emptyset \neq X_2' \subseteq X_2$, $A_1' \subseteq A_1$, $A_2' \subseteq A_2$,
  $P_{\ccert_1[w]}(X,A)=\mu(X_1',A_1')$ and $P_{\ccert_2[w]}(X,A)=\mu(X_2',A_2')$.

  Consider the circuit~$\ccert_1' := ((V(\ccert_1) \setminus V(\ccert_1[w])) \cup V(\ccert_2[w]),
  (E(\ccert_1) \setminus E(\ccert_1[w])) \cup E(\ccert_2[w]))$.
  Since $w$ is the earliest ``overlapping'' node in the topological ordering,
  ${V(\ccert_1) \setminus V(\ccert_1[w])}$ and $V(\ccert_2[w])$ do not share any nodes.
  Then, $\ccert_1'$ is the tree certificate for $u_1$ encoding $((X_1 \setminus X_1') \cup X_2', (A_1 \setminus A_1') \cup A_2')$.
  Similarly, define $\ccert_2' := (V(\ccert_2) \setminus V(\ccert_2[w]) \cup V(\ccert_1[w]),
  E(\ccert_2) \setminus E(\ccert_2[w]) \cup E(\ccert_1[w]))$.
  We know that $\ccert_2'$ is the tree certificate for $u_2$ encoding $((X_2 \setminus X_2') \cup X_1', (A_2 \setminus A_2') \cup A_1')$.
  
  For convenience, let
  $\tilde{X}_1:=(X_1 \setminus X_1') \cup X_2'$,
  $\tilde{A}_1:=(A_1 \setminus A_1') \cup A_2'$,
  $\tilde{X}_2:=(X_2 \setminus X_2') \cup X_1'$, and 
  $\tilde{A}_2:=(A_2 \setminus A_2') \cup A_1'$.
  Notice that by the inductive hypothesis,
  $P_{\C[u_1]}(X,A)$ contains monomials $\mu(X_1,A_1)$
  and $\mu(\tilde{X}_1, \tilde{A}_1)$.
  Similarly,
  $P_{\C[u_2]}(X,A)$ contains monomials $\mu(X_2,A_2)$
  and $\mu(\tilde{X}_2, \tilde{A}_2)$.
  %
  Because $v$ is a multiplication gate, we have
  \begin{align*}
    &\quad P_{\C[v]}(X,A)\\
    =\ &\prod_{i=1}^{\ell}P_{\C[u_i]}(X,A)\\
    =\ &P_{\C[u_1]}(X,A) \cdot P_{\C[u_2]}(X,A) \cdot \prod_{i=3}^{\ell}P_{\C[u_i]}(X,A)\\
    =\ &(\mu(X_1,A_1) + \mu(\tilde{X}_1,\tilde{A}_1)+\ldots)\cdot\\
    & (\mu(X_2,A_2) + \mu(\tilde{X}_2,\tilde{A}_2)+\ldots)\cdot
    \prod_{i=3}^{\ell}(\mu(X_i,A_i) + \ldots)\\
    =\ &(\mu(X_1,A_1)\mu(X_2,A_2) + \mu(\tilde{X}_1,\tilde{A}_1)\mu(\tilde{X}_2,\tilde{A}_2) + \ldots)\cdot\\
    & \prod_{i=3}^{\ell}(\mu(X_i,A_i) + \ldots)\\
    =\ &(\mu(X_1,A_1)\mu(X_2,A_2) + \mu(\tilde{X}_1,\tilde{A}_1)\mu(\tilde{X}_2,\tilde{A}_2))\cdot\\
    &\prod_{i=3}^{\ell}\mu(X_i,A_i) + p(X,A),
  \end{align*}
  for some polynomial $p(X,A)$.
  This can be simplified as follows:
  \begin{align*}
     & \mu(\tilde{X}_1,\tilde{A}_1)\mu(\tilde{X}_2,\tilde{A}_2)\\
    =\ & \mu(\tilde{X}_1 \cup \tilde{X}_2, \tilde{A}_1 \cup \tilde{A}_2)\\
    =\ & \mu(((X_1 \setminus X_1') \cup X_2') \cup ((X_2 \setminus X_2') \cup X_1'),\\
    & ((A_1 \setminus A_1') \cup A_2') \cup ((A_2 \setminus A_2') \cup A_1'))\\
    =\ & \mu(((X_1 \setminus X_1') \cup X_1') \cup ((X_2 \setminus X_2') \cup X_2'),\\
    & ((A_1 \setminus A_1') \cup A_1') \cup ((A_2 \setminus A_2') \cup A_2'))\\
    =\ & \mu(X_1 \cup X_2, A_1 \cup A_2)\\
    =\ & \mu(X_1,A_1)\mu(X_2,A_2)
  \end{align*}
  \begin{align*}
    &P_{\C[v]}(X,A)\\
    =\ & 2\mu(X_1,A_1)\mu(X_2,A_2)\cdot
    \prod_{i=3}^{\ell}\mu(X_i,A_i) + p(X,A)\\
    =\ & 2\cdot \prod_{i=1}^{\ell}\mu(X_i,A_i) + p(X,A)\\
    =\ & 2\mu(\bigcup_{i=1}^\ell X_i, \bigcup_{i=1}^\ell A_i) + p(X,A)\\
    =\ & 2\mu(X',A') + p(X,A)
  \end{align*}

  This result implies
  that $P_{\C[v]}(X,A)$ contains a monomial~$\alpha \mu(X',A')$ for some integer~$\alpha \geq 2$,
  contradicting our assumption that $P_{\C[v]}(X,A)$ contains $\mu(X',A')$ as a monomial.

  ($\Leftarrow$)
  Suppose that $\C$ has a tree certificate $\ccert$ for $v \in V(\C)$ encoding $(X',A')$
  for some $\emptyset \neq X' \subseteq X$ and $A' \subseteq A$.
  We will show that $P_{\C[v]}(X,A)$ contains the monomial~$\mu(X',A')$.

  By definition, we have $P_{\ccert[v]}(X,A)=\mu(X',A')$.
  We iteratively construct ${\C}'$ as follows.
\begin{enumerate}
  \item Initially, let ${\C}' \gets (V(\C), E(\ccert))$.
  \item Add all the edges in $E(\C) \setminus E(\ccert)$ to $\C'$
  that are not pointing to $V(\ccert)$.
  At this point, we still have $P_{{\C}'[v]}(X,A)=\mu(X',A')$.
  \item Add an arbitrary edge $uw \in E(\C) \setminus E({\C}')$ pointing to $w \in V(\ccert)$.
  Here, node $w$ is an addition gate because otherwise edge $uw$ must have been included in~$\ccert$.
  Notice that newly introduced terms in $P_{{\C}'[v]}(X,A)$ have $a_{uw}$ as a factor.
  Recall that for every edge $e$ such that $a_e \in A'$,
  we have $e \in E(\ccert)$.
  Since $uv \not \in E(\ccert)$, $a_{uw} \not\in A'$,
  and $P_{{\C}'[v]}(X,A)$ still have the monomial $\mu(X',A')$.
  \item Repeat from Step 3 until we reach $\C' = \C$.
\end{enumerate}

  From the construction above, we have shown that $P_{\C[v]}(X,A)$ contains $\mu(X',A')$ as a monomial.
\end{proof}

The following proposition characterizes variable nodes in a certificate.

\begin{proposition}\label{prp:cert-var-outdeg}
  Let $\ccert$ be a certificate for node $v$ that is not a variable node.
  Then, every variable node in $\ccert$ has out-degree $1$.
\end{proposition}

\begin{proof}
If there is a variable node $x$ with out-degree $0$, then $x$ can be removed
and $\ccert$ is not minimal.

Assume towards a contradiction that a variable node $x$ has out-degree at least $2$.
Then, $\ccert$ contains a node $u$ with two vertex-disjoint $x$-$u$ paths.
Since $\ccert$ is minimal, multilinear monomials in $P_{\ccert[v]}(X)$
require a monomial in $P_{\ccert[u]}(X)$ using the two $x$-$u$ paths $P_1,P_2$.

First, $u$ cannot be a multiplication gate as the terms in~$P_{\ccert[u]}(X)$ using
$P_1$ and $P_2$ contain $x^2$.
Suppose $u$ is an addition gate. Then, the terms in $P_{\ccert[u]}(X)$ using
$P_1$ and $P_2$ are in the form $\alpha_1 x + \alpha_2 x$ for some monomials $\alpha_1$ and $\alpha_2$.
Then, having only one of the in-neighbors of $u$ is sufficient,
and thus $\ccert$ is not minimal, a contradiction.\looseness-1
\end{proof}

Next, we prove \Cref{lem:simple-circuit}.

\lemsimplecircuit*

\begin{proof}
  Let $\ccert$ be a certificate in $\C$.
  We prove by induction on the number $n$ of nodes in~$\ccert$.
  It is clear for the base case $n=1$.
  For the inductive step with $n>1$,
  let $v$ be an output node of~$\ccert$.
  We consider two cases (note that $v$ cannot be a variable node).

  Suppose node $v$ is an addition gate.
  Then, it must have one in-neighbor $u$ due to the minimality.
  First, we show that $\ccert[u]$ is a certificate for $u$.
  Since $P_{\ccert[v]}(X)=P_{\ccert[u]}(X)$,
  we know that $P_{\ccert[u]}(X)$ contains a multilinear monomial.
  Also for the same reason, if $\ccert[u]$ is not minimal, then $\ccert[v]$ is not minimal.
  Second, from the inductive hypothesis $\ccert[u]$ is a tree certificate,
  and adding edge $uv$ to $\ccert[u]$ does not create a cycle in the underlying graph.
  Thus, $\ccert$ is a tree certificate.

  Suppose node $v$ is a multiplication gate.
  If $v$ does not have a non-variable in-neighbor, then $\ccert$ is clearly a tree certificate.
  Assume that the in-neighbors of $v$ are
  variable nodes $X' \subseteq X$ and a non-variable node $u$.
  Since $P_{\ccert[v]}(X)=\mu(X')\cdot P_{\ccert[u]}(X)$,
  if $P_{\ccert[v]}(X)$ contains a multilinear monomial~$m_v(X)$,
  then $P_{\ccert[u]}(X)$ contains a multilinear monomial~$m_u(X):=m_v(X)/\mu(X')$.
  Here $m_u(X)$ may be a constant; recall that constant terms are also considered multilinear.

  First, we show that $\ccert[u]$ is a certificate for $u$.
  As $P_{\ccert[u]}(X)$ contains a multilinear monomial,
  it suffices to show that $\ccert[u]$ is minimal.
  Assume not.
  Then, there exists a certificate $\ccert'[u]$ as a sub-circuit of $\ccert[u]$
  such that $P_{\ccert'[u]}(X)$ contains a multilinear monomial $m_u'(X)$.
  Now, the circuit $\ccert'$ constructed from $\ccert$
  by replacing $\ccert[u]$ with $\ccert'[u]$ must have the monomial $\mu(X')\cdot m_u'(X)$.
  From \Cref{prp:cert-var-outdeg}, $m'_u(X)$ does not contain any variables from $X'$
  because for each $x \in X'$, $v$ is the only out-neighbor of $x$.
  Hence, $\mu(X')\cdot m_u'(X)$ is a multilinear monomial,
  contradicting that $\ccert$ is minimal.

  Second, from the inductive hypothesis $\ccert[u]$ is a tree certificate.
  Also, from \Cref{prp:cert-var-outdeg}, $\ccert[u]$ is disjoint from $X'$.
  Thus, $\ccert$ is a tree certificate.
\end{proof}

To extend our results to general circuits,
we first show that we can preprocess scalar nodes in linear time.

\begin{algorithm2e}[t]
  \SetAlgoLined
  \KwInput{%
    An instance $(\C, k)$ of \MLD with output nodes $O$.\\
  }
  \KwOutput{
    Sets $R_0,R_1 \subseteq O$
    such that 
    $P_{\C[r]}(X)=0$ for each $r \in R_0$, and
    $P_{\C[r]}(X)$ contains a constant term for each $r \in R_1$, 
    and an equivalent scalar-free sub-circuit $\C'$ of $\C$ containing $O \setminus \{R_0\cup R_1\}$.
  }

  \BlankLine
  Initialize $S_0$ and $S_1$ to $\emptyset$.

  \tcp{Forward traversal from inputs.}
  \For{$v \in V(\C)$ in topological ordering of $\C$}{
    \If{$v$ is a scalar input} {\label{code:prep-scalar:scalar}
      \If{$P_{\C[v]}(X)=0$}{
        Let $S_0 \gets S_0 \cup \{v\}$.\;
      }
      \uElse{
        Let $S_1 \gets S_1 \cup \{v\}$.\;
      }
    }
    \uElseIf{$v$ is an addition gate}{
      \If{$N_{\C}^-(v) \subseteq S_0$} {\label{code:prep-scalar:add0}
        Let $S_0 \gets S_0 \cup \{v\}$.\;
      }
      \uElseIf{$N_{\C}^-(v) \cap S_1 \neq \emptyset$} {\label{code:prep-scalar:add1}
        Let $S_1 \gets S_1 \cup \{v\}$.\;
      }
    }\uElseIf{$v$ is a multiplication gate}{
      \If{$N_{\C}^-(v) \cap S_0 \neq \emptyset$}{\label{code:prep-scalar:mul0}
        Let $S_0 \gets S_0 \cup \{v\}$.\;
      }
      \uElseIf{$N_{\C}^-(v) \subseteq S_1$} {\label{code:prep-scalar:mul1}
        Let $S_1 \gets S_1 \cup \{v\}$.\;
      }
    }
  }

  \KwRet{$(O \cap S_0, O \cap S_1, \C-(S_0 \cup S_1))$}\;
  \caption{\preprocessscalar}
  \label{alg:preprocess-scalar}
\end{algorithm2e}

\begin{lemma}\label{lem:preprocess-scalar}
  Given an instance $(\C,k)$ of \MLD with $m$ edges and output nodes $O$,
  there exists an $\Oh(m)$-time algorithm that
  determines
  the sets~$R_0, R_1 \subseteq O$ such that
  $P_{\C[r]}(X)=0$ for each $r \in R_0$, and
  $P_{\C[r]}(X)$ contains a constant term for each $r \in R_1$.
  The algorithm also outputs an equivalent scalar-free sub-circuit~$\C'$ of~$\C$
  containing $O \setminus \{R_0,R_1\}$.
\end{lemma}

\begin{proof}
  Consider \Cref{alg:preprocess-scalar}.
  We traverse all nodes of~$\C$ in topological ordering
  while updating two sets:
  the set $S_0$ of nodes $v$ such that $P_{\C[v]}(X)=0$
  and the set $S_1$ of nodes $v$ such that $P_{\C[v]}(X)$
  contains a constant term.
  In the end, we have $R_0 = S_0 \cap O$ and $R_1 = S_1 \cap O$.

  We first show that after the algorithm processes node $v$,
  we have $v \in S_0$ if and only if $P_{\C[v]}(X)=0$, and%
  ~$v \in S_1$ if and only if $P_{\C[v]}(X)$ contains a constant term.
  If~$v$ is an input node, there are three cases (\Cref{code:prep-scalar:scalar}).
  If~$v$ is a variable node, then $v \not\in S_0, S_1$.
  If~$v$ is a scalar input~$0$, then $v \in S_0$ and $v \not\in S_1$.
  If~$v$ is a nonzero scalar input, then $v \not\in S_0$ and $v \in S_1$.

  Suppose $v$ is an addition gate.
  $P_{\C[v]}(X)=0$ if and only if $P_{\C[u]}(X)=0$ for all in-neighbors $u$ of $v$,
  or equivalently, $N_{\C}^-(v) \subseteq S_0$ (\Cref{code:prep-scalar:add0}).
  $P_{\C[v]}(X)$ contains a constant term if and only if there exists an in-neighbor $u$ of $v$
  such that $P_{\C[u]}(X)$ contains a constant term,
  or $N_{\C}^-(v) \cap S_1 \neq \emptyset$ (\Cref{code:prep-scalar:add1}).

  Suppose $v$ is a multiplication gate.
  $P_{\C[v]}(X)=0$ if and only if there exists an in-neighbor $u$ of $v$
  such that $P_{\C[u]}(X)=0$,
  or $N_{\C}^-(v) \cap S_0 \neq \emptyset$ (\Cref{code:prep-scalar:mul0}).
  $P_{\C[v]}(X)$ contains a constant term if and only if
  $P_{\C[u]}(X)$ contains a constant term for all in-neighbors $u$ of $v$,
  or $N_{\C}^-(v) \subseteq S_1$ (\Cref{code:prep-scalar:mul1}).

  Since $R_0$ is a subset of $S_0$,
  for each $r \in R_0$, $P_{\C[r]}(X)=0$.
  Similarly, since $R_1$ is a subset of $S_1$,
  for each $r \in R_1$, $P_{\C[r]}(X)$ contains a constant term.
  Let $\C'=\C - (S_0 \cap S_1)$.
  $\C'$ is a sub-circuit of $\C$ by definition.
  \Cref{alg:preprocess-scalar} checks each edge at most once,
  and the inner loop takes $\Oh(1)$ time,
  so the running time is $\Oh(m)$.

  It remains to show that for each output node~$r$ in~$\C'$,
  $P_{\C[r]}(X)$ contains a multilinear monomial of degree at most $k$
  if and only if
  $P_{\C'[r]}(X)$ contains a multilinear monomial of degree at most $k$.
  For the sake of brevity, we write a polynomial $p(X)$ is in the class~$\Pi$
  if $p(X)$ contains a multilinear monomial of degree at most $k$.
  We proceed by induction on the nodes $v$ in $\C'$
  to show that $P_{\C[v]}(X) \in \Pi$
  if and only if $P_{\C'[v]}(X) \in \Pi$.
  For the base case, any input~$v$ in~$\C'$ is a variable node for some variable $x$.
  So, $P_{\C[v]}(X)=P_{\C'[v]}(X)=x$.
  For the inductive step for a node $v$, assume the claim holds for
  all vertices in $\C'[v]$.

  Suppose $v$ is an addition gate.
  Then, by construction,
  for each in-neighbor $u$ of $v$, $P_{\C[u]}(X)$ does not contain a constant term.
  Thus, $N_{\C}^-(v)=N_{\C'}^-(v)$.
  If $P_{\C[v]}(X) \in \Pi$,
  then there exists an in-neighbor $u$ of $v$ such that
  $P_{\C[u]}(X) \in \Pi$.
  From the inductive hypothesis,
  $P_{\C'[u]}(X) \in \Pi$,
  implying
  $P_{\C'[v]}(X) \in \Pi$.
  Conversely, if $P_{\C'[v]}(X) \in \Pi$,
  then there exists an in-neighbor $u$ of $v$ such that
  $P_{\C'[u]}(X) \in \Pi$.
  From the inductive hypothesis,
  $P_{\C[u]}(X) \in \Pi$,
  implying
  $P_{\C[v]}(X) \in \Pi$.

  Suppose $v$ is a multiplication gate.
  Then, we have $P_{\C[v]}(X) = P_{\C'[v]}(X) \cdot (p(X)+c)$ for some polynomial~$p$
  and a constant term~$c$.
  If $P_{\C'[v]}(X) \not\in \Pi$, then
  every monomial $q'(X)$ in $P_{\C'[v]}(X)$ either is not multilinear or has degree larger than $k$.
  Then, $q'(X)\cdot (p(X)+c)$ cannot create a monomial that is in $\Pi$.
  Hence, $P_{\C[v]}(X) \not\in \Pi$.
  If $P_{\C'[v]}(X) \in \Pi$,
  then there exists a monomial $q'(X)$ in $P_{\C'[v]}(X)$ that is also in $\Pi$.
  Then, $P_{\C[v]}(X)$ contains the term $q(X)=c \cdot q'(X) \in \Pi$, implying $P_{\C[v]}(X) \in \Pi$.
  This completes the proof.
\end{proof}

Next, we study a specific class of circuits that forces certificates to be tree certificates.

\begin{proposition}\label{prp:cert-internal-outdeg}
  Let $\C$ be a scalar-free circuit such that every internal node has out-degree at most $1$.
  Then, every certificate in $\C$ is a tree certificate.
\end{proposition}

\begin{proof}
  Let $\ccert$ be a certificate in $\C$ for node $v$.
  It is trivial when $v$ is a variable node.
  Otherwise, from \Cref{prp:cert-var-outdeg},
  every variable node in $\ccert$ has out-degree $1$,
  and since $\C$ is scalar-free, $\ccert$ does not contain any scalar inputs.
  Hence, every node in $\ccert$ has out-degree at most $1$.
  Then, $\ccert$ cannot create a cycle in the underlying graph,
  making $\ccert$ a tree certificate.
\end{proof}

Finally, we reprove \Cref{lem:williams} using our results.

\mldlemma*

\begin{proof}
  We preprocess scalar inputs in linear time, using \Cref{lem:preprocess-scalar}.
  For an output node $r$ with ${P_{\C[r]}(X)=0}$, clearly the polynomial at $r$
  does not contain any multilinear monomials.
  For an output node $r$ with $P_{\C[r]}(X)$ containing a constant term,
  by definition, the constant term is a multilinear monomial of degree $0$.
  The polynomial at $r$ contains a multilinear monomial of degree at most $k$.

  Let $\C$ be a scalar-free circuit with output node $r$
  (the same argument applies to circuits with multiple output nodes).
  We will show an algorithm to transform a circuit $\C$ to another circuit $\C'$
  such that $P_{\C[r]}(X)=P_{\C'[r]}(X)$,
  every internal node in $\C'$ has out-degree at most $1$,
  and the size of $\C'$ is polynomially bounded in the size of $\C$.

  Consider the following algorithm.
  We process all internal nodes $v$ except $r$ in the reversed topological ordering.
  Let $U$ be the in-neighbors of $v$ and let $W=\{w_1,\ldots,w_\ell\}$ be the out-neighbors of $v$.
  We replace $v$ with $\ell$ copies of $v$, $\{v_i \colon 1 \leq i \leq \ell\}$,
  that is, we remove $v$ and add $\ell$ new nodes with the same type as $v$
  to the current circuit.
  Then, for each $v_i$, we add edges from all of $U$ to $v_i$ and $v_i$ to $w_i$.
  Let $\C'$ be the resulting circuit after processing all internal nodes of $\C$.

  First, observe that this transformation does not change the output polynomial.
  Suppose $\C_1$ and $\C_2$ are the circuits before and after processing node $v$,
  respectively.
  For each $1 \leq i \leq \ell$, we have $P_{\C_1[v]}(X)=P_{\C_2[v_i]}(X)$
  as $N_{\C_1}^-(v)=N_{\C_2}^-(v_i)$.
  Also, $P_{\C_1[w_i]}(X)=P_{\C_2[w_i]}(X)$, which eventually leads to $P_{\C[r]}(X) = P_{\C'[r]}(X)$.
  $\C$ contains a certificate for $r$ if and only if $\C'$ contains a certificate for $r$.

  Second, after processing node $v$, $v$'s out-degree becomes $1$.
  Also, notice that the operation does not change
  the out-degrees of $v$'s descendants.
  Since we proceed in reversed topological ordering,
  every internal node in $\C'$ has out-degree at most $1$.

  Third, to argue the size of $\C'$,
  let ${n=|V(\C)|}$ and $m=|E(\C)|\leq n^2$.
  Suppose we process nodes $v_1,\ldots,v_{n'}$ in order,
  where clearly $n' \leq n$,
  and let $(\C=\C_1),\C_2,\ldots, \C_{n'}, (\C_{n'+1}=\C')$ be the sequence of circuits after processing each node.
  When processing node $v_j$ in $\C_j$, we replace $v_j$ with $\deg_{\C_j}^+(v_j)$ copies of it,
  creating $\C_{j+1}$.
  We will show that $\deg_{\C_j}^+(v_j) \leq m$ for each $j$ by a charging argument.
  Initially, we add a charge to each of $v_j$'s out-edges in $\C_j$.
  We transfer the charges in $\C_j$ to $\C_{j-1}$.
  If a charged edge is present in $\C_{j-1}$, then add a charge to the same edge in $\C_{j-1}$.
  Otherwise, the edge in $\C_j$ must be an edge from $v_j$ to a copy of $v_{j-1}$.
  Let $v_{j-1}^{(1)},\ldots,v_{j-1}^{(\ell)}$ be such copies of $v_{j-1}$.
  Then, $v_{j-1}$ must have $\ell$ out-neighbors $w_1,\ldots,w_\ell$ in~$\C_{j-1}$.
  We transfer a charge on $v_j v_{j-1}^{(i)}$ in $\C_j$
  to edge~$v_j w_i$ in $\C_{j-1}$.
  We continue this process until we transfer all charges to $\C_1$.

  Now, we show that during this process, each edge has at most one charge.
  Suppose not.
  Then, there exists a transfer of multiple charges into the same edge.
  That must involve an edge from $v_{j'}$ to $v_{j'-1}^{(i)}$ in $\C_{j'}$ for some~$i, j'$.
  But since $v_{j'-1}$ is not present in $\C_{j'}$,
  a charge on the edge from $v_{j'-1}$ to $w_i$ in $\C_{j'-1}$ must be the only charge on it,
  a contradiction.

  Hence, we can transfer all $\deg_{\C_j}^+(v_j)$ charges to distinct edges in $\C_1=\C$.
  The total number of charges is upper-bounded by $m$.
  This implies that $|V(\C_{j+1})| \leq |V(\C_j)| + m$,
  and $|V(\C')| \leq n'm \leq nm \leq n^3$.

  Putting these together,
  we can transform any circuit~$\C$ to a circuit $\C'$ in polynomial time.
  From \Cref{prp:cert-internal-outdeg}, every certificate in $\C'$ is a tree certificate.
  Also, there exists a randomized $\Oh^*(2^k)$-time and polynomial space algorithm
  to detect if $\C'$ has a tree certificate (\Cref{lem:tree-cert,lem:williams-poly-weak}),
  which is a one-sided Monte Carlo algorithm with a constant success probability.
  This completes the proof.
\end{proof}

\subsection{Algorithms for Finding Tree Certificates.}
\label{sec:tree-cert-alg}

\recoverymclemma*

\begin{proof}
  Consider \Cref{alg:mc-recovery}.
  This algorithm traverses all nodes in $\C$ from the given output node $r$
  to the variable nodes and removes nodes and edges that are not in a tree certificate.
  Whenever the algorithm sees an addition gate having a path
  to node $r$ in the current circuit, it keeps exactly one in-neighbor.
  This operation is safe due to \Cref{prp:tree-cert-1}.
  Let $\hat{\C}$ be the resulting circuit.
  Every addition gate in $\hat{\C}$ has in-degree $1$,
  and all the in-edges of a multiplication gate are kept
  if there is a path to node $r$ in $\hat{\C}$.
  Assuming that the algorithm solves \MLDshort correctly,
  $\hat{\C}$ is a tree certificate.

  The algorithm fails when \Cref{code:mc:mld} incorrectly concludes
  that $\C - A$ does not have a tree certificate.
  This happens with probability at most $(1-p)^\theta$,
  where $p$ is a constant success probability of \MLDshort (\Cref{lem:williams-poly-weak}).
  By assumption, \Cref{alg:mc-recovery} enters \Cref{code:mc:inner}
  no more than $\alpha$ times.
  Then, the overall failure probability $f(\theta,\alpha)$ is:
  $\sum_{i=0}^{\alpha - 1} (1- (1-p)^\theta)^i\cdot(1-p)^\theta
  =(1-p)^\theta \cdot \frac{1 - (1 - (1 - p)^\theta)^{\alpha}}{1 - (1 - (1 - p)^\theta)}
  =1 - (1 - (1 - p)^\theta)^{\alpha}$.
  For a fixed success probability $p'$ of the algorithm,
  we can find a value $\theta \in \Oh(\log \alpha)$ such that $f(\theta,\alpha)\leq p'$.

  Finally, the expected running time of this algorithm is asymptotically bounded by
  the running time of \Cref{code:mc:mld} as other operations can be done in $\Oh(\alpha m)$ time.
  From \Cref{lem:williams-poly-weak}, the total running time is $\Oh(\alpha \theta\cdot 2^k m)=\Oh(2^k \alpha m \log \alpha)=\tilde{\Oh}(2^k \alpha m)$
  with a constant success probability.
\end{proof}


\begin{algorithm2e}[t]
  \SetAlgoLined
  \KwInput{A recoverable circuit $\C$ with respect to $k$ and output $r$,
  and a failure count threshold $\theta$.\\
  }
  \KwOutput{A tree certificate.}


  \BlankLine
  \tcp{Reversed traversal from the output node.}
  \For{$v \in V(\C)$ in topological ordering of the reverse graph of $\C$}{
    \If{$v \neq r$ and $N^+(v)=\emptyset$}{
        \tcp{Remove unlinked nodes.}
        Let $\C \gets \C - v$.\;
    }\uElseIf{$v$ is an addition gate}{
        \tcp{Perform binary search.}
        \While{$\deg_{\C}^-(v) > 1$}{\label{code:mc:inner}
            Let $A, B $ be a balanced partition of the in-edges of $v$.\;
            \emph{Repeatedly} solve \MLDshort at most $\theta$ times
            with $(\C-A, k)$.\;\label{code:mc:mld}
            \If{$\C-A$ contains a tree certificate}{
                \tcp{Safe to remove $A$.}
                Let $\C \gets \C - A$.\;
            }\Else{
                \tcp{Should keep a vertex in $A$.}
                Let $\C \gets \C - B$.\;
            }
        }
    }
  }

  \KwRet{$\C$}\;
  \caption{\rmc}
  \label{alg:mc-recovery}
\end{algorithm2e}

\begin{algorithm2e}[t]
  \SetAlgoLined
  \KwInput{A recoverable circuit $\C$ with respect to $k$ and output $r$.\\
  }
  \KwOutput{A tree certificate.}



  \BlankLine
  \tcp{Reversed traversal from the output node.}
  \For{$v \in V(\C)$ in topological ordering of the reverse graph of $\C$}{
    \If{$v \neq r$ and $N^+(v)=\emptyset$}{
        \tcp{Remove unlinked nodes.}
        Let $\C \gets \C - v$.\;
    }\uElseIf{$v$ is an addition gate}{
        \tcp{Perform binary search.}
        \While{$\deg_{\C}^-(v) > 1$}{\label{code:lv:inner}
            Let $A, B $ be a balanced partition of the in-edges of $v$.\;

            \tcp{Alternatively set $A$ and $B$.}
            \For{$X \in [A,B,A,B,\ldots ]$}{
                Solve \MLDshort with $(\C-X, k)$.\;\label{code:lv:mld}
                \If{$\C-X$ contains a tree certificate}{
                  \tcp{Safe to remove $X$.}
                  Let $\C \gets \C - X$.\;
                  \textbf{break}\;
                }
            }

        }
    }
  }

  \KwRet{$\C$}\;
  \caption{\rlv}
  \label{alg:lv-recovery}
\end{algorithm2e}

\recoverylvlemma*

\begin{proof}
  Consider \Cref{alg:lv-recovery}.
  This algorithm is identical to \Cref{alg:mc-recovery} except
  for the inner loop starting from \Cref{code:lv:inner}.
  Now, since \Cref{code:lv:mld} is a one-sided error Monte Carlo algorithm,
  if $\C-X$ contains a tree certificate,
  then there exists a tree certificate $\ccert$ including some edge $uv \in A \cup B \setminus X$.
  The set $X$ is safe to remove,
  and with the argument in the proof of \Cref{lem:recovery-mc},
  the algorithm correctly outputs a tree certificate of $\C$.

  For the expected running time,
  note that by assumption, \Cref{alg:lv-recovery} enters \Cref{code:lv:inner} no more than $\alpha$ times.
  Again, the running time of the algorithm is asymptotically bounded by
  the running time of \Cref{code:lv:mld}.

  The expected number of executions of \Cref{code:lv:mld} is $\Oh(\alpha \log n)$ because
  we perform binary search on the $\Oh(n)$ in-edges of an addition gate $v$.
  We know that the ``correct'' edge exists in either $A$ or $B$, and
  the algorithm for \prob{Multilinear Detection} succeeds with a constant probability $p$.
  We expect to see one success for every $2/p$ runs.
  From \Cref{lem:williams-poly-weak}, the total expected running time is
  $\tilde{\Oh}(\alpha \log n \cdot \frac{2}{p} \cdot 2^k m)=\tilde{\Oh}(2^k \alpha m)$.
\end{proof}

\br
\noindent \textbf{Algorithm \simplifycircuit:}

\noindent \textit{Input:} 
A scalar-free recoverable circuit $\C$ with respect to degree $k$ and output $r$.

\noindent \textit{Output:}
A scalar-free recoverable circuit $\C'$ with respect to degree $k$ and output $r$,
where every tree certificate of~$\C'$ contains at most $(2k-1)$ addition gates.

\noindent We initialize $\C'$ to $\C$ and update $\C'$ as follows.

  \begin{enumerate}[label=(\arabic*)]
    \item Replace all in-degree-$1$ multiplication gates in $\C'$
    with addition gates.\label{alg:simplify:step1}

    \item For each addition gate $v$ in $\C'$,
    let $S_v$ be the set of variable nodes and multiplication gates
    that can reach $v$ in $\C'$, using only addition gates.\label{alg:simplify:step2}
    Replace the in-edges $\{uv\colon u \in N_\C^-(v)\}$ of $v$ with the edges
    $\{uv\colon u \in S_v\}$.

    \item Output $\C'[r]$.\label{alg:simplify:step3}
\end{enumerate}

\recoverypreprocess*

\begin{proof}
  First, it is clear to see that Step \ref{alg:simplify:step1} does not change any polynomials.
  Step \ref{alg:simplify:step2} is a technique for constructing the $\mathcal{A}$-circuit
  in \cite{koutis2016limitsapplications}.
  This removes any consecutive addition gates from $\C'$,
  and the safeness is discussed in \cite{koutis2016limitsapplications}.
  This operation does not introduce any new nodes,
  so $\C'$ is scalar-free and contains no more than $n$ nodes.
  Step \ref{alg:simplify:step3} does not affect the polynomials at output nodes.

  For the running time, the entire reachability can be tested in $\Oh(n^2)$ time.
  We process $\Oh(n)$ addition gates, and for each node,
  we remove and add $\Oh(n)$ edges;
  there are $\deg_{\C'}^-(v)$ deletions and $|S_v|$ additions,
  and those numbers are smaller than $n$.
  Hence, the overall running time is $\Oh(n^2)$.

  Now, let $\ccert'$ be a tree certificate in~$\C'$.
  Recall that for each multiplication gate $v$ in $\ccert'$,
  $P_{\ccert'[v]}(X)$ contains one multilinear monomial (\Cref{sec:proof-tree-cert}).
  And we have $P_{\ccert'[v]}(X)=\prod_{u \in N_{\ccert'}^-(v)}P_{\ccert'[u]}(X)$.
  Since $\C'$ is scalar-free and $v$ has at least $2$ in-neighbors in $\C'$,
  the number of variables in $P_{\ccert'[v]}(X)$ is strictly greater than
  the number of variables in $P_{\ccert'[u]}(X)$ for any $u \in N_{\ccert'}^-(v)$.
  Thus, $\ccert'$ contains at most $(k-1)$ multiplication gates.

  Next, we show that $\ccert'$ contains at most $(2k-1)$ addition gates, possibly including $r$.
  We prove by induction on $k$
  that $\ccert'$ contains at most $2k-2$ addition gates
  if the output node $r$ is not an addition gate.

  For the base case with $k=1$,
  $\ccert'$ contains no multiplication gates
  as $\ccert'$ may have at most $(k-1)$ multiplication gates.
  Then, the output node is a variable node, and we cannot place any addition gates,
  which satisfies $0 \leq 2k-2$.
  For the inductive step, suppose the output node is a multiplication gate
  with $\ell$ in-neighbors.
  Also, for each in-neighbor $u_i$ for ${1 \leq i \leq \ell}$,
  suppose $P_{\ccert'[u_i]}(X)$ contains $k_i$ variables.
  Notice that $\sum_{i=1}^\ell k_i = k$.
  The number of addition gates is maximized when all $u_i$ are addition gates.
  Then, from the inductive hypothesis and 
  that $\ccert$ contains no consecutive addition gates,
  the number of addition gates in $\ccert'$ is at most $\sum_{i=1}^\ell 1 + (2k_{\ell}-2)=2k-\ell$.
  This is maximized when $\ell=2$, and we obtain the desired property.
  If the output node $r$ is an addition gate, then its in-neighbor
  must be either a variable node or a multiplication gate.
  Hence, $\ccert'$ has at most $2k-2+1=2k-1$ addition gates.

  To show that $\C'$ is recoverable, let $\ccert$ be a tree certificate in $\C$.
  If $\ccert$ does not contain consecutive addition gates, then
  $\ccert$ should appear in $\C'$
  because each edge in $\ccert$ must be present in $\C'$.
  Otherwise, let $uv_1,\ldots,v_\ell$ be a path in $\ccert$
  where $u$ is not an addition gate, $v_i$ is an addition gate
  for each $1 \leq i \leq \ell$, and $v_\ell$ is either $r$ (output node)
  or an in-neighbor of a multiplication gate.
  By construction, $v_\ell$ remains in $\C'$,
  and there must be an edge $uv_\ell$ in $\C'$.
  We repeatedly replace such a path $uv_1,\ldots,v_\ell$ in $\C$ with an edge $uv_\ell$ in $\C'$
  to construct the circuit $\ccert'$ in $C'$.
  Since $\ccert$ is a tree certificate, those paths are vertex disjoint.
  So, the underlying graph of $\ccert'$ is also a tree.
  Furthermore, since every addition gate in a tree certificate has in-degree $1$,
  we have $P_{\ccert'[r]}(X) = P_{\ccert[r]}(X)$.
  $\ccert'$ is a tree certificate in~$\C'$,
  which means that $\C'$ is recoverable.
\end{proof}

\begin{figure}[t]
    \pgfdeclarelayer{bg}
    \pgfsetlayers{bg, main}

    \tikzstyle{bigblacknode} = [circle, fill=gray, text=white, draw, thick, scale=1, minimum size=0.6cm, inner sep=1.5pt]
    \tikzstyle{bigwhitenode} = [circle, fill=white, text=black, draw, thick, scale=1, minimum size=0.6cm, inner sep=1.5pt]

    \tikzstyle{blacknode} = [circle, fill=gray, draw, thick, scale=1, minimum size=0.2cm, inner sep=1.5pt]
    \tikzstyle{whitenode} = [circle, fill=white, draw, thick, scale=1, minimum size=0.2cm, inner sep=1.5pt]

    \tikzstyle{hugewhitenode} = [circle, fill=white, text=black, draw, thick, scale=1, minimum size=1.5cm, inner sep=1.5pt, font=\large]
    \tikzstyle{directed} = [color=black, arrows=- triangle 45]

    \definecolor{myblue}{RGB}{5,113,176}
    \definecolor{mypurple}{RGB}{123,50,148}
    \definecolor{myred}{RGB}{202,0,32}
    
    \tikzset{
        old inner xsep/.estore in=\oldinnerxsep,
        old inner ysep/.estore in=\oldinnerysep,
        double circle/.style 2 args={
            circle,
            old inner xsep=\pgfkeysvalueof{/pgf/inner xsep},
            old inner ysep=\pgfkeysvalueof{/pgf/inner ysep},
            /pgf/inner xsep=\oldinnerxsep+#1,
            /pgf/inner ysep=\oldinnerysep+#1,
            alias=sourcenode,
            append after command={
            let     \p1 = (sourcenode.center),
                    \p2 = (sourcenode.east),
                    \n1 = {\x2-\x1-#1-0.5*\pgflinewidth}
            in
                node [inner sep=0pt, draw, circle, minimum width=2*\n1,at=(\p1),#2] {}
            }
        },
        double circle/.default={-3pt}{black}
    }

    \centering
    \begin{minipage}[m]{\linewidth}
        \vspace{0pt}
        \centering
        \begin{tikzpicture}
            \node[bigblacknode] (x) at (0, 0) {$y$};
            \node[bigblacknode] (y) at (0, 1.6) {$x$};
            \node[bigwhitenode] (a1) at (1.6, 0) {$+$};
            \node[bigwhitenode] (a2) at (1.6, 1.6) {$+$};
            \node[bigwhitenode] (a3) at (3.2, 0.8) {$+$};
            \node[bigwhitenode] (m1) at (3.2, 1.6) {$\times$};
            \node[bigwhitenode] (m2) at (4.8, 0) {$\times$};
            \node[bigwhitenode, double circle] (out) at (6.4, 0.8) {$+$};

            \draw[directed] (y) -- (a1);
            \draw[directed] (y) -- (a2);
            \draw[directed] (x) -- (a2);
            \draw[directed] (x) -- (a3);
            \draw[line width=0.5mm,directed,draw=mypurple] (y) -- (a3);
            \draw[line width=0.5mm,directed,draw=mypurple] (x) -- (a1);
            \draw[directed] (x) -- (m1);
            \draw[line width=0.5mm,directed,draw=mypurple] (a1) -- (m2);
            \draw[directed] (a2) -- (m1);
            \draw[line width=0.5mm,directed,draw=mypurple] (a3) -- (m2);
            \draw[directed] (m1) -- (out);
            \draw[line width=0.5mm,directed,draw=mypurple] (m2) -- (out);

            \node () at (1.6, 2.1) {$x+y$};
            \node () at (3.2, 2.1) {$(x+y)y$};
            \node () at (1.6, -0.6) {$x+y$};
            \node () at (4.0, 0.8) {$x+y$};
            \node () at (4.8, -0.6) {$(x+y)^2$};
            \node () at (6.0, 1.5) {$(x+y)y+(x+y)^2$};
        \end{tikzpicture}
    \end{minipage}
    \caption{%
      The result of \simplifycircuit applied on the circuit in \Cref{fig:example-certificates}.
      A tree certificate for this circuit is shown in purple,
      but it does not suggest the tree certificate in the original circuit.
    }
    \label{fig:example-simplified}
\end{figure}

\Cref{fig:example-simplified} shows the circuit $\C'$ obtained by applying
\simplifycircuit on the circuit ($\C$) in \Cref{fig:example-certificates}.
Observe that there are no consecutive addition gates in $\C'$.
One tree certificate for $\C'$ is shown in purple, but it does not (directly) help
find the tree certificate in the original circuit (blue in \Cref{fig:example-certificates}).
Each new edge $uv$ in $\C'$ represents a path from $u$ to $v$ in $\C$.
However, there is no guarantee that such paths are vertex disjoint.

\subsection{Proofs of Circuit Construction.}

\lemconstructnaive*

\begin{proof}
  For each layer, there are $\Oh(\hi k n)$ addition gates $T_{t',v,d,c}$ and
  for each of them, there are $\Oh(kn)$ multiplication gates $r$ that link between layers.
  Hence, there are $\Oh(\hi t k^2 n^2 + k + \hi)=\Oh(\hi t k^2 n^2)$ nodes in total.
  By construction, there are $\Oh(\hi t k^2 n^2 + 2\hi t k^2 n^2 + \hi t k^2 n)=\Oh(\hi t k^2 n^2)$ edges.

  To show the correctness, suppose there is a solution walk $(s, v_1, v_2, \ldots, s)$ with weight $\ell$.
  Then, there must be a corresponding sequence $(v_1,c_1),\ldots,(v_t,c_t)$
  such that $\{v_i\}$ is an ordered (not necessarily distinct) vertex sets,
  and $\{c_i\}$ is a distinct set of colors with $c_i \in \col(v_i)$.
  Such a distinct set of colors exists because the solution walk collects at least $t$ colors.
  Let $d_i$ be the distance from $s$ to $v_i$ in this walk, i.e.
  $d_1=w(s,v_1), d_2=d_1+w(v_1,v_2), \ldots, d_i=d_{i-1}+w(v_{i-1}, v_i)$.
  Then, we construct a tree certificate as follows:
  pick $T_{i,v_i,d_i,c_i}$ for every computational layer $1 \leq i \leq t$,
  and connect $T_{t,v_t,d_t,c_t}$ to the output node~$O_\ell$.
  Observe that there is a path from $T_{1,v_1,d_1,c_1}$ to $O_\ell$ including all $T_{i,v_i,d_i,c_i}$,
  because by assumption $w(s,v_1)+w(v_1,v_2)+\ldots+ w(v_t,s)=d_t + w(v_t,s)=\ell$.
  We extend this path by adding all in-neighbors of any multiplication gates to construct $\hat{\C}$.
  This includes distinct $t$ colors $\{c_i\}$, so $\hat{\C}$ is a tree certificate for \MLDshort.

  Now, suppose there exists a tree certificate $\hat{\C}$ for \MLDshort.
  It is clear to see from construction that every monomial in $P_{\C[T_{t',v,d,c}]}(X)$ has degree $t'$
  and every monomial in $P_{\C[O_\ell]}(X)$ has degree $t$.
  Since the underlying graph of $\hat{\C}$ is a tree,
  $\hat{\C}$ includes exactly one node $T_{i,v_i,d_i,c_i}$ for each layer $1 \leq i \leq t$.
  For the same reason, the set $\{c_i\}$ is distinct.
  Also, we know that $\hat{\C}$ includes $t$ addition gates, each of which has only one in-neighbor.
  Consider a walk $(s, v_1, v_2, \ldots, v_t, s)$.
  This walk collects at least $t$ colors, and its total weight is
  $\ell$.
  Hence, this is a solution walk.
\end{proof}

\lemconstructstandard*

\begin{proof}
  For each layer, there are $\Oh(\hi kn)$ addition gates with in-degree $\Oh(kn)$
  and $\Oh(\hi kn)$ multiplication gates with in-degree $2$.
  There are $\Oh(\hi)$ output nodes with in-degree $kn$.
  Hence, in total there are $\Oh(\hi tkn)$ nodes and
  $\Oh(\hi tk^2n^2 + 2\hi t kn + \hi kn) = \Oh(\hi tk^2n^2)$ edges.

  The correctness proof is similar to that of \Cref{lem:construct-naive}.
  If there is a solution walk $(s, v_i, v_2, \ldots, s)$ with weight $\ell$,
  then there is a path including $T_{i,v_i,d_i,c_i}$ and the output node~$O_\ell$,
  as defined in the proof of \Cref{lem:construct-naive}.
  This path and the set of collected colors $\{c_i\}$ induce a tree certificate.

  If there is a tree certificate $\hat{\C}$, then $\hat{\C}$ must include
  $t$ addition gates and $T_{i,v_i,d_i,c_i}$ for each layer $1 \leq i \leq t$.
  The walk $(s, v_1, \ldots, v_t, s)$ will be a solution walk.
\end{proof}

\lemconstructcompact*

\begin{proof}
  In addition to $k$ variable nodes,
  there are $\Oh(t n)$ addition nodes $a_{t',v}$,
  $\Oh(\hi t n)$ addition nodes $T_{t',v,d}$,
  $\Oh(\hi t n)$ multiplication nodes $R_{t',v,d}$,
  and $\Oh(\hi)$ output nodes.
  The number of nodes is $\Oh(k+tn+2 \hi tn+\hi)
  =\Oh(\hi tn+k)$.
  To obtain the number of edges, we count in-degrees of those nodes.
  Each of $a_{t',v}$ has $\Oh(k)$ in-neighbors,
  each of $T_{t',v,d}$ has $2$ in-neighbors,
  and each of $R_{t',v,d}$ has $\Oh(n)$ in-neighbors.
  For the output nodes, if the search strategy is \sunified,
  there are $\Oh(\hi)$ nodes with in-degree $\Oh(n)$.
  Otherwise, there is $1$ node with in-degree $\Oh(\hi n)$.
  In either case, there will be $\Oh(\hi n)$ edges to the output.
  The total number of edges is
  $\Oh(tkn + 2\hi tn + \hi tn^2+\hi n)=\Oh(\hi tn^2+tkn)$.

  To show the correctness, suppose there is a solution walk of weight $\ell$.
  Since the instance is complete and metric, there exists a solution walk
  $W=s v_1, \ldots, v_p, s$ of weight at most $\ell$ with no repeated vertices other than $s$
  such that at every vertex $v$ in $V(W) \setminus \{s\}$ collects at least one new color.
  Then, we create a sequence $(v_1,c_1),\ldots,(v_t,c_t)$ as follows.
  First, pick exactly $t$ colors $C \subseteq \bigcup_{v \in V(W)}\col(v)$
  so that we still collect at least one new color at every vertex in $V(W) \setminus \{s\}$.
  Let $C_v \subseteq C$ be the newly collected colors at vertex $v$.
  Then, when we see a new vertex $v$ in $V(W) \setminus \{s\}$, append $\{(v, c)\colon c \in C_v\}$
  to the sequence.
  Now, we have $\{v_i\}$ as an ordered (not necessarily distinct) vertex sets,
  and $\{c_i\}$ is a distinct set of colors with $c_i \in \col(v_i)$.
  We write $d_i$ for the distance from $s$ to $v_i$ in the walk $W$,
  i.e. $d_1=w(s,v_1), d_2=d_1+w(v_1,v_2),\ldots$ with setting $w(v_i,v_i)=0$.
  By assumption, we have $d_t + w(v_t,s) \leq \ell$.

  We construct a tree certificate as follows.
  Let $S \subseteq V(\C)$ be a set of nodes such that 
  $S = \{T_{i,v_i,d_i} \colon 1 \leq i \leq t\} \cup \{R_{i,v_i,d_i}\colon 1<i\leq t\}
  \cup \{a_{i,v_i}\colon 1 \leq i \leq t\} \cup \{O_{\ell'}\}$,
  where $\ell'=d_t + w(v_t,s) \leq \ell$ for \sunified and $\ell'=\ell$ for the others.
  Let $\C' := \C[S]$ and observe that the underlying graph of $\C'$ is a tree.
  Then, we add variable nodes $\{c_i \colon 1 \leq i \leq t \}$ and
  edges $\{c_i a_{i,v_i}\colon 1 \leq i \leq t\}$ to~$\C'$.
  It is clear to see that $\C'$ contains $2t$ addition gates.
  Also, its underlying graph remains a tree because $c_i$ is distinct.
  By construction, $P_{\C'[O_{\ell'}]}(X)$ is a multilinear monomial representing
  the color set $\{c_i\}$ of size $t$.
  Observe that $\C'$ is a tree certificate for \MLDshort.

  Conversely, suppose there exists a tree certificate $\hat{\C}$ for \MLDshort
  with respect to weight $\ell$.
  Every multiplication gate in $\hat{\C}$ has the same in-neighbors as in $\C$,
  and from \Cref{prp:tree-cert-1}, every addition gate in $\hat{\C}$ has degree $1$ in $\hat{\C}$.
  This leaves us one structure:
  $t$ variable nodes $\{c_i\}$,
  their out-neighbors $\{a_{i,v_i}\}$ such that $c_i \in \col(v_i)$,
  multiplication gates $T_{i,v_i,d_i}$ in layers $1 \leq i \leq t$,
  accompanied addition gates $R_{i,v_i,d_i}$ for $i>1$,
  and the output node $O_\ell$.
  Consider a walk $(s, v_1, v_2, \ldots, v_t, s)$.
  This walk collects $t$ colors, and its total weight is at most $\ell$.
  Hence, this is a solution walk.
\end{proof}

\lemconstructsemicompact*

\begin{proof}
  We have $\Oh(tn)$ auxiliary nodes $a_v^{(i)}$ with in-degree~$\Oh(k)$.
  The first layer contains $\Oh(n)$ nodes with in-degree~$1$.
  For each layer $1<t'\leq t$, there are
  $\Oh(\hi n)$ addition gates $R_{t',v,d,1}$ with in-degree~$\Oh(tn)$,
  $\Oh(\hi tn)$ addition gates $R_{t',v,d,i}$ with $i>1$ and in-degree $1$,
  $\Oh(\hi tn)$ multiplication gates with in-degree~$2$.
  There are $\Oh(\hi)$ output nodes and $\Oh(\hi tn)$ edges to the output.
  Hence, there are $\Oh(\hi t^2n + k)$ nodes and
  $\Oh(tkn + \hi t^2n^2 + \hi t^2n + 2\hi t^2n + \hi tn) = \Oh(\hi t^2n^2 + tkn)$ edges in total.
  
  To show the correctness, suppose there is a solution walk of weight $\ell$.
  Since there exists a solution walk $W=(s, v_1, \ldots, s)$ of weight at most $\ell$
  with no repeated vertices other than $s$ such that every vertex $v$ in $V(W) \setminus \{s\}$
  collects at least one new color.
  Then, we create a sequence $(v_1,c_1,\mu_1),\ldots,(v_t,c_t,\mu_2)$ as follows.
  First, pick exactly $t$ colors $C \subseteq \bigcup_{v \in V(W)}\col(v)$
  so that we still collect at least one new color at every vertex in $V(W) \setminus \{s\}$.
  Let $C_v \subseteq C$ be the newly collected colors at vertex $v$.
  Then, when we see a new vertex $v$ in $V(W) \setminus \{s\}$,
  append $(v, c_{v,1}, 1), (v, c_{v,2},2), \ldots$
  to the sequence,
  where $C_v=\{c_{v,1}, c_{v,2},\ldots\}$.
  Now, we have $\{v_i\}$ as an ordered (not necessarily distinct) vertex sets,
  $\{c_i\}$ is a distinct set of colors with $c_i \in \col(v_i)$,
  and $\{\mu_i\}$ represents how many colors are collected at vertex $v_i$ so far.
  We write $d_i$ for the distance from $s$ to $v_i$ in the walk $W$,
  i.e., $d_1=w(s,v_1), d_2=d_1+w(v_1,v_2),\ldots, d_i = d_{i-1} + w(v_{i-1},v_i)$, with setting $w(v_i,v_i)=0$.
  By assumption, we have $d_t + w(v_t,s) \leq \ell$.

  We construct a tree certificate as follows.
  Let $S \subseteq V(\C)$ be a set of nodes such that 
  $S = \{T_{i,v_i,d_i,\mu_i} \colon 1 \leq i \leq t\} \cup \{R_{i,v_i,d_i,\mu_i}\colon 1<i\leq t\}
  \cup \{a_{v_i}^{(\mu_i)}\colon 1 \leq i \leq t\} \cup \{O_{\ell'}\}$,
  where $\ell'=d_t + w(v_t,s) \leq \ell$ for \sunified and $\ell'=\ell$ for the others.
  Let $\C' := \C[S]$ and observe that the underlying graph of $\C'$ is a tree.
  By construction, the pair $(v_i,\mu_i)$ is unique in the sequence.
  Then, we add variable nodes $\{c_i \colon 1 \leq i \leq t \}$ and
  edges $\{c_i a_{v_i}^{\mu_i}\colon 1 \leq i \leq t\}$ to $\C'$.
  It is clear to see that $\C'$ contains $2t$ addition gates.
  Also, its underlying graph remains a tree because $c_i$ is distinct.
  By construction, $P_{\C'[O_{\ell'}]}(X)$ is a multilinear monomial representing
  the color set $\{c_i\}$ of size $t$.
  $\C'$ is a tree certificate for \MLDshort.

  Conversely, suppose there exists a tree certificate $\hat{\C}$ for \MLDshort
  with respect to weight $\ell$.
  Every multiplication gate in $\hat{\C}$ has the same in-neighbors as in $\C$,
  and from \Cref{prp:tree-cert-1}, every addition gate in $\hat{\C}$ has degree $1$ in $\hat{\C}$.
  This leaves us one structure:
  $t$ variable nodes $\{c_i\}$,
  their out-neighbors $\{a_{v_i}^{(\mu_i)}\}$ such that $c_i \in \col(v_i)$,
  multiplication gates $T_{i,v_i,d_i,\mu_i}$ in the computational layers $1 \leq i \leq t$,
  accompanied addition gates $R_{i,v_i,d_i,\mu_i}$ for $i>1$,
  and the output node $O_\ell$.
  Consider a walk $(s, v_1, v_2, \ldots, v_t, s)$.
  This walk collects $t$ colors, and its total weight is at most $\ell$.
  This is a solution walk.
\end{proof}

\subsection{Proofs of Search Algorithms.}\label{sec:proof-search}

\lemSearchBS*

\begin{proof}
  We have $\theta' \in \tilde{\Oh}(\theta)$.
  It is known that binary search requires at most~$\log_2(\diff)$ evaluations,
  and each evaluation takes at most $\theta' f(\hi)$ time.
  Hence, the running time is $\Oh(\theta' f(\hi)\log(\diff))
  \subseteq \tilde{\Oh}(\theta \cdot f(\hi))$.

  The algorithm succeeds when all the $\log_2(\diff)$-many evaluations succeed.
  This probability is
  \begin{align*}
    (1-(1-p)^{\theta'})^{\log_2(\diff)} \geq 1 - (1-p)^\theta
  \end{align*}
  by our choice of $\theta'$.%
\end{proof}

\begin{algorithm2e}[ht]
  \SetAlgoLined
  \KwInput{An instance $\mathcal{I}$ of \gi{},
  bounds of the optimal weight $\lo \leq \hi$,
  a failure count threshold $\theta'$,
  and a probability $p'$.\\
  }
  \KwOutput{Optimal weight.}

  \BlankLine
  \tcp{Maintain midpoints as a stack.}
  Create an empty stack $S$.
  
  \BlankLine
  \While{$\lo < \hi$}{
    \If{$S$ is empty}{
      Push $(\lfloor \frac{\lo + \hi}{2} \rfloor, 0)$ to $S$.\;
    }
    Pop the top element $(\ell, c)$ from $S$.\;

    \BlankLine
    Create a circuit $\C$ of $\mathcal{I}$ for $\ell$.\;
    Solve \MLDshort with $(\C, k)$ and get result \textsf{out}.\;

    \BlankLine
    \If{$\textnormal{\textsf{out}} = \true$} {
      Let $\hi \gets \ell$.\;
    }
    \Else {
      \If {$c + 1 \geq \theta'$}{
        Let $\lo \gets \ell + 1$.\tcp*{reject $\ell$}
        Clear $S$.
      }
      \Else {
        Push $(\ell, c+1)$ to $S$.\;
        \If {$\ell < \hi - 1$} {
          \tcp{Randomly go \emph{higher}.}
          With probability $p'$, push $(\lfloor \frac{\ell + 1 + \hi}{2} \rfloor, 0)$ to $S$.\;
        }
      }
    }
  }
  \KwRet{$\hi$}\;
  \caption{\sprob}
  \label{alg:search-prob}
\end{algorithm2e}

\lemSearchProbBS*


\begin{proof}
  The algorithm fails when 
  for any feasible $\ell$, it observes no successes and $\theta'$ failures.
  For a fixed $\ell$, this probability is at most $(1-p)^{\theta'}$,
  and there are at most $\diff$ possible values to check.
  By the same argument for \sstandard,
  by setting $\theta' \in \tilde{\Oh}(\theta)$ such that $(1-(1-p)^{\theta'})^{\diff} \geq 1 - (1-p)^\theta$,
  we can achieve success probability $1-(1-p)^\theta$.

  To argue the running time,
  let $T(\tilde{\ell},\ell)$ be the expected number of runs of an \MLDshort solver
  for $\ell < \hi$, where the optimal weight is $\tilde{\ell}$.
  This is sufficient as \Cref{alg:search-prob} never evaluates $\hi$.

  We consider three cases.
  If $\ell$ is feasible, that is, $\ell \geq \tilde{\ell}$, then
  unless the algorithm fails, it will eventually find that $\ell$ is feasible
  because whenever the algorithm search for a higher value, $\ell$ is always in the stack.
  Hence, $T(\tilde{\ell},\ell) \leq p^{-1}$.
  Next, if $\ell = \tilde{\ell}-1$,
  then the algorithm must try $\theta'$ evaluations to conclude that $\ell$ is infeasible.
  We have $T(\tilde{\ell},\ell) = \theta'$.
  Lastly, if $\ell < \tilde{\ell}-1$,
  the algorithm moves higher with probability $p'$.
  If it goes to another infeasible value (lucky case), it will never come back to $\ell$.
  If it moves to a feasible value, it will then come back after finding that that value is feasible.
  Notice that the latter case only happens at most $\log_2 (\hi - \ell)$ times because
  when we come back from the higher part, the remaining search space will be shrunk into half.
  Hence, $T(\tilde{\ell},\ell) \leq (p')^{-1} + \log_2 (\hi - \ell)$.

  Putting these together, we sum over all possible $\tilde{\ell},\ell$,
  assuming each of $\tilde{\ell}$ appears with probability $\diff^{-1}$.
  The expected running time is
  \begin{align*}
    &\frac{f(\hi)}{\diff} \sum_{\tilde{\ell}=\lo}^{\hi}\sum_{\ell=\lo}^{\hi -1}T(\tilde{\ell},\ell)\\
    \leq \ & \frac{f(\hi)}{\diff} \left(\diff \theta' + \diff^2(p^{-1}+(p')^{-1}+\log_2(\diff))\right)\\
    \in \ & \tilde{\Oh}((\theta + \diff) \cdot f(\hi)),
  \end{align*}
  as desired.
  Note that $p$ and $p'$ are constants.
\end{proof}

\lemSearchUnified*

\begin{proof}
  Since there are at most $\theta$ circuit evaluations,
  the overall running time is bounded by $\theta \cdot f(\hi)$.
  Assuming the circuit construction is correct,
  we have $\textsf{out}(\ell)=\false$ for every $\ell < \tilde{\ell}$.
  The algorithm fails only when $\textsf{out}(\tilde{\ell})$ is evaluated to \false
  $\theta$ times consecutively.
  This happens with probability at most $(1-p)^\theta$, which completes the proof.
\end{proof}

\subsection{Proofs on Two Phase Recovery}\label{appendix:two-phase}


\lemconsistentwalk*

\begin{proof}
  Recall that we assume $\chi(s)=\emptyset$.
  Let $\tilde{W} = (s,a_1,\ldots,a_p,s)$ be an optimal walk with weight~$\ell$.
  Let $\colorset' \subseteq \colorset$ be any color set of size~$t$ that is collected in $\tilde{W}$.
  We define, for $1 \leq i \leq p$, $f_i$ as the subset of $\colorset'$ that are newly collected at $a_i$,
  that is, $f_i := (\chi(a_i) \setminus \bigcup_{j=1}^{i-1} \chi(a_j)) \cap \colorset'$.
  By definition, $\bigcup_{i=1}^p f_i = \colorset'$.
  We construct the walk $W_1$ by repeatedly applying the following operation:
  if $f_i=\emptyset$ for some $i>0$, replace $(a_{i-1},a_i,a_{i+1})$ with $(a_{i-1},a_{i+1})$.
  Since $G$ is complete and metric,
  there must be the edge $a_{i-1} a_{i+1}$.
  This operation never increases the weight and collects the same color set.
  Hence, $W_1$ is also an optimal walk with weight $\ell$.
  The operation always makes a shorter walk, so it terminates.
  Also, $W_1$ does not contain any repeated vertices except the terminal $s$.

  Next, we construct the walk $W_2$ by applying the following operation to $W_1$:
  for each $i$, duplicate vertex~$a_i$ $(|f_i|-1)$ times.
  Here, \emph{duplicate} means replacing one occurrence of $(a_i)$ in a walk with $(a_i,a_i)$.
  Again, this operation does not add extra weights, so $W_2$ is also an optimal walk.

  Notice that the internal vertices in $W_2$ begin with $|f_1|$ copies of $a_1$,
  and in general,
  $|f_i|$ copies of~$a_i$, followed by $|f_{i+1}|$ copies of~$a_{i+1}$.
  Since $\sum_{i=1}^p |f_i|=|\colorset'|=t$, $W_2$ contains $t$ internal vertices.
  Let $W_2=(s,v_1,\ldots,v_t,s)$.
  Then, $v_i=a_j$ if and only if $\sum_{k=1}^{j-1}|f_k| < i \leq \sum_{k=1}^{j}|f_k|$.

  Now, we define the ordered color set $C$ by concatenating $f_1,\ldots,f_p$ in order
  (the ordering of each $f_i$ may be arbitrary).
  By definition, $C$ contains $t$ distinct colors, and let $C=\{c_1,\ldots,c_t\}$.

  Observe the correspondence between vertices~$v_i$ in~$W_2$ and colors~$c_i$ in~$C$.
  Color $c_i$ is in $f_j$ if ${\sum_{k=1}^{j-1}|f_k| < i \leq \sum_{k=1}^{j}|f_k|}$.
  Hence, if $v_i = a_j$, then ${c_i \in f_j \subseteq \chi(a_j) = \chi(v_i)}$.
  This proves that $W_2$ is a walk of weight $\ell$ consistent with $C$.
\end{proof}

\subsubsection{Finding Optimal Color Order.}

\begin{algorithm2e}[t]
  \SetAlgoLined
  \KwInput{%
    A circuit $\C$ and a set $S$ of edges to invalidate.\\
  }
  \KwOutput{An updated circuit.}

  \BlankLine
  
  \tcp{Forward traversal from input nodes.}
  \For{$v \in V(\C)$ in topological ordering of $\C$}{
    Let $E^-(v)$ be the set of the in-edges of $v$.\;
    Let $E^+(v)$ be the set of the out-edges of $v$.\;

    \If{$v$ is an addition gate}{
      \If{$E^-(v) \subseteq S$} {\label{code:invalidate:add}
        \tcp{All inputs are invalid, so $v$ is invalid.}
        Let $S \gets S \cup E^+(v) \cup \{v\}$.\;
      }
    }\uElseIf{$v$ is a multiplication gate}{
      \If{$E^-(v) \cap S \neq \emptyset$} {\label{code:invalidate:mul}
        \tcp{One invalid in-edge makes $v$ invalid.}
        Let $S \gets S \cup E^-(v) \cup E^+(v) \cup \{v\}$.\;
      }
    }
  }

  \KwRet{$\C - S$}\;
  \caption{\invalidateedges}
  \label{alg:invalidate-edges}
\end{algorithm2e}

Before proving \Cref{lem:tp-color-recovery}, we introduce a new operation,
\emph{invalidating edges}, in a circuit, detailed in \Cref{alg:invalidate-edges}.
We need this operation, instead of removing edges, to make sure that
every multiplication gate with an invalid in-edge is also invalid,
and every addition gate with no in-edges is invalid.
We first show that the algorithm returns a valid sub-circuit.

\begin{lemma}\label{lem:invalidate-valid}
  \invalidateedges returns a sub-circuit~$\C'$ of~$\C$,
  where for each internal node $v$ in $\C$,
  $v$ either does not exist or has at least one in-neighbor in $\C'$.
\end{lemma}

\begin{proof}
  Suppose there exists a node $v$ in $\C'$
  such that $N_{\C}^-(v)\neq \emptyset$ and $N_{\C'}^-(v) = \emptyset$.
  Then, all the in-neighbors of $v$ must be invalidated,
  and the algorithm invalidates node $v$ as well, a contradiction.
\end{proof}

Invalidating edges maintains the following properties.

\begin{lemma}\label{lem:invalidate-cert}
  For a circuit $\C$ and edges $S \subseteq E(\C)$,
  let $\C'$ be the circuit after invalidating $S$.
  If there is a tree circuit $\hat{\C}$ in $\C$ such that $S \cap E(\hat{\C})=\emptyset$,
  then $\hat{\C}$ is also a sub-circuit of $\C'$.
\end{lemma}

\begin{proof}
  From \Cref{lem:invalidate-valid}, $\C'$ is a sub-circuit of $\C$.
  Assume towards a contradiction that there exists an edge~$uv \in E(\ccert)$
  such that $uv \not\in \C'$.
  By definition, any edge in $\ccert$ must not be in $S$.
  We choose $uv$ such that $u$ is a source in $\ccert - E(\C')$.
  That is, there is no $w \in V(\ccert)$ such that $wu \in E(\ccert) \setminus E(\C')$.
  Since $\ccert$ is a DAG, such an edge must exist.
  Now, consider the following cases.

  First, suppose $u$ is an addition gate with the only in-neighbor $w$ in $\ccert$,
  where we know $wu \not\in S$.
  Then, all of $u$'s in-edges in $\C$ must be invalidated (\Cref{code:invalidate:add}).
  We have $wu \not\in \C'$,
  which violates our choice of $uv$.

  Next, suppose $u$ is a multiplication gate.
  Then, by the definition of a tree certificate,
  all of $u$'s in-neighbors in $\C$ must be included in $\ccert$.
  Hence, $wu \not\in S$ for every $w \in N_{\C}^-(u)$.
  However, there must be an in-edge~$wu$ of $u$ that was invalidated before $uv$,
  and thus $wu \not\in \C'$.
  This again contradicts our choice of $uv$.

  Therefore, $u$ must be an input node.
  And since $uv \not\in S$, $v$ must be a multiplication gate,
  and for every in-neighbor $w$ of $v$ in $\C$, $wv$ must be invalidated
  and also present in $\hat{\C}$.
  Since $S \cap E(\ccert) = \emptyset$,
  for each $w \in N_{\C}^-(v)$,
  we have $wv \not\in S$, and
  from the previous arguments, $w$ is an input node.
  Then, \Cref{alg:invalidate-edges} cannot invalidate any $wv$,
  a contradiction.
\end{proof}

The following characterizes the degrees of monomials
after we invalidate edges.

\begin{lemma}\label{lem:invalidate-degree}
  For a circuit $\C$ and edges $S \subseteq E(\C)$,
  let $\C'$ be the circuit after invalidating $S$.
  For each node~$v$ in~$\C'$,
  let $D_v$ and $D'_v$ be the set of the degrees of all monomials in
  $P_{\C[v]}(X)$ and $P_{\C'[v]}(X)$, respectively.
  Then, $D'_v \subseteq D_v$.
\end{lemma}

\begin{proof}
  Assume towards a contradiction that
  there exists a node~$v$ in~$\C'$ such that
  $D'_v \not\subseteq D_v$.
  We choose $v$ as such a node that is the earliest in some topological ordering.
  Trivially, $v$ cannot be an input node because $P_{\C[v]}(X) = P_{\C'[v]}(X)$.
  Since $v$ is the earliest node such that $D'_v \not\subseteq D_v$ in topological ordering,
  for each in-neighbor $u$ of $v$ in $\C'$, we have $D'_u \subseteq D_u$.

  Suppose $v$ is an addition gate.
  From \Cref{lem:invalidate-valid}, $v$ has at least one in-neighbor.
  Notice that $D_v = \bigcup_{u \in N_{\C}^-(v)} D_u$.
  Then, we have $D'_v = \bigcup_{u \in N_{\C'}^-(v)} D'_u \subseteq \bigcup_{u \in N_{\C'}^-(v)} D_u \subseteq D_v$,
  a contradiction.

  Now, suppose $v$ is a multiplication gate.
  If there is an in-neighbor $u$ of $v$ in $\C$ such that edge $uv$ is invalidated,
  then node $v$ must be invalidated by \Cref{code:invalidate:mul},
  contradicting that $v$ is in $\C'$.
  Hence, $N_{\C}^-(v)=N_{\C'}^-(v)$.
  Let $d$ be a degree in $D'_v \setminus D_v$.
  Notice that $d$ is the sum over $\{ d_u \}$ for each $u \in N_{\C'}^-(v)=N_{\C}^-(v)$
  such that $d_u \in D'_u$.
  However, since $D'_u \subseteq D_u$, $d$ must be included in~$D_v$,
  a contradiction.
\end{proof}

Now we are ready to prove \Cref{lem:tp-color-recovery}.

\lemtpcolorrecovery*

\begin{algorithm2e}[t]
  \SetAlgoLined
  \KwInput{%
    An instance $(G=(V,E),\colorset,w,\chi,s,t)$ of \gi
    and its recoverable circuit $\C$ with respect to output $r$ and objective $\ell$.\\
  }
  \KwOutput{A color ordering $C$.}

  \BlankLine
  Initialize $C$ to an empty sequence.\;
  Let $R \gets \colorset$. \tcp*{Stores remaining colors.}
  \For{$i \gets t$ \KwSty{down to} $1$}{
    Identify or create the color control nodes $Y$ for layer $i$.\;\label{code:co:y}
    Let $Z \gets R$. \tcp*{Candidates for layer $i$.}
    Let $\C' \gets \C$.\;

    \tcp{Binary-search the color for layer $i$.}
    \While{$|Z| > 1$}{
      Let $A, B$ be a balanced partition of $Z$.\;
      \For{$X \in [A,B,A,B,\ldots]$}{
        Let $E'$ be the set of the edges from the variable nodes for $X$ to $Y$.\;
        Let $\C' \gets \invalidateedges(\C,E')$.\;

        \If{$N_{\C'}^-(r)=\emptyset$}{
          \tcp{Output node is invalid.}
          \KwSty{continue}\;
        }
        Solve \MLDshort with $(\C', k)$.\;\label{code:tp:mld}
        \If{$\C'$ contains a tree certificate}{
          \tcp{Safe to remove $X$.}
          Let $Z \gets Z - X$.\;
          \KwSty{break}\;
        }
      }
    }
    \tcp{Now, $Z$ contains only one color.}
    Let $Z = \{c\}$.\;
    Insert $c$ into the front of $C$.\;
    Let $R \gets R \setminus \{c\}$.\;

    Let $x_c$ be the variable node for $c$.\;
    Let $E'' \gets \{x_c w\colon w \in N^+(x_c) \setminus Y \}$.\;\label{code:co:e}
    Let $\C \gets \invalidateedges(\C',E'')$.\;
  }

  \KwRet{$C$}\;
  \caption{\rco}
  \label{alg:color-order-recovery}
\end{algorithm2e}

\begin{proof}
  Consider \Cref{alg:color-order-recovery}, describing a high-level procedure
  of our color order recovery.
  It uses \Cref{alg:invalidate-edges} as a subroutine.
  We need to \emph{invalidate} edges instead of removing them
  in order to maintain the property that every monomial of any output of layer $i$ has degree $i$
  (\Cref{lem:invalidate-degree}).

  Specifics of \Cref{code:co:y}, which determines the color control nodes $Y$,
  depend on the type of circuit construction.
  For \cnaive, $Y$ is the set of receivers at layer $i$,
  except for $i=1$, where the transmitters at layer $1$ becomes~$Y$.
  For \cstandard, $Y$ is the set of transmitters at layer $i$.
  For \ccompact, as shown in \Cref{fig:tp-circuit},
  $Y$ is the auxiliary nodes for layer $i$.
  Therefore, for \cnaive, \cstandard, and \ccompact,
  we do not modify the circuit $\C$.
  Let $\tilde{\C}=\C$.

  For \csemi, we construct a new circuit~$\tilde{\C}$ as follows.
  First, create $Y = \{b_v^{(j)}\}$ as a copy of $a_v^{(j)}$
  for each vertex $v$ and multiplicity $j$.
  The in-neighbors of $b_v^{(j)}$ are the same as
  the in-neighbors of $a_v^{(j)}$.
  For each $v$ and $j$,
  let $T'(v,j)$ be the set of the transmitters $T_{i,v,d,j}$ at layer $i$ for any weight $d$.
  We remove the edges from $a_v^{(j)}$ to $T'(v,j)$ and
  add new edges from $b_v^{(j)}$ to $T'(v,j)$.
  Observe that if $\C$ has a tree certificate
  including $a_v^{(j)}$ and $T_{i,v,d,j}$, then
  $\tilde{\C}$ must have a tree certificate
  including $b_v^{(j)}$ and $T_{i,v,d,j}$,
  and vise versa.
  Hence, $\C$ contains a tree certificate
  if and only if $\tilde{\C}$ also contains a tree certificate
  for any construction type.
  Also, since every monomial at layer $i$ has degree $i$,
  a tree certificate in $\tilde{\C}$ must contain
  one multiplication gate at layer $i$,
  one node from $Y$, and one variable connected to $Y$.

  Once the algorithm finds a color $c$ for layer $i$,
  \Cref{code:co:e} invalidates the edges from the variable node $x_c$
  for color~$c$ to the layers up to $i-1$ for reducing the circuit size.
  This does not affect the correctness
  because any tree certificate in $\hat{\C}$ with $x_c$
  must include an edge from $x_c$ to a node in $Y$,
  and then the other edges from $x_c$ are never used.
  Invalidating unused edges does not affect the existence of a tree certificate (\Cref{lem:invalidate-cert}).

  The analysis of running time is very similar to that for \rlv.
  Inside the outer loop, any modification of the circuit can be done in $O(m)$.
  The overall running time is dominated by the running time of \Cref{code:tp:mld},
  and the expected number of executions of \Cref{code:tp:mld} is $\Oh(t \log k)$
  as we binary-search a color for each layer from $|\colorset|=k$ colors.
  Notice that this value is smaller than that of \rlv ($\Oh(t \log n)$) if $k < n$.
  Each execution of \Cref{code:tp:mld} takes $\tilde{\Oh}(2^t m)$ time,
  so the total expected running time is $\tilde{\Oh}(t \log k \cdot 2^t m) = \tilde{\Oh}(2^t tm)$.
\end{proof}

\begin{figure*}[h]
    \pgfdeclarelayer{bg}
    \pgfsetlayers{bg, main}

    \tikzstyle{bigblacknode} = [circle, fill=gray, text=white, draw, thick, scale=1, minimum size=0.6cm, inner sep=1.5pt]
    \tikzstyle{bigwhitenode} = [circle, fill=white, text=black, draw, thick, scale=1, minimum size=0.6cm, inner sep=1.5pt]

    \tikzstyle{blacknode} = [circle, fill=gray, text=white, draw, thick, scale=1, minimum size=0.2cm, inner sep=1.5pt]
    \tikzstyle{whitenode} = [circle, fill=white, draw, thick, scale=1, minimum size=0.2cm, inner sep=0pt]

    \tikzstyle{hugewhitenode} = [circle, fill=white, text=black, draw, thick, scale=1, minimum size=1.5cm, inner sep=1.5pt, font=\large]
    \tikzstyle{directed} = [color=gray, arrows=- triangle 45]
    \tikzstyle{thickedge} = [color=black, line width=0.8mm, arrows=-{Latex[length=2mm,width=3mm]}]

    \definecolor{mygreen}{RGB}{0,136,55}
    \definecolor{myblue}{RGB}{5,113,176}
    \definecolor{mypurple}{RGB}{123,50,148}
    \definecolor{myred}{RGB}{202,0,32}
    \definecolor{myredlight}{RGB}{255,204,204}

    \tikzmath{\offset = 0.2;}
    \tikzmath{\layershrink = 0.3;}

    \tikzset{
        old inner xsep/.estore in=\oldinnerxsep,
        old inner ysep/.estore in=\oldinnerysep,
        double circle/.style 2 args={
            circle,
            old inner xsep=\pgfkeysvalueof{/pgf/inner xsep},
            old inner ysep=\pgfkeysvalueof{/pgf/inner ysep},
            /pgf/inner xsep=\oldinnerxsep+#1,
            /pgf/inner ysep=\oldinnerysep+#1,
            alias=sourcenode,
            append after command={
            let     \p1 = (sourcenode.center),
                    \p2 = (sourcenode.east),
                    \n1 = {\x2-\x1-#1-0.5*\pgflinewidth}
            in
                node [inner sep=0pt, draw, circle, minimum width=2*\n1,at=(\p1),#2] {}
            }
        },
        double circle/.default={-3pt}{black}
    }

    \begin{minipage}[m]{.98\linewidth}
        \centering
        \begin{tikzpicture}
          \begin{pgfonlayer}{main}
            \node[bigblacknode] (s) at (-8.5, 8-2) {$s$};
            \node[bigwhitenode] (u) at (-7, 8-2) {$u$};
            \node[bigwhitenode] (v) at (-7, 6.5-2) {$v$};
            \node[bigwhitenode] (w) at (-8.5, 6.5-2) {$w$};
            \draw (s) -- (u) node[midway, above] {$2$};
            \draw (s) -- (v) node[midway, xshift=-10, yshift=15] {$4$};
            \draw (s) -- (w) node[midway, left, xshift=2] {$3$};
            \draw (u) -- (v) node[midway, right, xshift=-2] {$2$};
            \draw (u) -- (w) node[midway, xshift=-10, yshift=-4]{$2$};
            \draw (v) -- (w) node[midway, below]{$1$};
            \draw (u) node[above, yshift=6] {$\{c_1,c_2\}$};
            \draw (w) node[below, yshift=-5] {$\{c_3\}$};
            \draw (v) node[below, yshift=-5] {$\{c_2,c_3\}$};
            \node[] () at (-7.6, 3.3) {Graph instance $G$};

            \node[] () at (-5.8 + 3 * \layershrink, 3) {$2$};
            \node[] () at (-5.8 + 3 * \layershrink, 4) {$3$};
            \node[] () at (-5.8 + 3 * \layershrink, 5) {$4$};
            \node[] () at (-5.8 + 3 * \layershrink, 6) {$5$};
            \node[] () at (-5.8 + 3 * \layershrink, 7) {$6$};
            \node[] () at (-5.8 + 3 * \layershrink, 8) {$7$};

            \draw[dashed, gray] (-5.8 + 3 * \layershrink, 2.5) -- (7.4, 2.5);
            \draw[dashed, gray] (-5.8 + 3 * \layershrink, 3.5) -- (7.4, 3.5);
            \draw[dashed, gray] (-5.8 + 3 * \layershrink, 4.5) -- (7.4, 4.5);
            \draw[dashed, gray] (-5.8 + 3 * \layershrink, 5.5) -- (7.4, 5.5);
            \draw[dashed, gray] (-5.8 + 3 * \layershrink, 6.5) -- (7.4, 6.5);
            \draw[dashed, gray] (-5.8 + 3 * \layershrink, 7.5) -- (7.4, 7.5);
            \draw[dashed, gray] (-5.8 + 3 * \layershrink, 8.5) -- (7.4, 8.5);

            \draw[dashed, gray] (-4.5 + 2 * \layershrink, 1) -- (-4.5 + 2 * \layershrink, 9);
            \draw[dashed, gray] (-3.5 + \layershrink, 1) -- (-3.5 + \layershrink, 9);
            \draw[dashed, gray] (-0.5, 1) -- (-0.5, 9);
            \draw[dashed, gray] (+0.5, 1) -- (+0.5, 9);
            \draw[dashed, gray] (+3.5, 1) -- (+3.5, 9);
            \draw[dashed, gray] (+4.5, 1) -- (+4.5, 9);

            \node[blacknode] (c1) at (-1, 0) {$x_{c_1}$};
            \node[blacknode] (c2) at (0, 0) {$x_{c_2}$};
            \node[blacknode] (c3) at (1, 0) {$x_{c_3}$};

            \node[whitenode, draw=mygreen] (a11) at (-5 + 2.5 * \layershrink, 1.5) {$+$};
            \node[whitenode, draw=mygreen] (a12) at (-4 + 1.5 * \layershrink, 1.5) {$+$};
            \node[whitenode, draw=mygreen] (a13) at (-3 + 0.5 * \layershrink, 1.5) {$+$};
            \node[whitenode, draw=mygreen] (a21) at (-1, 1.5) {$+$};
            \node[whitenode, draw=mygreen] (a22) at (+0, 1.5) {$+$};
            \node[whitenode, draw=mygreen] (a23) at (+1, 1.5) {$+$};
            \node[whitenode, draw=mygreen] (a31) at (3, 1.5) {$+$};
            \node[whitenode, draw=mygreen] (a32) at (4, 1.5) {$+$};
            \node[whitenode, draw=mygreen] (a33) at (5, 1.5) {$+$};

            \node[whitenode, draw=myblue] (y112) at (-5 + 2.5 * \layershrink, 3) {$\times$};
            \node[whitenode, draw=myblue] (y124) at (-4 + 1.5 * \layershrink, 5) {$\times$};
            \node[whitenode, draw=myblue] (y133) at (-3 + 0.5 * \layershrink, 4) {$\times$};

            \node[whitenode, draw=mypurple] (z212) at (-1 - \offset, 3 - \offset) {$+$};
            \node[whitenode, draw=myblue] (y212) at (-1 + \offset, 3 + \offset) {$\times$};
            \node[whitenode, draw=mypurple] (z215) at (-1 - \offset, 6 - \offset) {$+$};
            \node[whitenode, draw=myblue] (y215) at (-1 + \offset, 6 + \offset) {$\times$};
            \node[whitenode, draw=mypurple] (z216) at (-1 - \offset, 7 - \offset) {$+$};
            \node[whitenode, draw=myblue] (y216) at (-1 + \offset, 7 + \offset) {$\times$};

            \node[whitenode, draw=mypurple] (z224) at ( 0 - \offset, 5 - \offset) {$+$};
            \node[whitenode, draw=myblue] (y224) at ( 0 + \offset, 5 + \offset) {$\times$};

            \node[whitenode, draw=mypurple] (z233) at ( 1 - \offset, 4 - \offset) {$+$};
            \node[whitenode, draw=myblue] (y233) at (1  + \offset, 4 + \offset) {$\times$};
            \node[whitenode, draw=mypurple] (z234) at ( 1 - \offset, 5 - \offset) {$+$};
            \node[whitenode, draw=myblue] (y234) at (1  + \offset, 5 + \offset) {$\times$};
            \node[whitenode, draw=mypurple] (z235) at ( 1 - \offset, 6 - \offset) {$+$};
            \node[whitenode, draw=myblue] (y235) at (1  + \offset, 6 + \offset) {$\times$};

            \node[whitenode, draw=mypurple] (z312) at (3 - \offset, 3 - \offset) {$+$};
            \node[whitenode, draw=myblue] (y312) at (3 + \offset, 3 + \offset) {$\times$};
            \node[whitenode, draw=mypurple] (z315) at (3 - \offset, 6 - \offset) {$+$};
            \node[whitenode, draw=myblue] (y315) at (3 + \offset, 6 + \offset) {$\times$};
            \node[whitenode, draw=mypurple] (z316) at (3 - \offset, 7 - \offset) {$+$};
            \node[whitenode, draw=myblue] (y316) at (3 + \offset, 7 + \offset) {$\times$};
            \node[whitenode, draw=mypurple] (z317) at (3 - \offset, 8 - \offset) {$+$};
            \node[whitenode, draw=myblue] (y317) at (3 + \offset, 8 + \offset) {$\times$};

            \node[whitenode, draw=mypurple] (z324) at (4 - \offset, 5 - \offset) {$+$};
            \node[whitenode, draw=myblue] (y324) at (4 + \offset, 5 + \offset) {$\times$};
            \node[whitenode, draw=mypurple] (z325) at (4 - \offset, 6 - \offset) {$+$};
            \node[whitenode, draw=myblue] (y325) at (4 + \offset, 6 + \offset) {$\times$};
            \node[whitenode, draw=mypurple] (z326) at (4 - \offset, 7 - \offset) {$+$};
            \node[whitenode, draw=myblue] (y326) at (4 + \offset, 7 + \offset) {$\times$};
            \node[whitenode, draw=mypurple] (z327) at (4 - \offset, 8 - \offset) {$+$};
            \node[whitenode, draw=myblue] (y327) at (4 + \offset, 8 + \offset) {$\times$};

            \node[whitenode, draw=mypurple] (z333) at (5 - \offset, 4 - \offset) {$+$};
            \node[whitenode, draw=myblue] (y333) at (5 + \offset, 4 + \offset) {$\times$};
            \node[whitenode, draw=mypurple] (z334) at (5 - \offset, 5 - \offset) {$+$};
            \node[whitenode, draw=myblue] (y334) at (5 + \offset, 5 + \offset) {$\times$};
            \node[whitenode, draw=mypurple] (z335) at (5 - \offset, 6 - \offset) {$+$};
            \node[whitenode, draw=myblue] (y335) at (5 + \offset, 6 + \offset) {$\times$};
            \node[whitenode, draw=mypurple] (z337) at (5 - \offset, 8 - \offset) {$+$};
            \node[whitenode, draw=myblue] (y337) at (5 + \offset, 8 + \offset) {$\times$};

            \node[bigwhitenode,double circle] (r7) at (6.5, 8) {$+$};

            \draw[directed,thickedge] (c1) -- (a11);
            \draw[directed,thickedge] (c2) -- (a21);
            \draw[directed,thickedge] (c3) -- (a33);
            \draw[directed,thickedge] (a11) -- (y112);
            \draw[directed,thickedge] (a21) -- (y212);
            \draw[directed,thickedge] (y112) -- (z212);
            \draw[directed,thickedge] (y212) -- (z334);
            \draw[directed,thickedge] (z212) -- (y212);
            \draw[directed,thickedge] (z334) -- (y334);
            \draw[directed,thickedge] (a33) -- (y334);
            \draw[directed,thickedge] (y334) -- (r7);
          \end{pgfonlayer}

          \begin{pgfonlayer}{bg}
            \draw[fill=myredlight] (+2.5, 2.5) rectangle ++ (1.0, 6.0);
            
            \draw[directed] (c1) -- (a21);
            \draw[directed, thickedge, dashed, draw=myred] (c1) -- (a31);
            
            \draw[directed] (c2) -- (a11);
            \draw[directed, thickedge, dashed, draw=myred] (c2) -- (a31);

            \draw[directed] (c2) -- (a12);
            \draw[directed] (c2) -- (a22);
            \draw[directed, thickedge, dashed, draw=myred] (c2) -- (a32);

            \draw[directed] (c3) -- (a12);
            \draw[directed] (c3) -- (a22);
            \draw[directed] (c3) -- (a32);

            \draw[directed] (c3) -- (a13);
            \draw[directed] (c3) -- (a23);

            \draw[directed] (a12) -- (y124);
            \draw[directed] (a13) -- (y133);

            \draw[directed] (a21) -- (y215);
            \draw[directed] (a21) -- (y216);
            \draw[directed] (a22) -- (y224);
            \draw[directed] (a23) -- (y233);
            \draw[directed] (a23) -- (y234);
            \draw[directed] (a23) -- (y235);

            \draw[directed, thickedge, dashed, draw=myred] (a31) -- (y312);
            \draw[directed, thickedge, dashed, draw=myred] (a31) -- (y315);
            \draw[directed, thickedge, dashed, draw=myred] (a31) -- (y316);
            \draw[directed, thickedge, dashed, draw=myred] (a31) -- (y317);
            \draw[directed] (a32) -- (y324);
            \draw[directed] (a32) -- (y325);
            \draw[directed] (a32) -- (y326);
            \draw[directed] (a32) -- (y327);
            \draw[directed] (a33) -- (y333);
            \draw[directed] (a33) -- (y335);
            \draw[directed] (a33) -- (y337);

            \draw[directed] (y112) -- (z224);
            \draw[directed] (y112) -- (z234);

            \draw[directed] (y124) -- (z216);
            \draw[directed] (y124) -- (z224);
            \draw[directed] (y124) -- (z235);

            \draw[directed] (y133) -- (z215);
            \draw[directed] (y133) -- (z224);
            \draw[directed] (y133) -- (z233);

            \draw[directed] (z215) -- (y215);
            \draw[directed] (z216) -- (y216);
            \draw[directed] (z224) -- (y224);
            \draw[directed] (z233) -- (y233);
            \draw[directed] (z234) -- (y234);
            \draw[directed] (z235) -- (y235);

            \draw[directed] (y212) -- (z312);
            \draw[directed] (y212) -- (z324);

            \draw[directed] (y215) -- (z315);
            \draw[directed] (y215) -- (z327);
            \draw[directed] (y215) -- (z337);
            \draw[directed] (y216) -- (z316);

            \draw[directed] (y224) -- (z316);
            \draw[directed] (y224) -- (z324);
            \draw[directed] (y224) -- (z335);

            \draw[directed] (y233) -- (z315);
            \draw[directed] (y233) -- (z324);
            \draw[directed] (y233) -- (z333);
            \draw[directed] (y234) -- (z316);
            \draw[directed] (y234) -- (z325);
            \draw[directed] (y234) -- (z334);
            \draw[directed] (y235) -- (z317);
            \draw[directed] (y235) -- (z326);
            \draw[directed] (y235) -- (z335);

            \draw[directed] (z312) -- (y312);
            \draw[directed] (z315) -- (y315);
            \draw[directed] (z316) -- (y316);
            \draw[directed] (z317) -- (y317);
            \draw[directed] (z324) -- (y324);
            \draw[directed] (z325) -- (y325);
            \draw[directed] (z326) -- (y326);
            \draw[directed] (z327) -- (y327);
            \draw[directed] (z333) -- (y333);
            \draw[directed] (z335) -- (y335);
            \draw[directed] (z337) -- (y337);

            \draw[directed, thickedge, dashed, draw=myred] (y315) -- (r7);

          \draw[rounded corners, gray] (-5.6 + 3 * \layershrink, 2.4) rectangle ++ (3.2 - 3 * \layershrink, 6.2);
          \draw[rounded corners, gray] (-1.6, 2.4) rectangle ++ (3.2, 6.2);
          \draw[rounded corners, gray] (2.4, 2.4) rectangle ++ (3.2, 6.2);
          \node[] at (-4 + 1.5 * \layershrink, 9.3) {Layer $1$};
          \node[] at ( 0, 9.3) {Layer $2$};
          \node[] at ( 4, 9.3) {Layer $3$};
          \node[] at ( 6.5, 9.2) {Output};
          \node[] at ( 6.5, 8.8) {nodes};
          \node[] at (-5 + 2.5 * \layershrink, 8.8) {$u$};
          \node[] at (-4 + 1.5 * \layershrink, 8.8) {$v$};
          \node[] at (-3 + 0.5 * \layershrink, 8.8) {$w$};
          \node[] at (-1, 8.8) {$u$};
          \node[] at ( 0, 8.8) {$v$};
          \node[] at ( 1, 8.8) {$w$};
          \node[] at ( 3, 8.8) {$u$};
          \node[] at ( 4, 8.8) {$v$};
          \node[] at ( 5, 8.8) {$w$};

          \node[rotate=90] at (-6.1 + 3 * \layershrink - 0.1, 5.5) {walk weight from $s$};

          \draw[rounded corners, mygreen] (-5.6 + 3 * \layershrink, 1) rectangle ++ (3.2 - 3 * \layershrink, 1);
          \draw[rounded corners, mygreen] (-1.6, 1) rectangle ++ (3.2, 1);
          \draw[rounded corners, mygreen] (2.4, 1) rectangle ++ (3.2, 1);
          \draw[black] (2.2, 0.8) rectangle ++ (3.6, 1.4);
          \node[] at ( 3.4, 1.8) {\color{mygreen} $a_{3,u}$};
          \node[] at ( 4.4, 1.8) {\color{mygreen} $a_{3,v}$};
          \node[] at ( 5.4, 1.8) {\color{mygreen} $a_{3,w}$};
          \node[] at ( -1, 7.5 + \offset) {\color{myblue}$T_{2,u,6}$};
          \node[] at (-1.5, 7.0 + \offset) {\color{mypurple}$R_{2,u,6}$};
          \node[] at (5.0, 0.6) {Color control nodes $Y$};

          \node[] () at (-5.0, 0.7) {Auxiliary nodes};
          \end{pgfonlayer}

        \end{tikzpicture}
    \end{minipage}
    \caption{%
      An example of the color-guessing step in \rco, with a circuit constructed by \ccompact.
      The input instance is the same as \Cref{fig:construction}.
      Here, we guess the color used in layer $3$ to be $c_3$.
      For \ccompact, the color control nodes $Y$ are the auxiliary nodes at layer $3$.
      We invalidate all edges from the variable nodes for the other colors $Y$,
      depicted as dashed lines in red.
      Notice that since $a_{3,u}$ has no input from $x_{c_3}$,
      all multiplication gates for vertex $u$ in layer $3$ are also invalidated
      (shaded in pink).
      }

  \label{fig:tp-circuit}
\end{figure*}

\subsubsection{Reconstructing Walk from Color Order.}

\lemtpwalkrecovery*

\begin{proof}
  We construct an auxiliary \emph{directed} graph $G'=(V',E')$ with edge weights $w':E' \to \R_{\geq 0}$ as follows.
  $V'$ consists of $t$ layers,
  where layer $i$ consists of vertices $U_i = \{u_{i,v} \colon v \in \chi^{-1}(c_i)\}$,
  which are copies of the vertices in $V$ having color $c_i$.
  We add two more vertices, $s$ for the source and its copy $s'$ for the sink.
  We connect these to internal vertices by $N^+(s)=U_1$ and $N^-(s')=U_t$
  and set edge weights to $w'(s,u_{1,v})=w(s,v)$ and $w'(u_{t,v},s')=w(v,s)$.
  For $1 \leq i \leq t$, we add edges from all of $U_i$ to all of $U_{i+1}$.
  Edge weights are set to $w'(u_{i,v},u_{i+1,v'}) = w(v,v')$,
  where we define $w(v,v)=0$.
  We will show that a shortest path from $s$ to $s'$ in $G'$ corresponds to
  a minimum-weight walk $W=(s,v_1, \ldots, v_t, s)$ in $G$
  such that $c_i \in \chi(v_i)$ for each $1 \leq i \leq t$.
  From \Cref{lem:consistent-walk}, there must be an optimal walk with $t$
  internal vertices, with possibly repeated vertices.

  First, observe that there is a one-to-one mapping
  from a path $P=(s, u_{1,v_1}, \ldots, u_{t,v_t}, s')$ in $G'$
  to a walk $W=(s, v_1, \ldots, v_t, s)$ in $G$ that is consistent with the given color order.
  By construction, $v_i \in \chi^{-1}(c_i)$ for every~$i$, and hence $c_i \in \chi(v_i)$.
  Also, from
  $w(s, v_1) + w(v_t,s) + \sum_{i=1}^{t-1} w(v_i, v_{i+1})
  =w'(s,u_{1,v_1}) + w'(u_{t,v_t}, s') + \sum_{i=1}^{t-1} w'(u_{i,v_i}, v_{i+1,v_{i+1}})$,
  the weight of $P$ is equal to the weight of $W$.

  Suppose $\tilde{P}=(s, a_1, \ldots, a_p, s')$ be a shortest path from $s$ to $s'$ in $G'$.
  By construction, $a_1 \in U_1$, $a_k \in U_t$,
  and each of the other edges increases the layer index by one.
  So, $\tilde{P}$ must be in the form of $(s, u_{1,v_1}, \ldots, u_{t,v_t}, s')$.
  This implies that $P$ has the minimum weight, then $W$ also has the minimum weight.

  For the running time, constructing $G'$ and running Dijkstra's algorithm on $G'$ take
  $\Oh(|E'| + |V'|\log |V'|)$,
  which is $\tilde{\Oh}(tn^2)$ since $|V'|\leq tn+2$ and $|E'|\leq (t-1)n^2+2n$.
\end{proof}

%

\begin{figure*}[t]
  \centering
  \begin{subfigure}{0.48\textwidth}
    \includegraphics[width=\linewidth]{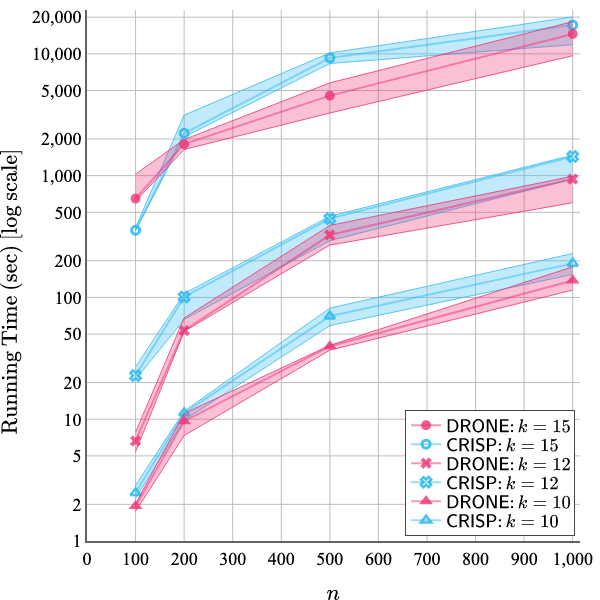}
  \end{subfigure}
  \hfill
  \begin{subfigure}{0.48\textwidth}
    \includegraphics[width=\linewidth]{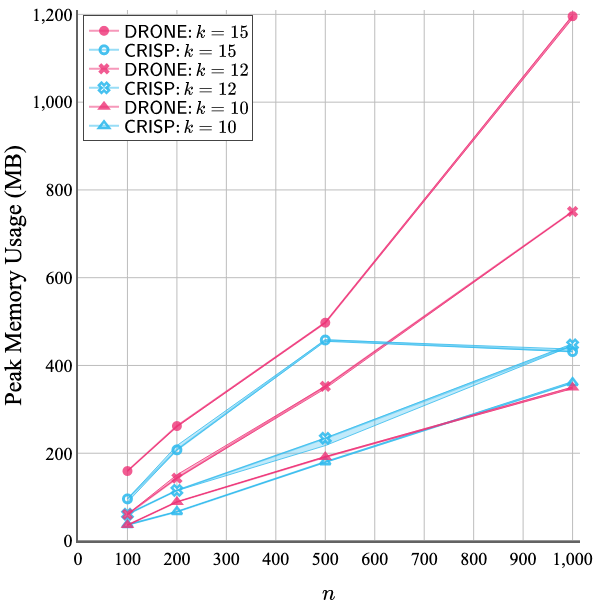}
  \end{subfigure}
  \caption{
  Overall running time (left) and peak memory usage (right) for different graph sizes and $k$ values.
  The median values are drawn with lines and markers,
  and the minimum and maximum values are shown as envelopes
  (same for \Cref{fig:multithread}).}
  \label{fig:large}
\end{figure*}

\begin{figure*}[t]
  \centering

  \begin{subfigure}{0.32\textwidth}
    \includegraphics[width=\linewidth]{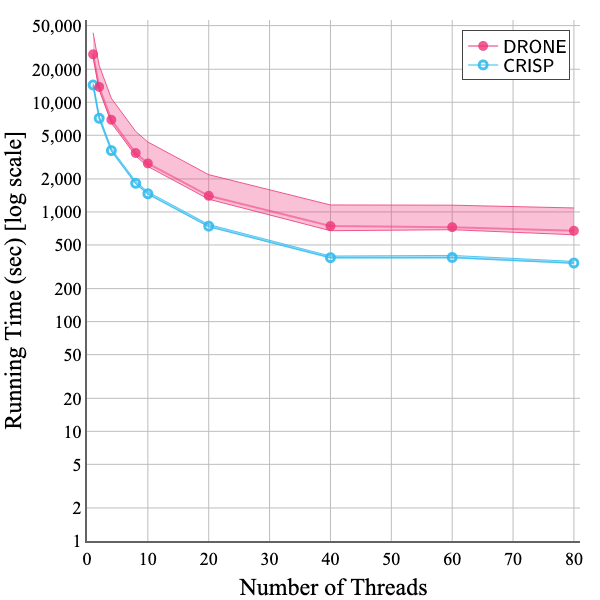}
  \end{subfigure}
  \hfill
  \begin{subfigure}{0.32\textwidth}
    \includegraphics[width=\linewidth]{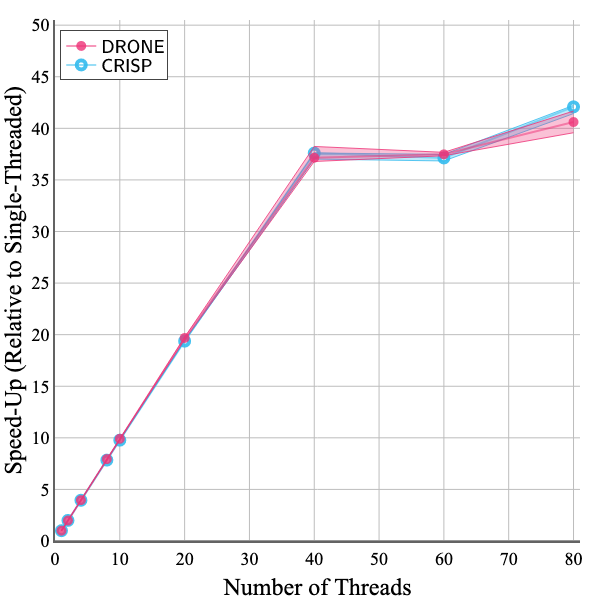}
  \end{subfigure}
  \hfill
  \begin{subfigure}{0.32\textwidth}
    \includegraphics[width=\linewidth]{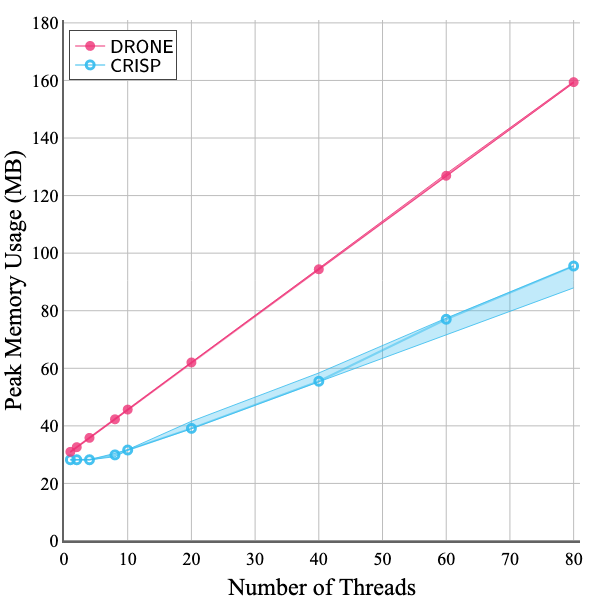}
  \end{subfigure}
  \caption{
  The running time (left),
  the speed-up compared to single-threaded (middle),
  and the peak memory usage (right) of \algipa
  for different numbers of threads.}
  \label{fig:multithread}
\end{figure*}

\section{Experiments on Scalability}\label{appendix:scalability}

In addition to the experiments in the main paper,
we ran experiments on scalability, specifically with respect to different graph sizes and numbers of threads.

\subsection{Scalability with Respect to Graph Size.}\label{appendix:scalability-large}

To see the scalability for $n$ values (the number of vertices),
we ran \algipa with our default subroutines
(\csemi, \sunified, \rtp), $\lambda_\text{small}$, success probability $0.9$, $k=10,12,15$, and varied $n$
from $100$ to $1000$, which matches the graph size of the experiments in~\cite{fu2023asymptoticallyoptimal}.

\Cref{fig:large} (left) shows the overall running time for each configuration.
It is clear to see that for every instance, the running time grows polynomially with respect to $n$.
\Cref{fig:large} (right) plots the peak memory usage for the same experiment.
We observe almost-linear relationship between $n$ and memory usage,
with an exception with \CRISP, $n=1000$, and $k=15$.
This is due to a very low solution value triggered by the small scaling factor
(recall that the circuit size is sensitive to the solution value).

\subsection{Multithreading Analysis.}

Finally, we show results of multithreading experiments for \algipa, visualized in \Cref{fig:multithread}.
We tested with $n=100$, $k=15$, and different numbers of threads from $1$ (single-threaded) to $80$.
Other configurations are the same as those in \Cref{appendix:scalability-large}.

The left and middle figures depict running times, absolute and relative (to single-threaded).
We see a linear speed-up up to $40$ threads, which is the number of physical cores of our hardware.
Notice that for both datasets, up to that point, the speed-up is nearly ideal, achieving 
the speed-up of around factor-$38$ for $40$ threads,
in contract to around $17$ for \dpipa \cite{mizutani2024leveraging}.
The speed-up degrades beyond $40$ threads, but with $80$ threads, the speed-up reaches $40$.

The trend of the peak memory usage (right) is also straightforward;
the memory usage grows linearly to the number of threads,
independent of the number of physical cores.

\begin{figure*}[t]
  \centering

  \begin{subfigure}{0.32\textwidth}
    \includegraphics[width=\linewidth]{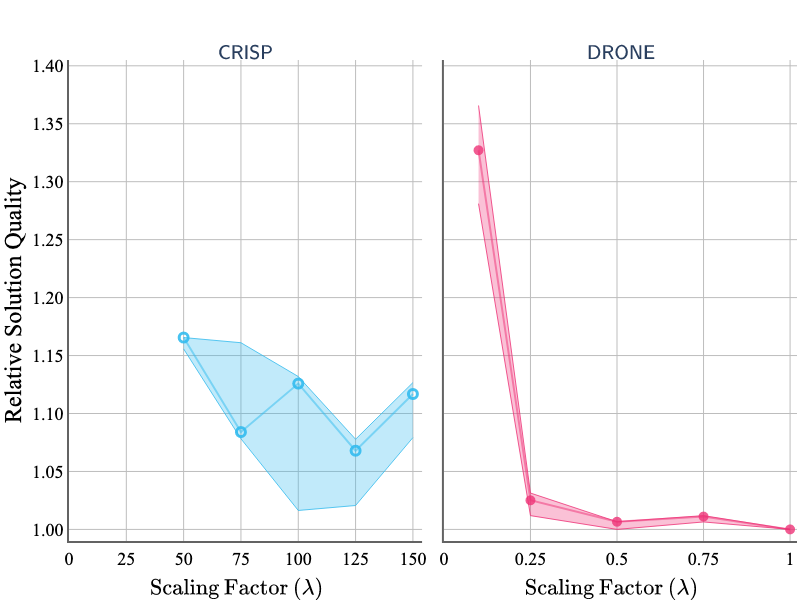}
  \end{subfigure}
  \hfill
  \begin{subfigure}{0.32\textwidth}
    \includegraphics[width=\linewidth]{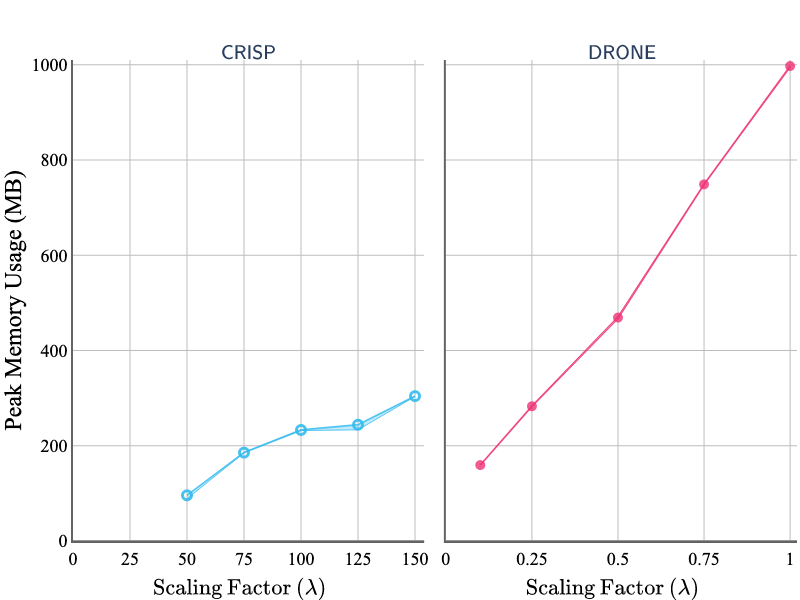}
  \end{subfigure}
  \hfill
  \begin{subfigure}{0.32\textwidth}
    \includegraphics[width=\linewidth]{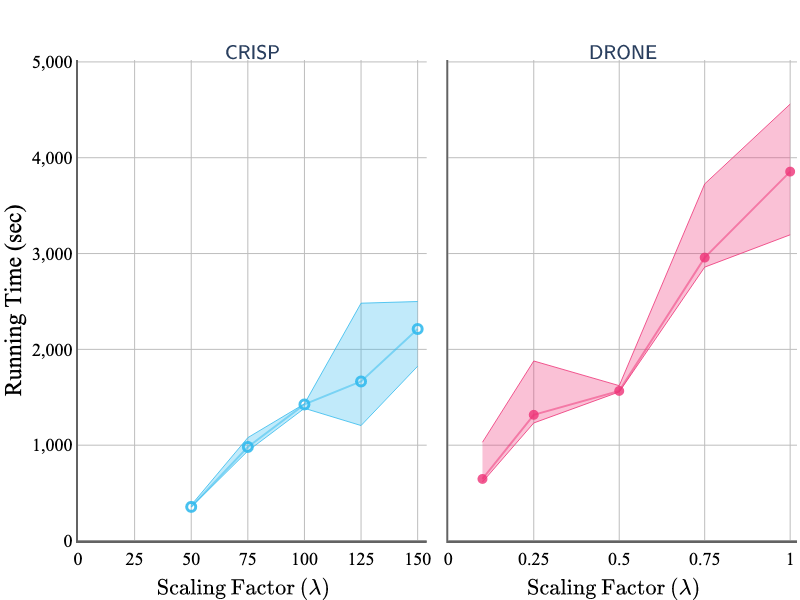}
  \end{subfigure}
  \caption{
  The ratio of the weight obtained by \algipa to the optimal weight (left),
  the peak memory usage (middle),
  and the overall running time (right) for different scaling factors.
  The median values are drawn with lines and markers,
  and the minimum and maximum values are shown as envelopes
  (same for \Cref{fig:success,fig:algdp}).
  Larger scaling factors improve solution quality
  at the cost of increased memory usage and running time
  (discussions in \Cref{sec:exp:scale})}
  \label{fig:scale}
\end{figure*}

\begin{figure*}[t]
  \centering
  \begin{subfigure}{0.48\textwidth}
    \includegraphics[width=\linewidth]{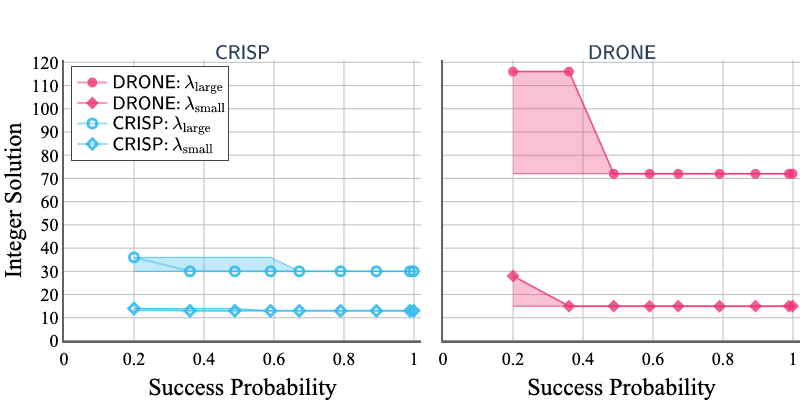}
  \end{subfigure}
  \hfill
  \begin{subfigure}{0.48\textwidth}
    \includegraphics[width=\linewidth]{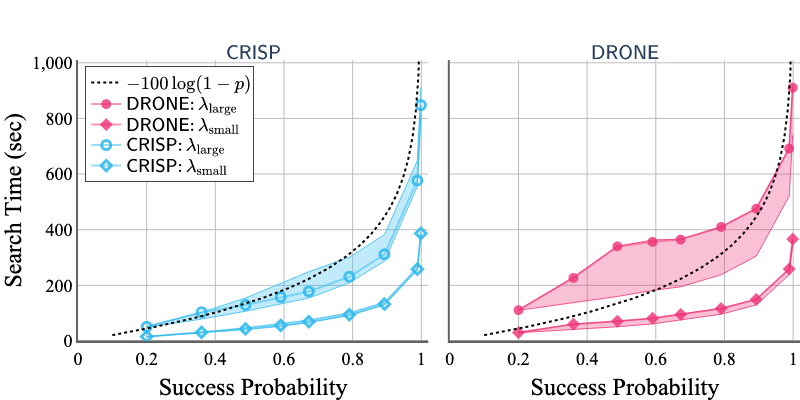}
  \end{subfigure}
  \caption{
  Obtained integer solutions (left) and search times (right) for varied success probabilities.
  Our algorithm finds optimal integer solution even with the expected success probability around $0.7$
  (discussions in \Cref{sec:exp:success}).}
  \label{fig:success}
\end{figure*}

\section{Preprocessing Results}\label{appendix:preprocessing}

We present results of preprocessing on our \gi{} instances. Because $k$ (the number of colors present in the graph)
is small, this preprocessing is quite effective. In particular, removing colorless vertices significantly reduces
the size of the graphs.
\Cref{table:preprocessing} shows statistics of the instances used for \Cref{appendix:scalability-large}.

\begin{table*}[h]
    \caption{\label{table:preprocessing}
    Instance sizes before and after preprocessing.
    $n_{\text{build}}$ is the parameter given to the software of Fu \emph{et al.}~\cite{fu2021computationally}
    during instance construction, and can be thought of as a ``target'' number of vertices for the constructed graph.
    $k$ is the number of colors remaining after performing color reduction.
    $n$ and $m$ are the number of vertices and edges in the color-reduced instance,
    while $n'$ and $m'$ are the same statistics after preprocessing.
    Note that we require a complete graph, $m' = \binom{n'}{2}$.
    }
    \centering
    \begin{tabular}{l | r | r | r r | r r}
      Dataset &$n_\text{build}$ &$k$ &$n$ &$m$ &$n'$ &$m'$ \\
      \hline
      \multirow{12}*{\CRISP} &\multirow{3}*{100}  &10 &\multirow{3}*{113} &\multirow{3}*{951}  &20  &190\\
                            &                    &12 &                   &                    &24  &276\\
                            &                    &15 &                   &                    &25  &300\\
                            \cline{2-7}
                            &\multirow{3}*{200}  &10 &\multirow{3}*{209} &\multirow{3}*{2132} &35  &595\\
                            &                    &12 &                   &                    &40  &780\\
                            &                    &15 &                   &                    &46  &1035\\
                            \cline{2-7}
                            &\multirow{3}*{500}  &10 &\multirow{3}*{502} &\multirow{3}*{6602} &76  &2850\\
                            &                    &12 &                   &                    &77  &2926\\
                            &                    &15 &                   &                    &91  &4095\\
                            \cline{2-7}
                            &\multirow{3}*{1000} &10 &\multirow{3}*{1006}&\multirow{3}*{15727}&145 &10440\\
                            &                    &12 &                   &                    &150 &11175\\
                            &                    &15 &                   &                    &170 &14365\\
      \hline
      \multirow{12}*{\DRONE} &\multirow{3}*{100}  &10 &\multirow{3}*{119} &\multirow{3}*{878}  &16  &120\\
                            &                    &12 &                   &                    &19  &171\\
                            &                    &15 &                   &                    &23  &253\\
                            \cline{2-7}
                            &\multirow{3}*{200}  &10 &\multirow{3}*{209} &\multirow{3}*{2043} &30  &435\\
                            &                    &12 &                   &                    &33  &528\\
                            &                    &15 &                   &                    &43  &903\\
                            \cline{2-7}
                            &\multirow{3}*{500}  &10 &\multirow{3}*{501} &\multirow{3}*{6806} &65  &2080\\
                            &                    &12 &                   &                    &79  &3081\\
                            &                    &15 &                   &                    &97  &4656\\
                            \cline{2-7}
                            &\multirow{3}*{1000} &10 &\multirow{3}*{1002}&\multirow{3}*{16227}&107 &5671\\
                            &                    &12 &                   &                    &143 &10153\\
                            &                    &15 &                   &                    &176 &15400\\
    \end{tabular}
    \vspace{1pt}
    \vspace{-2.5em}
\end{table*}

\end{document}